\documentclass{article}

\usepackage{ragged2e}  
\usepackage{xurl}      

\usepackage{amssymb}
\usepackage{pifont}
\usepackage{xcolor}
\usepackage{booktabs}

\newcommand{\xmark}{\textcolor{red}{\ding{55}}}  
\renewcommand{\checkmark}{\textcolor{green!60!black}{\ding{51}}}  
\usepackage{booktabs}
\usepackage{tabularx}
\usepackage{makecell}
\usepackage{array}
\usepackage{microtype}
\usepackage{graphicx}
\usepackage{subcaption}
\usepackage{booktabs} 
\usepackage{enumitem}
\usepackage[utf8]{inputenc}
\usepackage{CJKutf8}
\usepackage{tabularx} 
\usepackage{booktabs}   
\usepackage{longtable}  
\usepackage{multirow} 
\usepackage{tabularx}   
\usepackage{ragged2e}   
\usepackage{array}      

\usepackage{hyperref}

\usepackage[preprint]{icml2026}

\usepackage{amsmath}
\usepackage{amssymb}
\usepackage{mathtools}
\usepackage{amsthm}

\usepackage{tabularx}   
\usepackage{booktabs}   
\usepackage{makecell}   
\usepackage{amsmath}    
\usepackage{amssymb}    

\usepackage{tikz}
\usepackage{colortbl}
\usepackage{hyperref}
\newcommand{\Ev}[1]{\hyperref[ev:#1]{[A#1]}}
\usetikzlibrary{shadows}
\definecolor{kleinblue}{RGB}{0, 47, 167} 
\definecolor{kleinblue2}{RGB}{20, 20, 125} 
\newcommand*\circled[1]{%
\tikz[baseline=(char.base)]{
  \node[shape=circle, fill=kleinblue2, draw=white, thick, drop shadow={shadow xshift=0.2ex,shadow yshift=-0.2ex}, inner sep=1pt] (char) {\textcolor{white}{#1}};
}}
\usepackage[capitalize,noabbrev]{cleveref}

\theoremstyle{plain}
\newtheorem{theorem}{Theorem}[section]
\newtheorem{proposition}[theorem]{Proposition}
\newtheorem{lemma}[theorem]{Lemma}

\theoremstyle{definition}
\newtheorem{definition}[theorem]{Definition}

\theoremstyle{remark}

\usepackage[textsize=tiny]{todonotes}

\icmltitlerunning{Submission and Formatting Instructions for ICML 2026}

\begin{document}

\twocolumn[
  \icmltitle{DoAtlas-1: A Causal Compilation Paradigm for Clinical AI}



  \icmlsetsymbol{equal}{*}

\begin{icmlauthorlist}
  \icmlauthor{Yulong Li}{equal,mbzuai}
  \icmlauthor{Jianxu Chen}{equal,mbzuai}
  \icmlauthor{Xiwei Liu}{equal,mbzuai}
  \icmlauthor{Chuanyue Suo}{mbzuai}
  \icmlauthor{Rong Xia}{mbzuai}
  \icmlauthor{Zhixiang Lu}{xjtlu}
  \icmlauthor{Yichen Li}{mbzuai}
  \icmlauthor{Xinlin Zhuang}{mbzuai}
  \icmlauthor{Niranjana Arun Menon}{mbzuai}
  \icmlauthor{Yutong Xie}{mbzuai}
  \icmlauthor{Eran Segal}{mbzuai}
  \icmlauthor{Imran Razzak}{mbzuai}
\end{icmlauthorlist}

\icmlaffiliation{mbzuai}{Mohamed bin Zayed University of Artificial Intelligence}
\icmlaffiliation{xjtlu}{Xi'an Jiaotong-Liverpool University}

\icmlcorrespondingauthor{Eran Segal}{Eran.Segal@mbzuai.ac.ae}
\icmlcorrespondingauthor{Imran Razzak}{Imran.Razzak@mbzuai.ac.ae}

]



\printAffiliationsAndNotice{}  

\begin{abstract}
Medical foundation models generate narrative explanations but cannot quantify intervention effects, detect evidence conflicts, or validate literature claims, limiting clinical auditability. We propose \textbf{causal compilation}, a paradigm that transforms medical evidence from narrative text into executable code. The paradigm standardizes heterogeneous research evidence into structured estimand objects, each explicitly specifying intervention contrast, effect scale, time horizon, and target population, supporting six executable causal queries: do-calculus, counterfactual reasoning, temporal trajectories, heterogeneous effects, mechanistic decomposition, and joint interventions. We instantiate this paradigm in DoAtlas-1, compiling 1,445 effect kernels from 754 studies through effect standardization, conflict-aware graph construction, and real-world validation (Human Phenotype Project, 10,000 participants). The system achieves 98.5\% canonicalization accuracy and 80.5\% query executability. This paradigm shifts medical AI from text generation to executable, auditable, and verifiable causal reasoning.
\end{abstract}
\vspace{-2em}
\section{Introduction}

Medical foundation models are emerging as the universal interface layer of clinical decision support systems (CDSS), delivering textual answers for clinical questions with evidence grounding by combining generative reasoning with retrieval augmentation~\cite{Singhal2023llm}.
However, as shown in Table~\ref{tab:table1}, what ultimately determines for clinical deployment is not generating plausible-sounding knowledge narratives, but supporting accountable intervention decisions: quantifying interventional causal effects on outcomes for a specified target population under explicit constraints~\cite{Singhal2023llm}. Existing paradigms constructed around generative reasoning and citation alignment typically lack a principled mechanism for compiling heterogeneous research findings into executable interventional estimands~\cite{hernal_causal_2020, hernan_target_2016}. In particular, contrast definition, effect scale/unit, time window, and target population are often left implicit; uncertainty and evidential conflicts are difficult to represent and adjudicate in a structured manner, which limits auditability in high-stakes clinical settings~\cite{kahan_estimands_2024}.
To address this, we propose DoAtlas, a new paradigm for clinical causal inference that compiles heterogeneous studies into structured interventional estimands, explicitly models effect contrasts, scale, time window, and target population, and adjudicates evidential conflicts using real-world data validation signals.



\begin{table}[t]
\centering

\label{tab:paradigm_shift}
\small
\setlength{\tabcolsep}{10pt}
\resizebox{\linewidth}{!}{
\begin{tabular}{lll}
\toprule
 & \textbf{Narrative Generation} & \textbf{Causal Compilation (DoAtlas)} \\
\midrule

\textbf{Representation} & 
Text with citations & 
Executable estimand objects \\
\addlinespace[0.3em]

\textbf{Query Interface} & 
Natural language questions & 
Typed queries ($Q_{do}$, $Q_{med}$, $Q_{CATE}$, ...) \\
\addlinespace[0.3em]

\textbf{Output Format} & 
``Studies suggest...'' + citations & 
$\hat{\theta}=0.73$ [0.60--0.88] \checkmark validated \\
\addlinespace[0.3em]

\textbf{Effect Quantification} & 
\xmark\ Qualitative descriptions only & 
\checkmark\ Quantitative estimates with CI \\
\addlinespace[0.3em]

\textbf{Conflict Detection} & 
\xmark\ Implicit (``results vary'') & 
\checkmark\ Explicit (directional, interval, $I^2$) \\
\addlinespace[0.3em]

\textbf{External Validation} & 
\xmark\ Cannot verify claims & 
\checkmark\ Statistical test vs. real-world data \\
\addlinespace[0.3em]

\textbf{Causal Reasoning} & 
\xmark\ Narrative descriptions only & 
\checkmark\ 6 executable query types (Table~\ref{tab:query-gallery}) \\
\addlinespace[0.3em]

\textbf{Composability} & 
\xmark\ Each query independent & 
\checkmark\ Graph-based query composition \\
\addlinespace[0.3em]

\textbf{Reproducibility} & 
Stochastic (varies across runs) & 
Deterministic with full provenance \\
\addlinespace[0.3em]

\textbf{Evidence Update} & 
Requires model retraining (weeks) & 
Incremental compilation (minutes) \\
\addlinespace[0.3em]

\textbf{Example Systems} & 
Med-PaLM, GPT-4-Med, RAG systems & 
DoAtlas (ours) \\

\bottomrule
\end{tabular}
}
\caption{Paradigm shift from narrative generation to causal compilation for clinical decision support.}
\label{tab:table1}
\vspace{-2.5em}
\end{table}

\begin{figure*}[!t]
\centering
\includegraphics[width=0.9\textwidth]{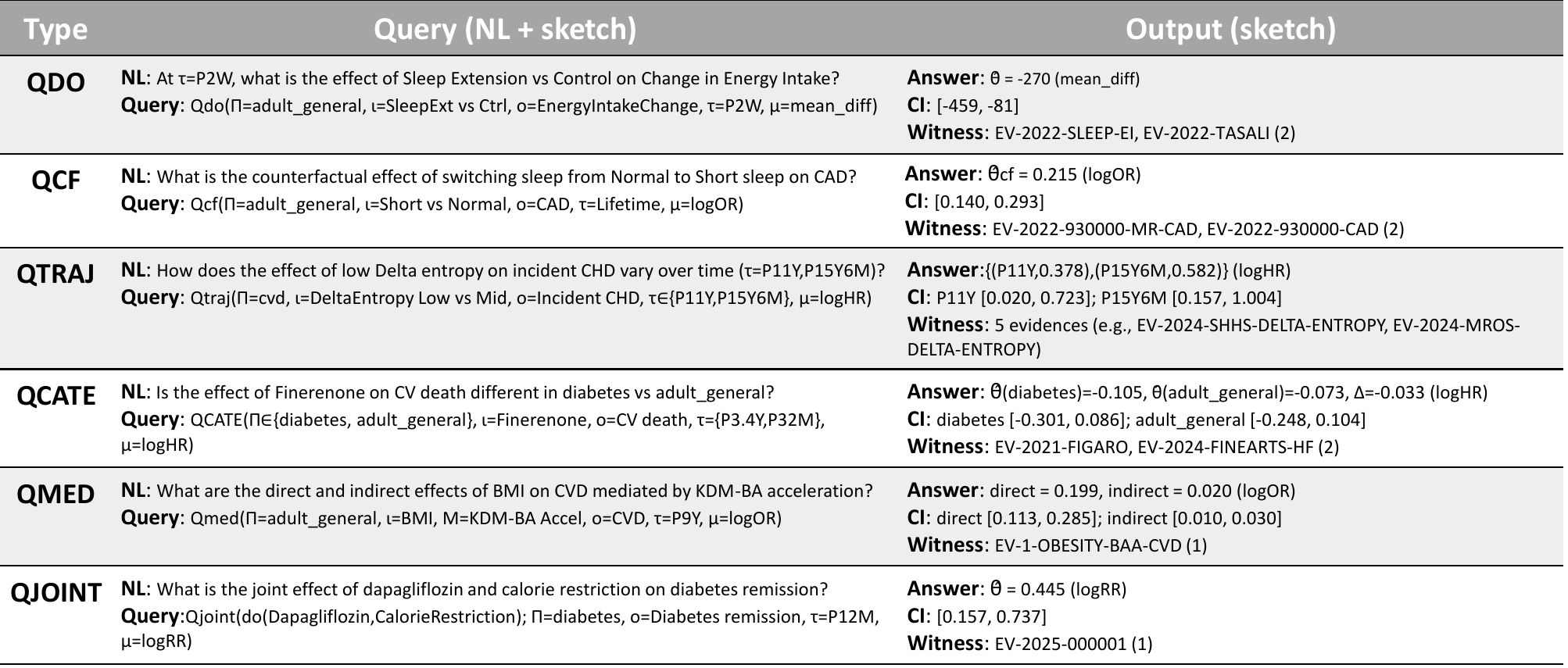}
\caption{\textbf{Query Gallery.} DoAtlas supports six executable query types:
\textsc{Qdo} (interventional effect), \textsc{Qcf} (counterfactual effect under an alternative exposure),
\textsc{Qtraj} (time-indexed effect trajectory), \textsc{QCATE} (heterogeneous effects across populations/strata),
\textsc{Qmed} (direct/indirect effects via a mediator), and \textsc{Qjoint} (joint intervention effect).
Each answer is computed from a \emph{witness} subset of evidence objects (EV-*); when a query is not executable, the system returns diagnostic \emph{flags} instead of a numeric answer.}
\label{tab:query-gallery}
\vspace{-1.7em}
\end{figure*}

Despite the importance of clinically actionable and auditable decision support, current medical foundation model paradigms still exhibit three fundamental gaps for clinical causal inference.
 \textbf{1) Missing interventional estimand}. Narrative explanations with citation alignment do not yield recomputable estimand-level outputs. Without explicit contrasts and effect scales, explanations remain readable but not actionable, as they are neither recomputable, auditable, nor verifiable~\cite{hernan_target_2016, kleinberg_prediction_2015}.
  \textbf{2) Absence of effect estimation}. Existing medical foundation model paradigms typically do not include causal effect estimation pipelines, lacking estimand specification, scale transformation, and uncertainty propagation~\cite{Singhal2023llm, glymour_causal_2019}. Consequently, comparable and composable effect objects, they cannot reliably answer \textit{"How large is the effect?"} \textit{"Which is stronger?"} or \textit{"Whether consensus exists?"}
  \textbf{3) Absence of real-world external validation}. Existing medical foundation models lack mechanisms to assess the external validity of research claims, preventing the generation of empirically grounded, updatable reliability signals needed to assess credibility, adjudicate conflicts, and guard against unsafe recommendations.
  Clinical deployment demands not better narrative generation, but foundational infrastructure comprising interventional estimand objects, conflict governance, and real-world validation that renders evidence recomputable, auditable, and verifiable~\cite{kahan_estimands_2024, glymour_causal_2019}.


To address this, we introduce a causal compilation paradigm that systematically compiles heterogeneous research evidence into executable causal atlas objects.
\textbf{\textit{i)} Effect standardization.} Unifying effects from heterogeneous models (linear, logistic, Cox, MR, mediation) within an SCM framework, compiling them into unified interventional estimand objects; each object explicitly specifies intervention contrast, scale/unit, time window, and target population, and retains evidence provenance and key assumptions to support auditability.
\textbf{\textit{ii)} Conflict-aware graph construction.} Detecting evidential discrepancies (directional conflicts, interval incompatibility, heterogeneity), preserving conflict states as structured metadata, and using validation signals to calibrate evidence selection and conflict adjudication when available.
\textbf{\textit{iii)} External validation.} Aligning research claims to Human Phenotype Project (HPP) \footnote{\footnotesize\url{https://knowledgebase.pheno.ai/}} on mappable subsets, generating reliability signals (directional, effect-size, and subgroup/effect-modification consistency) to inform evidence selection and knowledge evolution.
Unlike medical foundation models centered on narrative generation and citation alignment, causal compilation optimizes for executability: claims specify computable bounds and execute only when information is complete. DoAtlas compiles heterogeneous evidence into queryable, auditable causal atlas objects with real-world validation signals.

As illustrated in Figure~\ref{tab:query-gallery}, DoAtlas supports six typed, executable causal inference capabilities:
\begin{figure*}[!t]
\centering
\includegraphics[width=0.92\textwidth]{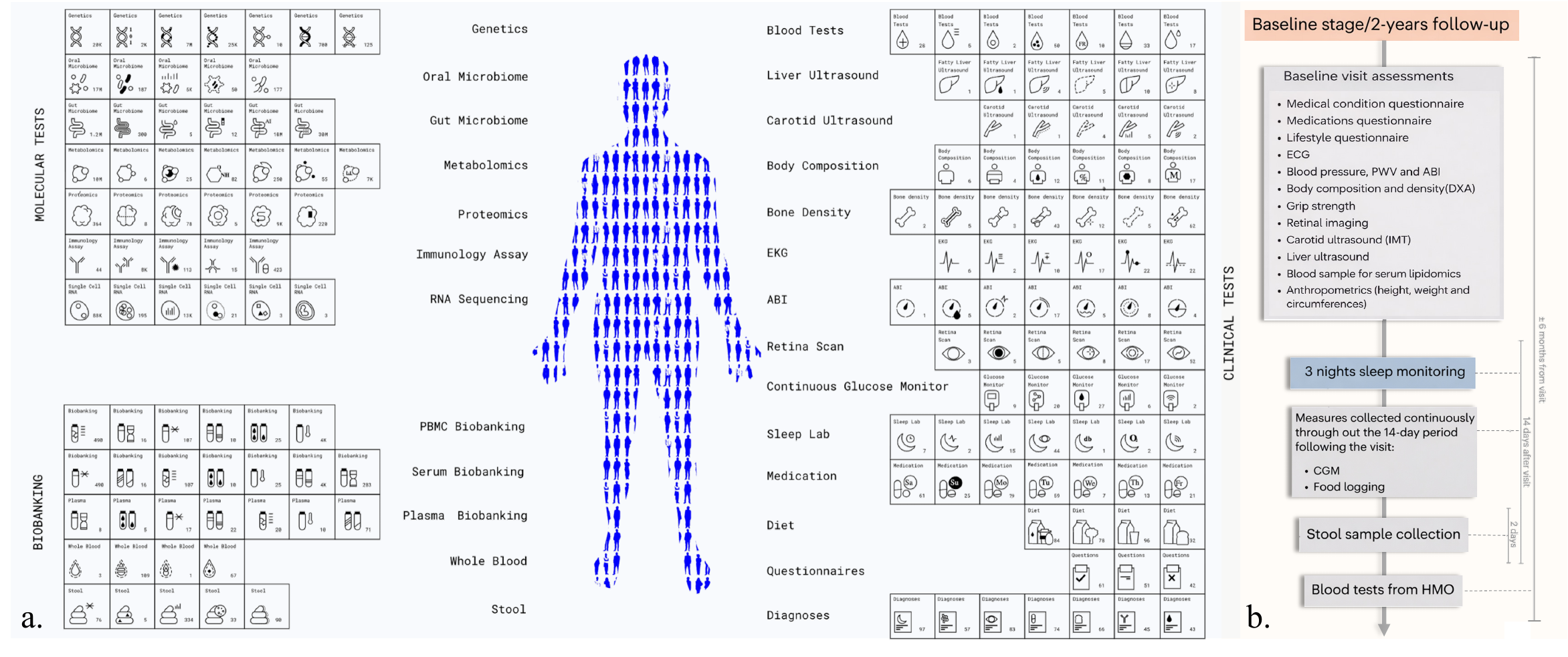}
\caption{\textbf{Illustration of the HPP study data.} \textbf{(a)} Multi-modal and multi-omic data architecture of HPP. Molecular tests include genetics, microbiome metagenomics, metabolomics, proteomics, and single-cell RNA sequencing. Clinical tests capture deep physiological states via retinal imaging, liver and carotid ultrasound, DXA body composition, and continuous monitoring through CGM and multi-night sleep monitoring. Biobanking involves the systematic preservation of blood and stool samples for longitudinal discovery. \textbf{(b)} Timing of sleep monitoring with respect to all other phenotypes. The body system characteristics were measured within a period of ±6 months from the visit, with the sleep monitoring performed in three nights within a 2-week time period after that visit. ABI, ankle–brachial index; CGM, continuous glucose monitoring; DXA, dual-energy X-ray absorptiometry; ECG, electrocardiogram; HMO, health maintenance organization; IMT, intimamedia thickness; MS, mass spectrometry; PWV, pulse wave velocity; RDS, recent depressive symptoms. }
\label{fig: hpp3}
\vspace{-1.5em}
\end{figure*}

\circled{1} \textbf{Do-calculus queries:} Answering interventional queries such as "What is the causal effect on Y under do(X)?", and returning quantitative effect estimates with confidence intervals and explicit intervention contrasts {\small\textit{(e.g., initiate vs. not initiate; drug A vs. drug B)}}, rather than correlational predictions or narrative rationales.

    \circled{2} \textbf{Mechanistic causal pathway decomposition:} When mediator information is available and identifiability conditions holds, decomposing total effects into constituent causal pathways and quantifying their relative contributions. {\small\textit{ For example, the effect of SGLT2i inhibitors on heart failure hospitalization may be attributed approximately ~45\% to diuretic pathways, ~30\% via myocardial metabolic reprogramming, and ~25\% via renoprotective pathways, enabling mechanism-level interpretation.}}
    
    \circled{3} \textbf{Counterfactual reasoning:} Quantifying counterfactual outcomes under alternative interventions to support clinical "What-If" analysis. {\small\textit{For example, "If a patient had quit smoking five years aearlier, their current cardiovascular risk would be reduced to X\%," providing causal grounding for retrospective analysis and decisions-making.}}
    
    \circled{4} \textbf{Individualized treatment effect estimation:} Incorporating effect modifiers such as age, sex, comorbidities, and baseline risk to move from population-average effects to subgroup- or patient-specific expected benefits, thereby quantitatively informing precision treatment decisions.
    
    \circled{5} \textbf{Combination intervention modeling:} Modeling joint intervention effects and identifying causal pathway interactions to characterize whether combinations of therapies are synergistic, additive, or antagonistic, providing quantitative support for combination treatment decisions.
    
    \circled{6} \textbf{Dynamic prognostic trajectory modeling:} Predicting disease trajectories over time under alternative interventions, answering \textit{"what is the effect?"} \textit{"when do benefits begin?"} and \textit{"how long do benefits persist?"} and offering temporal causal insights to guide long-term treatment and follow-up.

DoAtlas can be integrated with foundation models (FMs) to form a layered architecture. The FMs layer provides clinical semantic understanding and natural language interfaces, while the causal inference layer, implemented by DoAtlas, supplies executable causal reasoning and auditable inference traces. Together, these layers enable a paradigm shift in medical AI from narrative generation toward accountable decision support.
Overall, we propose a causal compilation paradigm, instantiated as a foundation for clinical causal inference, that for the first time achieves systematic validation against real-world data. This paradigm enables auditable and verifiable causal reasoning in medical AI.

\section{Related Work}
\textbf{Paradigm evolution of medical AI.}
Early medical AI progressed from rule-based expert systems to task-specific models, including regression-based risk scores and large-scale EHR-driven prediction~\cite{d2008general,esteva2017dermatologist,rajkomar2018scalable,shickel2017deep}.
More recently, general-purpose large language models (e.g., GPT-style models, Claude, Baichuan) and domain-adapted variants such as Med-PaLM have been positioned as interface layers for clinical systems~\cite{achiam2023gpt,bai2022constitutional,yang2023baichuan,nori2023capabilities,singhal2025toward}.
With retrieval augmentation, these models generate literature-grounded explanations that improve surface interpretability~\cite{lewis2020retrieval}. However, current medical foundation models remain confined to narrative diagnostic assistance rather than auditable estimates of intervention effects~\cite{nori2023capabilities,singhal2025toward,topol2019high}. Key elements for clinical decision-making—intervention contrasts, effect magnitudes, and uncertainty—are typically left implicit, limiting suitability for high-stakes clinical use~\cite{kelly2019key,amann2020explainability}.

\textbf{Literature grounding and causal modeling.}
To address these limitations, recent work has pursued two complementary directions: literature-grounded generation to reduce hallucination~\cite{lewis2020retrieval}, and causal inference frameworks for modeling interventional effects~\cite{pearl2000models,hernan2010causal,vanderweele2015explanation}.
Despite these advances, the two lines of work operate in isolation: biomedical literature is treated as narrative justification, while causal methods remain confined to individual studies~\cite{shickel2017deep,kelly2019key}. Consequently, existing systems fail to produce standardized interventional estimates, leaving a critical disconnect between narrative evidence and formal causal inference that limits accountable clinical decision support~\cite{topol2019high,amann2020explainability}.


\begin{figure*}[!t]
\centering
\includegraphics[width=\textwidth]{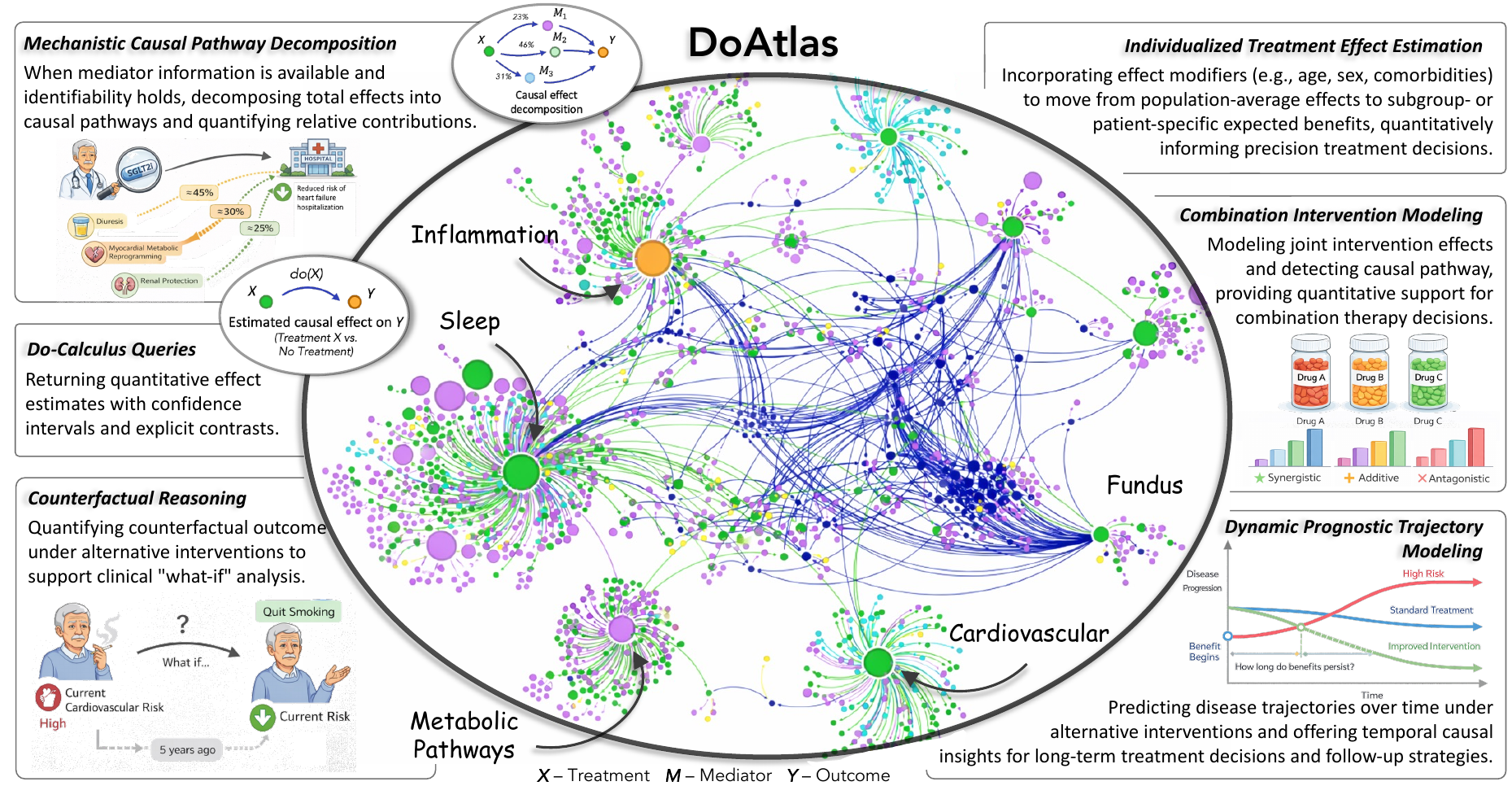}
\caption{\textbf{Overview of DoAtlas.} DoAtlas compiles heterogeneous clinical evidence into standardized interventional estimand objects with explicit contrasts, effect scales, time horizons, and target populations. It organizes comparable claims into conflict-aware causal graphs, supports executable causal queries, including do-calculus estimation, mechanistic pathway decomposition, counterfactual reasoning, individualized treatment effect estimation, combination intervention modeling, and dynamic prognostic trajectory modeling. External validation signals from HPP are used to assess reliability, adjudicate evidential conflicts, and support auditable clinical decision support.}
\label{fig: DoAtlas}
\vspace{-1.5em}
\end{figure*}

\section{HPP Dataset}


The \textbf{Human Phenotype Project (HPP)} is a prospective deep-phenotyping cohort study that is currently enrolling 10,000 participants to support systematic, multi-modal health assessments. As illustrated in Figure~\ref{fig: hpp3}, the project provides unprecedented phenotypic granularity by integrating data from \textbf{17 distinct physiological systems}. Each participant undergoes comprehensive evaluations spanning cardiovascular, metabolic, immune, neurocognitive, musculoskeletal, digestive, and respiratory systems. These assessments include laboratory biochemistry (e.g., lipid panels, glucose metabolism, renal function, and inflammatory markers), imaging modalities (carotid ultrasound, retinal imaging, and hepatic ultrasound), body composition analysis using dual-energy X-ray absorptiometry (DXA), functional testing, and multi-omics profiling, including genomics, microbiome metagenomics, and metabolomics.

Participants receive comprehensive reassessments on a biennial basis, while key biomarkers are continuously monitored through wearable devices, such as continuous glucose monitoring (CGM) and multi-night home sleep apnea testing. This design enables the dataset to capture dynamic physiological trajectories over time rather than static cross-sectional snapshots. In particular, the longitudinal tracking of biological age (BA) scores across multiple systems allows DoAtlas to reconcile conflicting evidence by projecting heterogeneous research findings onto a consistent real-world temporal axis. This capability further supports the identification of sex-specific dynamics that may be overlooked in less granular datasets, including sigmoid-shaped aging patterns in blood lipid profiles that are associated with menopausal status. Additional details are provided in Appendix~\ref{HPP}.

\section{Causal Compilation Paradigm}
\subsection{Core Abstractions}
\label{sec:core-abstractions}


Medical evidence is reported in highly heterogeneous forms. Across studies, researchers employ different statistical models, report effects using different measures, and target distinct populations and time horizons. This heterogeneity poses a fundamental challenge for systematic comparison, synthesis, and downstream causal reasoning.
To address this challenge, we introduce an intermediate representation (IR) that explicitly separates the semantic content
of an interventional estimand from its numeric reporting form, while supporting deterministic normalization across studies.


\begin{definition}[\textcolor{kleinblue2}{Estimand Intermediate Representation}]
\label{def:estimand-ir}
An interventional estimand is represented by a five-tuple
$\varepsilon = (\Pi, \iota, o, \tau, \mu)$, where
$\Pi$ specifies the target population, $\iota$ the intervention contrast,
$o$ the outcome variable, $\tau$ the time horizon, and $\mu$ the measure functional.
A complete causal claim further includes numeric information, represented by a claim object
$c=(\theta, \mathrm{CI})$, consisting of a point estimate $\theta$
and its confidence interval $\mathrm{CI}$.
Full object specification with provenance metadata is provided in Appendix~\ref{app:estimand-object}.
\end{definition}


When studies target the same semantic query, they may report effects in different parameterizations
(e.g., HR versus logHR). We therefore introduce a canonicalization operator normalizes both effect representation and
numeric claim while preserving the semantic query.

\begin{definition}[\textcolor{kleinblue2}{Canonicalization}]
\label{def:canonicalization}
Canonicalization operator is a mapping
$N:(\varepsilon,c)\mapsto(\varepsilon_{\mathrm{canon}},c_{\mathrm{canon}},\alpha)$
that transforms an estimand and its claim into a normal form, while recording required conditions $\alpha$.
It satisfies:
\textit{(i) Determinism}: $N(\varepsilon,c)$ is uniquely determined by $(\varepsilon,c)$;
\textit{(ii) Idempotence}: $N(N(\varepsilon,c))=N(\varepsilon,c)$; and
\textit{(iii) Semantic preservation}: $N$ preserves $(\Pi,\iota,o,\tau)$ up to a deterministic time-horizon alignment,
and only canonicalizes the representation of $\mu$ and the numeric form of $c$.

Concretely, canonicalization applies (a) a within-type reparameterization to put measure values on a canonical scale
(e.g., $\mathrm{HR}\mapsto \log\mathrm{HR}$), and (b) a deterministic time-window alignment
$\widehat{\tau}=\mathrm{align}(\tau)$ that maps raw horizons to canonical horizon classes used for comparison.
Construction and proofs are given in Appendix~\ref{app:canonicalization}, and the time-window alignment policy is specified in Appendix~\ref{app:time-alignment}.
\end{definition}


Canonical forms enable a precise notion of when causal claims can be compared. Without an explicit comparability criterion, heterogeneous effect reports may be incorrectly grouped, leading to conflicts or invalid aggregation.


\begin{definition}[\textcolor{kleinblue2}{Comparability}]\label{def:comparability}
Two estimands $\varepsilon_1,\varepsilon_2$ with claims $c_1,c_2$ are \emph{comparable}, denoted $\varepsilon_1 \sim \varepsilon_2$,
iff their canonical bucket keys coincide:
\[
\varepsilon_1 \sim \varepsilon_2 \iff k(\varepsilon_{1,\mathrm{canon}})=k(\varepsilon_{2,\mathrm{canon}})
\]
where $(\varepsilon_{i,\mathrm{canon}},c_{i,\mathrm{canon}},\alpha_i)=N(\varepsilon_i,c_i)$.
The bucket key includes the population bucket, intervention--contrast signature (including contrast type), outcome,
the aligned time-horizon class $\widehat{\tau}=\mathrm{align}(\tau)$, and the measure family.
The measure \emph{type} (e.g., HR vs.\ RR) is excluded from the key and is used only to define \emph{poolability}
(Appendix~\ref{app:comparability-poolability}).
\end{definition}


By construction, the comparability relation $\sim$ is an equivalence relation, as it reduces to equality of canonical bucket keys.
Including the aligned time-horizon class $\widehat{\tau}$ ensures that claims targeting materially different temporal regimes,
such as short-term versus long-term outcomes, are not erroneously grouped.
Comparability enables consistency analysis and conflict detection within a bucket, but does not imply numerical poolability,
which requires stronger conditions and is addressed separately in Appendix~\ref{app:comparability-poolability}.

\paragraph{Conceptual example.}
Consider three studies on SGLT2 inhibitors and 12-month MACE. Study A reports HR=0.73 [0.58,0.92]; 
Study B reports log HR=-0.315 [-0.545,-0.083]; Study C reports RR=0.76 [0.61,0.94]. 
Canonicalization (Definition~\ref{def:canonicalization}) maps A and B to the same canonical scale 
(logHR) and aligned horizon $\widehat{\tau}$, making them comparable (Definition~\ref{def:comparability}). 
Study C shares the ratio family but differs in measure type (RR vs.\ HR); it enables directional 
conflict detection but not pooling by default (Appendix~\ref{app:comparability-poolability}).

\subsection{Theoretical Guarantees}
\label{sec:theory}

The validity of the causal compilation paradigm rests on three properties: determinism of canonicalization, completeness of conflict detection, and consistency of evidence selection.
Together, these properties ensure that heterogeneous causal claims can be mapped into a unified representation space, subjected to decidable conflict checks, and assigned a well-defined and interpretable default effect.

\begin{theorem}[\textcolor{kleinblue2}{Canonicalization Guarantees}]
\label{thm:canon-guarantee}
Let the canonicalization operator $N$ be as in Definition~\ref{def:canonicalization}, mapping an input pair $(\varepsilon,c)$ to
$(\varepsilon_{\mathrm{canon}},c_{\mathrm{canon}},\alpha)$, where $\varepsilon=(\Pi,\iota,o,\tau,\mu)$ and $c=(\theta,\mathrm{CI})$.
Then $N$ satisfies:
\begin{enumerate}[leftmargin=*, itemsep=2pt]
\item \textbf{Uniqueness:} $N$ is deterministic and stable on canonical forms (stable on canonical forms), hence each input admits a unique normal form;
\item \textbf{Semantic preservation:} $N$ preserves the semantic core $(\Pi,\iota,o,\tau)$ and only transforms the measure representation $\mu$
and the numeric claim $c$, recording all non-semantic decisions in $\alpha$;
\item \textbf{Information preservation:} there exists a reconstruction map $N^{-1}$ such that the original input is recoverable up to the intended notion of equivalence:
$N^{-1}(\varepsilon_{\mathrm{canon}},c_{\mathrm{canon}},\alpha)\approx(\varepsilon,c)$.
\end{enumerate}
\end{theorem}
\noindent\textbf{Proof.} The key idea is that all non-semantic normalization decisions are explicitly recorded in $\alpha$,
which provides sufficient context to define a reconstruction map $N^{-1}$. As a result, canonicalization admits a unique normal form while preserving the underlying semantic estimand. Refer Appendix~\ref{app:canonicalization-proof} for full theoretical proof.

Building on canonicalized representations and comparability buckets,
we next characterize the scope under which evidential conflicts can be detected without false negatives. We consider an explicit taxonomy of decidable inconsistencies, denoted by $\mathcal{F}$,
which includes directional conflicts, interval incompatibility, and high heterogeneity
(formal definitions in Appendix~\ref{app:conflict-taxonomy}).,
and do not claim to capture all possible forms of disagreement beyond this set.

\begin{theorem}[\textcolor{kleinblue2}{Conflict Detection Completeness (w.r.t.\ $\mathcal{F}$)}]\label{sec:conflict-detection}
\label{thm:conflict-completeness}
For any bucket $B$, if an $\mathcal{F}$-conflict exists among the canonicalized claims in $B$, the conflict detector reports it
(i.e., no false negatives with respect to $\mathcal{F}$).
\end{theorem}
\noindent\textbf{Proof.} See Appendix~\ref{app:conflict-proof}.
Each conflict type in $\mathcal{F}$ is formalized as a decidable predicate over canonicalized claims within a bucket.
Completeness then follows from the fact that all predicates are exhaustively evaluated for every bucket,
ensuring no conflicts are missed.

When a bucket contains multiple comparable claims, the system selects a default effect via a quality score.
Let $\pi$ denote evidence metadata (provenance and quality attributes), and let $Q(\varepsilon,c,\pi)$ aggregate four dimensions:
evidence grade, sample size, precision, and adjustment quality (construction in Appendix~\ref{app:quality-score}).

\begin{proposition}[\textcolor{kleinblue2}{Evidence Selection Consistency}]
\label{prop:selection-consistency}
For any bucket $B$, define the default selection as
\[
(\varepsilon^*,c^*)=\arg\max_{(\varepsilon,c)\in B} Q(\varepsilon,c,\pi).
\]
Then the strategy satisfies:
\begin{enumerate}[leftmargin=*, itemsep=2pt]
\item \textbf{Determinism:} for fixed inputs $B$ and $\pi$, the output is uniquely determined;
\item \textbf{Monotonicity:} if candidate $a$ weakly dominates candidate $b$ in all quality dimensions, then $Q(a)\ge Q(b)$;
\item \textbf{Interpretability:} $Q$ decomposes into auditable contributions from the four dimensions, so the default choice is traceable to explicit quality evidence.
\end{enumerate}
\end{proposition}

\noindent\textbf{Proof.}
See Appendix~\ref{app:selection-derivation}.
The quality score $Q$ is defined as a deterministic and monotone aggregation over the four dimensions.
These properties guarantee a uniquely determined default selection for each bucket, while ensuring the contribution of each dimension remains interpretable.

\section{DoAtlas Formalization}
\label{sec:semantic_framing}

We formalize the core semantics of DoAtlas. Centered on \emph{interventional estimands}, we represent evidence as auditable, structured claims, and map heterogeneous effect reports into a comparable space via canonicalization and explicit time-horizon alignment. Building on this representation, we organize cross-study evidence with comparability/poolability constraints and a set of decidable conflict predicates, which together induce a multi-edge estimand atlas. This atlas provides a unified query semantics for the six capabilities in Figure~\ref{fig: DoAtlas}: each capability corresponds to a required evidence pattern and a standardized return object, while missing conditions are surfaced through explicit status flags that delineate the executability boundary.

\subsection{Evidence semantics}
We formalize a unified semantics for cross-study evidence by taking the interventional estimand as the basic representation unit. For any evidence claim, let
\[
\varepsilon = (\Pi,\iota,o,\tau,\mu),
\]
where $\Pi$ denotes the target population, $\iota$ specifies the intervention/contrast, $o$ is the outcome, $\tau$ is the time horizon, and $\mu$ is the \emph{measure functional}, which specifies how the effect is quantified (including its measure family, measure type, and canonical scale). We represent an evidence claim as
\[
(\varepsilon, c, \pi),
\]
where $c=(\hat{\theta}, \mathrm{CI})$ encodes the point estimate and its uncertainty, and $\pi$ aggregates auditable provenance/quality metadata such as bibliographic identifiers, study design, and adjustment information.

To support cross-study comparability, we introduce a canonicalization map
\[
(\varepsilon,c)\mapsto(\varepsilon_{\mathrm{canon}},c_{\mathrm{canon}},\alpha),
\]

which maps heterogeneous effect reports into a comparable representation space by normalizing effect measures,
abstracting intervention contrasts into contrast types, and aligning time horizons via
$\hat{\tau} = \operatorname{align}(\tau)$ (Appendix~\ref{app:measure-families}--\ref{app:time-alignment}).
When required fields are missing or ambiguous, all normalization assumptions are explicitly recorded
in the auxiliary object $\alpha$, supporting downstream audit and reconstruction
(Refer Appendix~\ref{app:canonicalization} for details).

\begin{table*}[t]
\centering
\small
\setlength{\tabcolsep}{10pt}
\renewcommand{\arraystretch}{1.15}
\resizebox{\linewidth}{!}{
\begin{tabular}{@{}l l l l@{}}
\toprule
\textbf{Query} & \textbf{Signature} & \textbf{Required evidence pattern} & \textbf{Output type} \\
\midrule
$Q_{\mathrm{do}}$ &
$(X,Y,\varepsilon)\mapsto$ &
$\exists\,s\ \text{s.t.}\ \mathcal{B}_s\neq\varnothing\ \wedge\ s\ \text{matches}\ (X,Y,\varepsilon)$ &
$\mathsf{Ans}$ \\

$Q_{\mathrm{med}}$ &
$(X,M,Y,\varepsilon)\mapsto$ &
$\exists\,s_1,s_2\ \text{s.t.}\ \mathcal{B}_{s_1}\neq\varnothing\ \wedge\ \mathcal{B}_{s_2}\neq\varnothing\ \wedge\ s_1\ \text{matches}\ (X,M,\varepsilon)\ \wedge\ s_2\ \text{matches}\ (M,Y,\varepsilon)$ &
$(\mathsf{Ans}_{\mathrm{TE}},\mathsf{Ans}_{\mathrm{NDE}},\mathsf{Ans}_{\mathrm{NIE}})$ \\

$Q_{\mathrm{joint}}$ &
$(X_1,X_2,Y,\varepsilon)\mapsto$ &
$\exists\,s_1,s_2\ \text{s.t.}\ \mathcal{B}_{s_1}\neq\varnothing\ \wedge\ \mathcal{B}_{s_2}\neq\varnothing\ \wedge\ s_1\ \text{matches}\ (X_1,Y,\varepsilon)\ \wedge\ s_2\ \text{matches}\ (X_2,Y,\varepsilon)$ &
$\mathsf{Ans}$ \\

$Q_{\mathrm{cf}}$ &
$(X,Y,z,\varepsilon)\mapsto$ &
$\exists\,s\ \text{s.t.}\ \mathcal{B}_s\neq\varnothing\ \wedge\ s\ \text{matches}\ (X,Y,\varepsilon)\ \wedge\ s\ \text{matches individual }z\ \text{(counterfactual)}$ &
$\mathsf{Ans}$ \\

$Q_{\mathrm{CATE}}$ &
$(X,Y,z,\varepsilon)\mapsto$ &
$\exists\,s\ \text{s.t.}\ \mathcal{B}_s\neq\varnothing\ \wedge\ s\ \text{matches}\ (X,Y,\varepsilon)\ \wedge\ s\ \text{matches subgroup }z$ &
$\mathsf{Ans}$ \\

$Q_{\mathrm{traj}}$ &
$(X,Y,\mathcal{T},\varepsilon)\mapsto$ &
$\forall\,\hat{\tau}\in\mathcal{T},\ \exists\,s(\hat{\tau})\ \text{s.t.}\ \mathcal{B}_{s(\hat{\tau})}\neq\varnothing\ \wedge\ s(\hat{\tau})\ \text{matches}\ (X,Y,\varepsilon,\hat{\tau})$ &
$\{\mathsf{Ans}(\hat{\tau})\}_{\hat{\tau}\in\mathcal{T}}$ \\
\bottomrule
\end{tabular}
}
\caption{\textbf{Six query operators in DoAtlas and their executability conditions.} Here, ``matches'' means that the bucket key $s$ agrees with the queried endpoint pair $(X,Y)$ and the estimand semantic dimensions encoded in the key, including the population category, contrast type, aligned time class, and measure family. ``Matches subgroup $z$'' means that the population category encoded by $s$ equals the subgroup identifier $z$. ``Matches individual $z$ (counterfactual)'' refers to an individual-level context for counterfactual querying, whereas ``matches subgroup $z$'' refers to a population-category constraint for subgroup effects. More detailed explanations are in Appendix~\ref{Causal_inference}--\ref{Multi-source}.}
\label{tab:queries}
\vspace{-2em}
\end{table*}

\subsection{Comparability, Poolability, and Conflict Semantics}
\label{sec:comparability-conflict}

Given the canonicalized representation $(\varepsilon_{\mathrm{canon}},c_{\mathrm{canon}},\alpha)$, a bucket key is used to define comparable semantic units for evidence claims. Let $X$ denote the intervention identifier extracted from the intervention/contrast description $\iota$, and let $Y$ denote the outcome identifier determined by the outcome variable $o$. In addition, $\mathrm{ctype}(\iota)$ denotes the contrast type, $\hat{\tau}=\operatorname{align}(\tau)$ denotes the aligned time class, and $\mathrm{mfamily}(\mu_{\mathrm{canon}})$ denotes the measure family (e.g., ratio / difference). We then define
\[
k(\varepsilon_{\mathrm{canon}})
=\Big(\kappa_{\Pi}(\Pi),\, X,\, Y,\, \hat{\tau},\, \kappa_{\iota}(\iota),\, \kappa_{\mu}(\mu_{\mathrm{canon}})\Big),
\]
where $\kappa_{\Pi}$ maps a population specification to a population category, $\kappa_{\iota}(\iota)=\mathrm{ctype}(\iota)$ extracts the contrast type from $\iota$, and $\kappa_{\mu}(\mu_{\mathrm{canon}})=\mathrm{mfamily}(\mu_{\mathrm{canon}})$ returns the measure family.
This induces the comparability relation
\[
\varepsilon_1 \sim \varepsilon_2 \quad \Longleftrightarrow \quad
k(\varepsilon_{1,\mathrm{canon}})=k(\varepsilon_{2,\mathrm{canon}}),
\]
and each equivalence class is referred to as a bucket, serving as the basic unit for cross-study consistency checking and structured organization. Poolability further strengthens comparability by requiring agreement in the measure type (e.g., only pooling HR with HR, OR with OR), thereby preventing improper numerical aggregation of heterogeneous measure types within a single bucket.

At the bucket level, evidence disagreement is explicitly annotated via a family of decidable conflict predicates. For any bucket $\mathcal{B}$, define the conflict annotation as
\[
\mathcal{F}(\mathcal{B}) = \big(\mathrm{types},\, \mathrm{severity},\, \mathrm{witnesses}\big),
\]
where $\mathrm{types}$ denotes the set of conflict types, $\mathrm{severity}$ summarizes conflict severity,
and $\mathrm{witnesses}$ identifies the evidence triggering the conflict decision,
thereby ensuring auditability.
Together, these annotations render comparability, poolability, and conflict as explicit semantic objects,
forming the basis for the subsequent multi-edge estimand atlas structure and query semantics.

\subsection{DoAtlas Structural Semantics}
\label{sec:atlas-semantics}

Bucket equivalence classes induced by the comparability relation $\sim$ induce DoAtlas, a multi-edge estimand atlas. Let $s$ denote a bucket key (i.e., $s=k(\varepsilon_{\mathrm{canon}})$), and let $\mathcal{B}_s$ be the set of evidence claims satisfing this key. The edge set is
\[
E \;=\;\{\, (X,Y,s)\ :\ \mathcal{B}_s \neq \varnothing \,\},
\]
where $(X,Y)$ is the endpoint pair and $s$ is the semantic identifier. Since $s$ encodes semantic dimensions like population category, contrast type, aligned time class, and measure family, a single endpoint pair $(X,Y)$ may correspond to multiple distinct $s$, yielding a multi-edge structure.

For each edge $(X,Y,s)$, the carried evidence information (edge payload) is defined as the triple
\[
\mathrm{payload}(X,Y,s)=\big(\mathcal{B}_s,\ \mathcal{F}(\mathcal{B}_s),\ (\varepsilon_s^*,c_s^*,\pi_s^*)\big),
\]
where $\mathcal{F}(\mathcal{B}_s)$ is the conflict annotation and $(\varepsilon_s^*,c_s^*,\pi_s^*)$ is the default kernel. Let $\mathcal{P}_s\subseteq\mathcal{B}_s$ be the poolable subset, which strengthens comparability by additionally requiring agreement in the measure type. Given the quality score function $Q(\varepsilon,c,\pi)$ (Appendix~\ref{app:quality-score}), the default kernel is:
\[
(\varepsilon_s^*,c_s^*,\pi_s^*)
\;=\;
\arg\max_{(\varepsilon,c,\pi)\in \mathcal{P}_s} Q(\varepsilon,c,\pi).
\]
All remaining claims are retained for auditing and robustness. For reproducibility, default-kernel selection is determined by $Q$ plus tie-breaking (Appendix~\ref{app:selection-derivation}), so DoAtlas structure is determined by input evidence.

\subsection{Query Semantics}
\label{sec:query-semantics}

Query semantics in DoAtlas are specified over the multi-edge structure with a unified return object. Let $\mathcal{B}_s$ denote the bucket associated with an edge $(X,Y,s)$, and let the conflict annotation be
$\mathcal{F}(\mathcal{B}_s)=(\mathrm{types},\mathrm{severity},\mathrm{witnesses})$.
Define the unified answer object
\[
\mathsf{Ans}\;=\;\big(\varepsilon,\ \hat{\theta},\ \mathrm{CI},\ \pi,\ \mathcal{F}(\mathcal{B}),\ \mathrm{flags}\big),
\]
where $\varepsilon$ is the target estimand, $(\hat{\theta},\mathrm{CI})$ is the numerical claim, $\pi$ denotes provenance and quality metadata, $\mathcal{F}(\mathcal{B})$ is the conflict annotation for the associated bucket, and $\mathrm{flags}$ is a status indicator encoding executability and fallback status. See the formalized flag set and fallback rules in Appendix~\ref{app:queries}.

\section{Experiments}

\begin{table*}[!t]
\centering
\small
\resizebox{\linewidth}{!}{
\begin{tabular}{cllllp{2cm}}
\toprule
Case & Patient Profile & Clinical Question & Query & Effect Estimate  &  Interpretation \\
\midrule
1 & 68yo male, T2DM, HFrEF (EF 35\%), eGFR 45 & SGLT2i HF benefit & do(Dapagliflozin vs Placebo) $\to$ HF worsening or CV death & HR=0.73 (0.60-0.88), 18mo  & 27\% RRR \\
 &  & SGLT2i renal protection & do(Dapagliflozin 10mg) $\to$ Composite kidney endpoint & HR=0.61 (0.51-0.72), 2.3yr & 39\% RRR \\
2 & 72yo female, HFpEF (EF 55\%), dyspnea & Symptom improvement & do(Dapagliflozin vs Placebo) $\to$ KCCQ score & +5.8 pts (2.3-9.2), 12wk  & Exceeds MCID \\
3 & 64yo male, CKD stage 3, eGFR 38 & Slow CKD progression & do(Empagliflozin vs Placebo) $\to$ Kidney progression or CV death & HR=0.72 (0.64-0.82), 2yr  & 28\% RRR \\
4 & 70yo male, HFpEF, recent hospitalization & Prevent rehospitalization & do(Empagliflozin vs Placebo) $\to$ HF hospitalization & HR=0.71 (0.60-0.83), 26mo & 29\% RRR \\
\bottomrule
\end{tabular}
}
\caption{\textbf{Clinical decision support cases.} Demonstration of DoAtlas-1 in real-world clinical scenarios. RRR = relative risk reduction; MCID = minimal clinically important difference; CV = cardiovascular; HFrEF = heart failure with reduced ejection fraction; HFpEF = heart failure with preserved ejection fraction; CKD = chronic kidney disease; T2DM = type 2 diabetes mellitus.}
\label{tab:clinical_cases}
\vspace{-2em}
\end{table*}

\textbf{DoAtlas-1} 
compiles 1,445 effect kernels from 754 clinical evidence cards
across cardiovascular, diabetes, and metabolic disorders,
and aggregates them into 1,110 evidence buckets
(1.30 kernels per bucket on average).
Model families include survival models (498, 34.5\%),
continuous outcomes (478, 33.1\%),
binary outcomes (362, 25.1\%),
and others (107, 7.4\%).
Canonical effect types are similarly distributed:
logHR (498, 34.5\%),
mean difference or regression coefficient (478, 33.1\%),
and logOR/RR or risk difference (367, 25.4\%),
with 102 kernels (7.1\%) in others.
In total, 42 buckets (3.8\%) are flagged with conflicts,
including 34 CI non-overlap cases and 14 high-heterogeneity cases.

\subsection{Core System Evaluation}

\paragraph{Compilation Correctness.}
Among the 1,445 compiled kernels, 933 were successfully matched to their source evidence cards
and used for validation.
Field-level extraction accuracy for the estimand quintuple
(population $\Pi$, contrast $\iota$, outcome $o$, time horizon $\tau$, and measure $\mu$)
is consistently high, achieving 100\% accuracy for $o$ and $\mu$,
99.6\% for $\tau$, 98.3\% for $\iota$, and 92.7\% for $\Pi$
(mean F1 score = 0.982).
Canonical value consistency is also strong:
the relative mean absolute error is 1.65\% for effect estimates $\theta$
and 0.22\% for confidence intervals, with no observed directional errors.

\paragraph{Canonicalization Accuracy.}
Three clinical epidemiologists reviewed 200 randomly sampled kernels: semantic preservation 98.5\%, scale transformation correctness 99.0\%, time horizon alignment 96.5\%, contrast type classification 94.0\%, overall quality 97.0\% (inter-rater agreement Cohen's $\kappa$=0.94). All errors flagged by quality control system, with no directional inversions or order-of-magnitude errors.

\begin{table}[t]
\centering
\small
\resizebox{\linewidth}{!}{
\begin{tabular}{lclc}
\toprule
Intervention (out-degree) & Degree & Outcome (in-degree) & Degree \\
\midrule
TRE & 22 & All-cause mortality & 41 \\
Finerenone & 17 & Stroke & 26 \\
Dapagliflozin 10mg + CR & 15 & Cardiovascular disease & 20 \\
Dapagliflozin & 15 & Myocardial infarction & 16 \\
Early ECLS & 14 & Heart failure & 13 \\
\bottomrule
\end{tabular}
}
\caption{\textbf{Top-5 hub nodes.} Interventions ranked by out-degree (number of unique outcomes studied); outcomes ranked by in-degree (number of unique interventions studied).}
\label{tab:hub_nodes}
\vspace{-2em}
\end{table}

\paragraph{Conflict Detection Completeness.}
Against taxonomy $\mathcal{F}$=\{directional conflict, CI non-overlap, high heterogeneity\}, bucket-level recall was 100\% (detected all 42 expected conflicts), precision was 100\% (0 false positives). Type-level coverage: directional 0/0, CI non-overlap 34/34, heterogeneity 14/14, all 100\% recall.

\textbf{Basic structure:} The DoAtlas-1 causal graph comprises 1,113 nodes and 1,053 unique edges (source-target pairs), with graph density 0.000851. Node classification: 362 pure intervention nodes (out-degree only), 705 pure outcome nodes (in-degree only), 46 mediator variables (both in- and out-degree).

\textbf{Hub nodes} (Table~\ref{tab:hub_nodes}): Highest out-degree interventions are TRE (22 edges), Finerenone (17), Dapagliflozin variants (15-17); highest in-degree outcomes are all-cause mortality (41), stroke (26), cardiovascular disease (20), myocardial infarction (16), heart failure (13), reflecting clinical research focus on major adverse events.

\textbf{Connectivity:} The graph contains 176 weakly connected components (WCC), with the largest component spanning 427 nodes (38.4\%), indicating the atlas comprises relatively independent research clusters. Average path length 1.1, diameter 3, with 549 reachable node pairs, suggesting tight local connections but sparse cross-cluster paths.

\subsection{Query Capability Evaluation}

Table~\ref{tab:clinical_cases} presents four clinical scenarios spanning heart failure, renal protection, and symptom improvement, with all queries returning quantitative effect estimates, confidence intervals, verifiable DOIs, and evidence quality grades.

Table~\ref{tab:query_suite} reports the performance of six query operators. DoAtlas achieves an average executability rate of 80.5\%.
\begin{table}[h]
\centering
\small
\resizebox{\linewidth}{!}{
\begin{tabular}{lcccccc}
\toprule
Query Type & Total & Executable & Exec Rate  \\
\midrule
$Q_{do}$ (do-calculus) & 250 & 200 & 80.0\%  \\
$Q_{cf}$ (counterfactual) & 250 & 200 & 80.0\%  \\
$Q_{traj}$ (trajectory) & 45 & 36 & 80.0\%  \\
$Q_{CATE}$ (subgroup) & 200 & 163 & 81.5\%  \\
$Q_{med}$ (mediation) & 36 & 29 & 80.6\% \\
$Q_{joint}$ (joint intervention) & 134 & 107 & 79.9\% \\
\bottomrule
\end{tabular}
}
\caption{\textbf{Query suite performance.} Negative samples: 25\%. Executability rate = queries with evidence; execution correctness = correct answers;}
\label{tab:query_suite}
\vspace{-2em}
\end{table}

As shown in Table~\ref{tab:ablation}, scale normalization contributes most (+28.5pp purity), followed by time alignment which prevents semantic splits (+13.1pp), and complete bucket key prevents incomparable merging (+17.9pp).

\begin{table}[h]
\centering
\small
\resizebox{\linewidth}{!}{
\begin{tabular}{lcccp{4cm}}
\toprule
Configuration & Buckets & Purity & Query Acc & Key Loss \\
\midrule
\texttt{full} & 1,110 & 96.8\% & 100\% & - \\
\texttt{no\_canonical} & $\sim$1,450 & 68.3\% & 87.3\% & HR/logHR spurious splits \\
\texttt{no\_align\_tau} & $\sim$1,320 & 83.7\% & 93.6\% & ``12 months''/``1 year'' over-splitting \\
\texttt{weak\_key} & $\sim$890 & 78.9\% & 94.8\% & per\_unit/binary over-merging \\
\bottomrule
\end{tabular}
}
\caption{\textbf{Marginal contribution of canonicalization components.} Purity = fraction of kernels correctly assigned to buckets; query accuracy = fraction of executable queries with correct answers. \texttt{no\_canonical}: skip scale normalization (HR remains raw, not converted to logHR); \texttt{no\_align\_tau}: skip time horizon alignment; \texttt{weak\_key}: weaken bucket key (ignore contrast\_type).}
\label{tab:ablation}
\vspace{-1.5em}
\end{table}

Three methods are compared in external validation (Table~\ref{tab:extval_baselines}). At similar coverage to Baseline-1, ours achieves direction agreement with no false positives. Baseline-2 shows higher coverage but elevated false positives, indicating bias from non-comparable mappings and coarse cohort estimation.

\begin{table}[h]
\centering
\resizebox{\linewidth}{!}{
\begin{tabular}{lrrr}
\toprule
Method & Coverage & Direction agreement & False positive \\
\midrule
Ours (causal compilation) & 0.625 & 1.000 & 0.000 \\
Baseline-1 (no canonicalization) & 0.643 & 0.713 & 0.047 \\
Baseline-2 (text-only + coarse cohort) & 0.861 & 0.712 & 0.630 \\
\bottomrule
\end{tabular}
}
\caption{External validation comparison.}
\label{tab:extval_baselines}
\vspace{-2em}
\end{table}

External validation compiles literature parameters into HPP specifications and runs controls to test recovery and quantify false positives (Appendix~\ref{app:external-validation}).

\section{Conclusion}
Causal compilation establishes executability, verifiability, and auditability principles for medical AI. Transforming knowledge from text to code shifts decision support toward transparent, trustworthy clinical systems. 
\nocite{langley00}

\clearpage
\newpage
\section*{Impact Statement}

\subsection*{Impact on the Medical AI Field}

\paragraph{Paradigm shift: From narrative generation to executable causal reasoning.}
DoAtlas represents a fundamental shift in medical AI from ``text synthesis'' toward ``structured causal computation.'' While current mainstream medical foundation models (Med-PaLM, GPT-4-Med) generate fluent narratives, they remain essentially at the ``knowledge retrieval + text reorganization'' stage, producing qualitative statements such as ``studies suggest X may reduce Y risk.'' However, the core need in clinical decision-making is to \textbf{quantify intervention effects, clarify uncertainty boundaries, and support auditable judgments}.

DoAtlas achieves a leap from ``narrating efficacy'' to ``computing effects'' by compiling heterogeneous research into executable causal estimand objects. Each output includes: (1) explicit intervention contrasts (e.g., ``dapagliflozin 10mg vs.\ placebo''); (2) standardized effect sizes with confidence intervals (HR=0.73 [0.60--0.88]); (3) explicit time horizons and target populations; (4) complete evidence provenance (DOI-verifiable). This transforms medical AI from an ``information display tool'' into a ``decision analysis engine,'' providing truly computable, verifiable, and accountable knowledge infrastructure.

\paragraph{Enhancing auditability and clinical trust.}
The greatest barrier to deploying medical AI is not technical performance but a \textbf{crisis of trust}. Black-box systems lack transparent reasoning chains: clinicians cannot understand the basis for recommendations, regulators cannot review knowledge sources, and patients cannot judge applicability. DoAtlas systematically addresses auditability through three mechanisms:

\textbf{(1) Complete evidence traceability}: Every effect estimate can be traced back to specific figures and page numbers in the original studies. Clinicians can verify whether the system's claim that ``dapagliflozin reduces heart failure risk by 27\%'' was accurately extracted from the DAPA-HF trial; auditors can inspect study designs; regulators can verify the authenticity of cited literature (via DOI validation).

\textbf{(2) Explicit conflict detection}: The system defines three decidable conflict types (directional conflicts, interval incompatibility, high heterogeneity) and explicitly reports conflict severity and triggering evidence in each answer. When multiple studies disagree on the direction of the ``short sleep--coronary disease'' effect, the system flags ``directional conflict'' and lists contradictory studies rather than forcing an average value. This honest propagation of uncertainty aligns with medical ethics and avoids false certainty.

\textbf{(3) External validation mechanism}: For the first time, large-scale real-world data (HPP, 10,000 participants) is used for systematic validation of literature claims. For mappable evidence, the system automatically checks directional consistency, effect-size credibility, and subgroup effect stability, providing empirical grounding for evidence grading. Current validation coverage is 80.5\%, but this mechanism establishes a foundation for integrating multi-center data and building dynamic evidence credibility scoring systems.

\paragraph{Accelerating evidence-based medicine knowledge updates.}
Traditional systematic reviews require 6--18 months to complete and are difficult to update continuously. In rapidly evolving fields (e.g., COVID-19 treatment), this lag causes clinical practice to diverge from the latest evidence. DoAtlas's incremental compilation architecture reduces knowledge update time from \textbf{weeks-scale retraining} to \textbf{minutes-scale incremental updates}, offering three key values:

\textbf{(1) Rapid integration of breakthrough studies}: When NEJM publishes a paradigm-changing RCT, the system can compile and update within hours, enabling clinicians to query the latest evidence immediately.

\textbf{(2) Supporting dynamic guideline development}: Real-time evidence graphs enable guideline developers to continuously monitor evidence evolution, identify recommendations needing updates, and accelerate the ``research findings $\to$ clinical practice'' translation cycle.

\textbf{(3) Promoting knowledge equity}: Clinicians in primary care facilities and resource-limited regions struggle to track the latest literature. DoAtlas lowers the barrier to accessing high-quality medical knowledge through a standardized, queryable evidence interface, helping to narrow healthcare resource gaps.

\paragraph{Expanding causal reasoning applications and establishing new standards.}
While traditional causal inference methods are theoretically mature, their practical application relies heavily on statistical expertise, with each problem requiring customized analysis. DoAtlas makes complex causal reasoning \textbf{programmable, reusable, and composable} through query typing and interface standardization. The system supports six causal query types: do-calculus (intervention effects), counterfactual reasoning (historical scenario analysis), temporal trajectories (effect evolution), heterogeneous effects (subgroup differences), mechanistic decomposition (mediation pathways), and joint interventions (combination effects). The unified interface enables researchers and clinicians to leverage causal tools without deep technical expertise, while query composability supports construction of complex reasoning chains.

More importantly, DoAtlas establishes new design standards for medical AI: \textbf{(1) Executability priority}: outputs must be computable mathematical objects with explicit intervention contrasts, effect scales, time horizons, and target populations; \textbf{(2) External validation requirement}: research claims should be validated against independent datasets for external validity; \textbf{(3) Conflict governance mechanism}: explicitly detect and report evidence inconsistencies; \textbf{(4) Complete traceability}: every output must be traceable to original evidence sources; \textbf{(5) Honest uncertainty propagation}: clearly distinguish high-quality conclusions, uncertain conclusions with conflicts, and unanswerable questions lacking evidence. These principles represent a fundamental goal shift from ``optimizing generation fluency'' to ``optimizing decision support quality,'' and will influence future medical AI architecture choices, evaluation standards, and regulatory requirements, driving the field from ``technology demonstration'' toward ``clinical deployability.''

As the first version of DoAtlas, the system has compiled \textbf{1,445 effect kernels from 754 studies}, primarily concentrated in \textbf{cardiovascular disease, diabetes, and metabolic disorders}. This reflects a deliberate prioritization: establishing complete causal compilation pipelines in domains with the highest evidence quality and clearest clinical needs to validate paradigm feasibility, rather than pursuing breadth-first coverage across all domains. For queries outside the current scope, the system returns ``not executable'' flags rather than forcing answer generation---this honest delineation of capability boundaries better aligns with medical AI ethics than systems covering all domains with inconsistent quality. This limitation is engineering-related rather than architectural. The system's incremental compilation mechanism, modular extraction pipelines, and standardized graph construction all support domain expansion.

We anticipate that this paradigm will continue evolving: expanding to oncology, neuroscience, immunology, and other domains to build a global medical causal graph covering major clinical disciplines; integrating multi-center data sources to establish multi-level validation systems and dynamically updated evidence credibility scores; incorporating multi-omics data to progress from population-average effects to individualized effect prediction; synergizing with foundation models to build next-generation medical AI assistants that understand natural language and return executable causal estimates; and integrating cost-effectiveness models, patient preferences, and ethical constraints to support complete multi-objective decision analysis. By systematically transforming medical knowledge from ``text'' into ``code,'' the DoAtlas paradigm will provide critical methodological foundations for trustworthy, scalable, and clinically deployable medical AI, ultimately achieving the leap from ``AI-assisted literature reading'' to ``AI-assisted clinical decision-making.''

\bibliography{example_paper}

@inproceedings{langley00,
 author    = {P. Langley},
 title     = {Crafting Papers on Machine Learning},
 year      = {2000},
 pages     = {1207--1216},
 editor    = {Pat Langley},
 booktitle     = {Proceedings of the 17th International Conference
              on Machine Learning (ICML 2000)},
 address   = {Stanford, CA},
 publisher = {Morgan Kaufmann}
}

@article{Singhal2023llm,
	author = {Singhal, Karan and Azizi, Shekoofeh and Tu, Tao and Mahdavi, S. Sara and Wei, Jason and Chung, Hyung Won and Scales, Nathan and Tanwani, Ajay and Cole-Lewis, Heather and Pfohl, Stephen and Payne, Perry and Seneviratne, Martin and Gamble, Paul and Kelly, Chris and Babiker, Abubakr and Sch{\"a}rli, Nathanael and Chowdhery, Aakanksha and Mansfield, Philip and Demner-Fushman, Dina and Ag{\"u}era y Arcas, Blaise and Webster, Dale and Corrado, Greg S. and Matias, Yossi and Chou, Katherine and Gottweis, Juraj and Tomasev, Nenad and Liu, Yun and Rajkomar, Alvin and Barral, Joelle and Semturs, Christopher and Karthikesalingam, Alan and Natarajan, Vivek},
	journal = {Nature},
	number = {7972},
	pages = {172--180},
	title = {Large language models encode clinical knowledge},
	volume = {620},
	year = {2023},
}

@book{hernal_causal_2020,
  title     = {Causal Inference: What If},
  author    = {Hern{\'a}n, Miguel A. and Robins, James M.},
  year      = {2020},
  publisher = {Chapman \& Hall/CRC},
  address   = {Boca Raton, FL}
}

@article{hernan_target_2016,
	author = {Miguel A. Hern{\'a}n and Brian C. Sauer and Sonia Hern{\'a}ndez-D{\'\i}az and Robert Platt and Ian Shrier},
    title = {Specifying a target trial prevents immortal time bias and other self-inflicted injuries in observational analyses},
    journal = {Journal of Clinical Epidemiology},
    year = {2016},
	pages = {70-75},
	volume = {79},
	}

@article{kahan_estimands_2024,
  title   = {The estimands framework: a primer on the ICH E9(R1) addendum},
  author  = {Kahan, Benjamin C. and Hindley, James and Edwards, Michael and Cro, Sarah and Morris, Tim P.},
  journal = {BMJ},
  volume  = {384},
  pages   = {e076316},
  year    = {2024}
}

@article{kleinberg_prediction_2015,
    author = {Kleinberg, Jon and Ludwig, Jens and Mullainathan, Sendhil and Obermeyer, Ziad},
    year = {2015},
    month = {05},
    pages = {491-495},
    title = {Prediction Policy Problems},
    volume = {105},
    journal = {American Economic Review},
}

@article{glymour_causal_2019,
	author = {Glymour, Clark and Zhang, Kun and Spirtes, Peter},
	date = {2019},
	journal = {Front Genet},
	pages = {524},
	title = {Review of Causal Discovery Methods Based on Graphical Models.},
	volume = {10},
	year = {2019},
}

@article{esteva2017dermatologist,
  title={Dermatologist-level classification of skin cancer with deep neural networks},
  author={Esteva, Andre and Kuprel, Brett and Novoa, Roberto A and Ko, Justin and Swetter, Susan M and Blau, Helen M and Thrun, Sebastian},
  journal={Nature},
  volume={542},
  number={7639},
  pages={115--118},
  year={2017},
  publisher={Nature Publishing Group UK London}
}

@article{rajkomar2018scalable,
  title={Scalable and accurate deep learning with electronic health records},
  author={Rajkomar, Alvin and Oren, Eyal and Chen, Kai and Dai, Andrew M and Hajaj, Nissan and Hardt, Michaela and Liu, Peter J and Liu, Xiaobing and Marcus, Jake and Sun, Mimi and others},
  journal={NPJ digital medicine},
  volume={1},
  number={1},
  pages={18},
  year={2018},
  publisher={Nature Publishing Group UK London}
}

@article{d2008general,
  title={General cardiovascular risk profile for use in primary care: the Framingham Heart Study},
  author={D’Agostino Sr, Ralph B and Vasan, Ramachandran S and Pencina, Michael J and Wolf, Philip A and Cobain, Mark and Massaro, Joseph M and Kannel, William B},
  journal={Circulation},
  volume={117},
  number={6},
  pages={743--753},
  year={2008},
  publisher={Lippincott Williams \& Wilkins}
}

@article{shickel2017deep,
  title={Deep EHR: a survey of recent advances in deep learning techniques for electronic health record (EHR) analysis},
  author={Shickel, Benjamin and Tighe, Patrick James and Bihorac, Azra and Rashidi, Parisa},
  journal={IEEE journal of biomedical and health informatics},
  volume={22},
  number={5},
  pages={1589--1604},
  year={2017},
  publisher={IEEE}
}

@article{singhal2025toward,
  title={Toward expert-level medical question answering with large language models},
  author={Singhal, Karan and Tu, Tao and Gottweis, Juraj and Sayres, Rory and Wulczyn, Ellery and Amin, Mohamed and Hou, Le and Clark, Kevin and Pfohl, Stephen R and Cole-Lewis, Heather and others},
  journal={Nature Medicine},
  volume={31},
  number={3},
  pages={943--950},
  year={2025},
  publisher={Nature Publishing Group US New York}
}

@article{nori2023capabilities,
  title={Capabilities of gpt-4 on medical challenge problems},
  author={Nori, Harsha and King, Nicholas and McKinney, Scott Mayer and Carignan, Dean and Horvitz, Eric},
  journal={arXiv preprint arXiv:2303.13375},
  year={2023}
}

@article{lewis2020retrieval,
  title={Retrieval-augmented generation for knowledge-intensive nlp tasks},
  author={Lewis, Patrick and Perez, Ethan and Piktus, Aleksandra and Petroni, Fabio and Karpukhin, Vladimir and Goyal, Naman and K{\"u}ttler, Heinrich and Lewis, Mike and Yih, Wen-tau and Rockt{\"a}schel, Tim and others},
  journal={Advances in neural information processing systems},
  volume={33},
  pages={9459--9474},
  year={2020}
}

@article{topol2019high,
  title={High-performance medicine: the convergence of human and artificial intelligence},
  author={Topol, Eric J},
  journal={Nature medicine},
  volume={25},
  number={1},
  pages={44--56},
  year={2019},
  publisher={Nature Publishing Group US New York}
}

@article{kelly2019key,
  title={Key challenges for delivering clinical impact with artificial intelligence},
  author={Kelly, Christopher J and Karthikesalingam, Alan and Suleyman, Mustafa and Corrado, Greg and King, Dominic},
  journal={BMC medicine},
  volume={17},
  number={1},
  pages={195},
  year={2019},
  publisher={Springer}
}

@article{amann2020explainability,
  title={Explainability for artificial intelligence in healthcare: a multidisciplinary perspective},
  author={Amann, Julia and Blasimme, Alessandro and Vayena, Effy and Frey, Dietmar and Madai, Vince I and Precise4Q Consortium},
  journal={BMC medical informatics and decision making},
  volume={20},
  number={1},
  pages={310},
  year={2020},
  publisher={Springer}
}

@article{achiam2023gpt,
  title={Gpt-4 technical report},
  author={Achiam, Josh and Adler, Steven and Agarwal, Sandhini and Ahmad, Lama and Akkaya, Ilge and Aleman, Florencia Leoni and Almeida, Diogo and Altenschmidt, Janko and Altman, Sam and Anadkat, Shyamal and others},
  journal={arXiv preprint arXiv:2303.08774},
  year={2023}
}

@article{bai2022constitutional,
  title={Constitutional ai: Harmlessness from ai feedback},
  author={Bai, Yuntao and Kadavath, Saurav and Kundu, Sandipan and Askell, Amanda and Kernion, Jackson and Jones, Andy and Chen, Anna and Goldie, Anna and Mirhoseini, Azalia and McKinnon, Cameron and others},
  journal={arXiv preprint arXiv:2212.08073},
  year={2022}
}

@article{yang2023baichuan,
  title={Baichuan 2: Open large-scale language models},
  author={Yang, Aiyuan and Xiao, Bin and Wang, Bingning and Zhang, Borong and Bian, Ce and Yin, Chao and Lv, Chenxu and Pan, Da and Wang, Dian and Yan, Dong and others},
  journal={arXiv preprint arXiv:2309.10305},
  year={2023}
}

@article{wang2025baichuan,
  title={Baichuan-m1: Pushing the medical capability of large language models},
  author={Wang, Bingning and Zhao, Haizhou and Zhou, Huozhi and Song, Liang and Xu, Mingyu and Cheng, Wei and Zeng, Xiangrong and Zhang, Yupeng and Huo, Yuqi and Wang, Zecheng and others},
  journal={arXiv preprint arXiv:2502.12671},
  year={2025}
}

@article{zhang2305huatuogpt,
  title={HuatuoGPT, towards taming language model to be a doctor. arXiv (2023)},
  author={Zhang, H and Chen, J and Jiang, F and Yu, F and Chen, Z and Li, J and Chen, G and Wu, X and Zhang, Z and Xiao, Q and others},
  journal={arXiv preprint arXiv:2305.15075},
  year={2023}
}

@article{xiong2023doctorglm,
  title={Doctorglm: Fine-tuning your chinese doctor is not a herculean task},
  author={Xiong, Honglin and Wang, Sheng and Zhu, Yitao and Zhao, Zihao and Liu, Yuxiao and Huang, Linlin and Wang, Qian and Shen, Dinggang},
  journal={arXiv preprint arXiv:2304.01097},
  year={2023}
}

@book{pearl2000models,
  title     = {Causality: Models, Reasoning, and Inference},
  author    = {Pearl, Judea},
  year      = {2000},
  publisher = {Cambridge University Press}
}

@book{vanderweele2015explanation,
  title={Explanation in causal inference: methods for mediation and interaction},
  author={VanderWeele, Tyler},
  year={2015},
  publisher={Oxford University Press}
}

@book{shortliffe1976mycin,
  title     = {Computer-Based Medical Consultations: MYCIN},
  author    = {Shortliffe, Edward H.},
  year      = {1976},
  publisher = {Elsevier}
}

@incollection{miller1985internist,
  title={Internist-I, an experimental computer-based diagnostic consultant for general internal medicine},
  author={Miller, Randolph A and Pople Jr, Harry E and Myers, Jack D},
  booktitle={Computer-assisted medical decision making},
  pages={139--158},
  year={1985},
  publisher={Springer}
}

@article{barnett1987dxplain,
  title={DXplain: an evolving diagnostic decision-support system},
  author={Barnett, G Octo and Cimino, James J and Hupp, Jon A and Hoffer, Edward P},
  journal={Jama},
  volume={258},
  number={1},
  pages={67--74},
  year={1987},
  publisher={American Medical Association}
}

@article{zakka2024almanac,
  title={Almanac—retrieval-augmented language models for clinical medicine},
  author={Zakka, Cyril and Shad, Rohan and Chaurasia, Akash and Dalal, Alex R and Kim, Jennifer L and Moor, Michael and Fong, Robyn and Phillips, Curran and Alexander, Kevin and Ashley, Euan and others},
  journal={Nejm ai},
  volume={1},
  number={2},
  pages={AIoa2300068},
  year={2024},
  publisher={Massachusetts Medical Society}
}

@article{bohnet2022attributed,
  title={Attributed question answering: Evaluation and modeling for attributed large language models},
  author={Bohnet, Bernd and Tran, Vinh Q and Verga, Pat and Aharoni, Roee and Andor, Daniel and Soares, Livio Baldini and Ciaramita, Massimiliano and Eisenstein, Jacob and Ganchev, Kuzman and Herzig, Jonathan and others},
  journal={arXiv preprint arXiv:2212.08037},
  year={2022}
}

@inproceedings{malaviya2024expertqa,
  title={Expertqa: Expert-curated questions and attributed answers},
  author={Malaviya, Chaitanya and Lee, Subin and Chen, Sihao and Sieber, Elizabeth and Yatskar, Mark and Roth, Dan},
  booktitle={Proceedings of the 2024 Conference of the North American Chapter of the Association for Computational Linguistics: Human Language Technologies (Volume 1: Long Papers)},
  pages={3025--3045},
  year={2024}
}

@inproceedings{ji2024chain,
  title={Chain-of-thought improves text generation with citations in large language models},
  author={Ji, Bin and Liu, Huijun and Du, Mingzhe and Ng, See-Kiong},
  booktitle={Proceedings of the AAAI Conference on Artificial Intelligence},
  volume={38},
  pages={18345--18353},
  year={2024}
}

@inproceedings{wang2025medcite,
  title={MedCite: Can Language Models Generate Verifiable Text for Medicine?},
  author={Wang, Xiao and Tan, Mengjue and Jin, Qiao and Xiong, Guangzhi and Hu, Yu and Zhang, Aidong and Lu, Zhiyong and Zhang, Minjia},
  booktitle={Findings of the Association for Computational Linguistics: ACL 2025},
  pages={18891--18913},
  year={2025},
}

@inproceedings{gao2023rarr,
  title={Rarr: Researching and revising what language models say, using language models},
  author={Gao, Luyu and Dai, Zhuyun and Pasupat, Panupong and Chen, Anthony and Chaganty, Arun Tejasvi and Fan, Yicheng and Zhao, Vincent and Lao, Ni and Lee, Hongrae and Juan, Da-Cheng and others},
  booktitle={Proceedings of the 61st Annual Meeting of the Association for Computational Linguistics (Volume 1: Long Papers)},
  pages={16477--16508},
  year={2023}
}

@inproceedings{manakul2023selfcheckgpt,
  title={Selfcheckgpt: Zero-resource black-box hallucination detection for generative large language models},
  author={Manakul, Potsawee and Liusie, Adian and Gales, Mark},
  booktitle={Proceedings of the 2023 conference on empirical methods in natural language processing},
  pages={9004--9017},
  year={2023}
}

@article{guyatt1992evidence,
  title={Evidence-based medicine: a new approach to teaching the practice of medicine},
  author={Guyatt, Gordon and Cairns, John and Churchill, David and Cook, Deborah and Haynes, Brian and Hirsh, Jack and Irvine, Jan and Levine, Mark and Levine, Mitchell and Nishikawa, Jim and others},
  journal={jama},
  volume={268},
  number={17},
  pages={2420--2425},
  year={1992},
  publisher={American Medical Association}
}

@article{sackett1996evidence,
  title={Evidence based medicine: what it is and what it isn't},
  author={Sackett, David L and Rosenberg, William MC and Gray, JA Muir and Haynes, R Brian and Richardson, W Scott},
  journal={BMJ},
  volume={312},
  number={7023},
  pages={71--72},
  year={1996},
  publisher={British Medical Journal Publishing Group}
}

@book{higgins2008cochrane,
  title={Cochrane handbook for systematic reviews of interventions},
  author={Higgins, Julian PT and Green, Sally and others},
  year={2008},
  publisher={Wiley Online Library}
}

@article{greenhalgh2014evidence,
  title={Evidence based medicine: a movement in crisis?},
  author={Greenhalgh, Trisha and Howick, Jeremy and Maskrey, Neal},
  journal={BMJ},
  pages = {g3725},
  volume={348},
  year={2014},
  publisher={British Medical Journal Publishing Group}
}

@article{kent2010assessing,
  title={Assessing and reporting heterogeneity in treatment effects in clinical trials: a proposal},
  author={Kent, David M and Rothwell, Peter M and Ioannidis, John PA and Altman, Doug G and Hayward, Rodney A},
  journal={Trials},
  volume={11},
  number={1},
  pages={85},
  year={2010},
  publisher={Springer}
}

@article{ioannidis2016most,
  title={Why most clinical research is not useful},
  author={Ioannidis, John PA},
  journal={PLoS medicine},
  volume={13},
  number={6},
  pages={e1002049},
  year={2016},
  publisher={Public Library of Science San Francisco, CA USA}
}

@book{hernan2010causal,
  title={Causal inference},
  author={Hern{\'a}n, Miguel A and Robins, James M},
  year={2010},
  publisher={CRC Boca Raton, FL}
}

@article{austin2011introduction,
  title={An introduction to propensity score methods for reducing the effects of confounding in observational studies},
  author={Austin, Peter C},
  journal={Multivariate behavioral research},
  volume={46},
  number={3},
  pages={399--424},
  year={2011},
  publisher={Taylor \& Francis}
}

@article{greenland2000introduction,
  title={An introduction to instrumental variables for epidemiologists},
  author={Greenland, Sander},
  journal={International journal of epidemiology},
  volume={29},
  number={4},
  pages={722--729},
  year={2000},
  publisher={Oxford University Press}
}

@article{athey2016recursive,
  title={Recursive partitioning for heterogeneous causal effects},
  author={Athey, Susan and Imbens, Guido},
  journal={Proceedings of the National Academy of Sciences},
  volume={113},
  number={27},
  pages={7353--7360},
  year={2016},
  publisher={National Academy of Sciences}
}

@article{chernozhukov2018double,
  title={Double/debiased machine learning for treatment and structural parameters},
  author={Chernozhukov, Victor and Chetverikov, Denis and Demirer, Mert and Duflo, Esther and Hansen, Christian and Newey, Whitney and Robins, James},
  journal = {The Econometrics Journal},
  year={2018},
  volume  = {21},
  number  = {1},
  pages   = {C1--C68},
  publisher={Oxford University Press Oxford, UK}
}

@article{ioannidis2005most,
  title={Why most published research findings are false},
  author={Ioannidis, John PA},
  journal={PLoS medicine},
  volume={2},
  number={8},
  pages={e124},
  year={2005},
  publisher={Public Library of Science}
}

@article{hoeting1999bayesian,
  title={Bayesian model averaging: a tutorial (with comments by M. Clyde, David Draper and EI George, and a rejoinder by the authors},
  author={Hoeting, Jennifer A and Madigan, David and Raftery, Adrian E and Volinsky, Chris T},
  journal={Statistical science},
  volume={14},
  number={4},
  pages={382--417},
  year={1999},
  publisher={Institute of Mathematical Statistics}
}

@article{greenland2005multiple,
  title={Multiple-bias modelling for analysis of observational data},
  author={Greenland, Sander},
  journal={Journal of the Royal Statistical Society Series A: Statistics in Society},
  volume={168},
  number={2},
  pages={267--306},
  year={2005},
  publisher={Oxford University Press}
}
\bibliographystyle{icml2026}

\newpage
\appendix
\onecolumn
\section{External Validation: Implementation, Sanity Checks, and Representative Mismatches}
\label{app:external-validation}

This appendix details the external validation setup. We map each literature-side canonical causal object $(\varepsilon,c)$ to an executable emulation specification on HPP and report validation outcomes together with diagnostic flags. We further include two sanity checks (positive and negative controls) and present 10 representative mismatch cases to characterize the boundaries and failure modes of external validation.

\subsection*{X.1\quad Artifacts and Reproducible Interfaces}

\paragraph{Inputs.}
\begin{enumerate}[leftmargin=*,itemsep=2pt]
\item \textbf{\texttt{conflict\_pool.json}}: a pool of claims (or within-bucket conflicting entries) selected for external validation. Each record contains: (i) study identifiers (e.g., DOI/PMID/title), (ii) canonical estimand $\varepsilon=(\Pi,\iota,o,\tau,\mu)$ and canonical claim $c=(\theta,CI)$, (iii) evidence-card ID (or a reversible key), and (iv) conflict tags (e.g., \texttt{conflict\_ci\_non\_overlap}, \texttt{conflict\_high\_heterogeneity}).
\item \textbf{Evidence-card metadata fields} aligned one-to-one with each claim:
\begin{enumerate}[leftmargin=*,itemsep=1pt]
\item \texttt{hpp\_mapping}: variable mappings for $(\Pi,\iota,o)$, including field names, coding systems, derivations, and availability.
\item \texttt{time\_semantics}: definitions of time-zero, baseline window, follow-up windows, and horizon.
\item \texttt{adjustment\_set\_used}: covariate list and model specification used in emulation (and, if applicable, balance diagnostics).
\end{enumerate}
\end{enumerate}

\paragraph{Outputs.}
\begin{enumerate}[leftmargin=*,itemsep=2pt]
\item \textbf{Validation results table} (one row per claim): RWD estimand estimate (e.g., logHR/logOR/RD/mean difference), CI/$p$-value (if applicable), sample size/event counts, balance statistics (e.g., SMD), and diagnostic flags.
\item \textbf{Mismatch casebook} (10 cases): for each case we report the canonical $(\varepsilon,c)$, the emulation specification, a five-category attribution, and the emitted flags.
\end{enumerate}

\subsection*{X.2\quad HPP Emulation Specification from Canonical $(\varepsilon,c)$}

For each canonical $\varepsilon=(\Pi,\iota,o,\tau,\mu)$ and numerical claim $c=(\theta,CI)$, we construct an HPP emulation specification comprising:
\begin{enumerate}[leftmargin=*,itemsep=2pt]
\item \textbf{Eligibility ($\Pi$).} Map inclusion/exclusion criteria to HPP-observable conditions (diagnoses, medications, labs). Missing or weak proxies are flagged as \texttt{hpp\_missing\_*}, \texttt{hpp\_tentative\_*}, or \texttt{hpp\_derived\_*}.
\item \textbf{Treatment/contrast ($\iota$).} Define treatment and comparator (index event, washout, concomitant restrictions). Non-equivalent contrasts (e.g., placebo/randomization) are flagged as \texttt{exposure\_contrast\_mismatch} or \texttt{contrast\_mismatch}.
\item \textbf{Outcome ($o$).} Map to event codes or numeric measurements. Missing or approximate outcomes are flagged as \texttt{hpp\_missing\_*} or \texttt{hpp\_close\_*}.
\item \textbf{Time semantics ($\tau$).} Use \texttt{time\_semantics} to define time-zero, baseline window (e.g., P28D), horizon (e.g., P16W), and follow-up. Temporal limitations are flagged as \texttt{tau\_short}/\texttt{tau\_long}/\texttt{tau\_missing\_partial}/\texttt{tau\_specific}.
\item \textbf{Estimation and adjustment ($\mu$).} Use \texttt{adjustment\_set\_used} to specify covariates and estimators. Adjustment or estimand discrepancies are flagged as \texttt{adjustment\_mismatch} or \texttt{estimand\_mismatch}.
\end{enumerate}

\subsection*{X.3\quad Diagnostic Flags and Mismatch Criteria}

A \emph{mismatch} indicates substantial divergence between the literature-side canonical claim and the RWD emulation estimate under comparable conditions, or instability due to non-comparability/missingness. We use the following flag categories:
\begin{itemize}[leftmargin=*,itemsep=2pt]
\item \textbf{Point-estimate sign flip:} \texttt{direction\_flip\_point}. Since taxonomy F has \texttt{direction\_conflict}=0 in this slice, we define ``direction flip'' as opposite point-estimate signs while CIs may cross 0.
\item \textbf{Non-overlapping CIs:} \texttt{conflict\_ci\_non\_overlap}.
\item \textbf{High heterogeneity / gradient conflicts:} \texttt{conflict\_high\_heterogeneity}, \texttt{contrast\_gradient}.
\item \textbf{Subgroup inconsistencies:} \texttt{subgroup\_inconsistency}, \texttt{cohort\_subgroup}.
\item \textbf{Mapping missing/approximate/derived:} \texttt{hpp\_missing\_*}, \texttt{hpp\_close\_*}, \texttt{hpp\_tentative\_*}, \texttt{hpp\_derived\_*}.
\item \textbf{Non-comparable contrasts:} \texttt{exposure\_contrast\_mismatch}, \texttt{contrast\_mismatch}.
\item \textbf{Temporal issues:} \texttt{tau\_short}, \texttt{tau\_long}, \texttt{tau\_specific}, \texttt{tau\_missing\_partial}.
\item \textbf{Estimand/adjustment discrepancies:} \texttt{estimand\_mismatch}, \texttt{adjustment\_mismatch}.
\item \textbf{Method heterogeneity (e.g., MR):} \texttt{method\_heterogeneity}.
\end{itemize}

\subsection*{X.4\quad Sanity Checks: Positive and Negative Controls}

To prevent overly optimistic validation signals, we include two sanity checks:
\begin{enumerate}[leftmargin=*,itemsep=2pt]
\item \textbf{Positive controls.} A set of well-established relationships expected to be recoverable in HPP, used to test whether the pipeline consistently reproduces known effect directions (and, where applicable, plausible magnitudes) under the same mapping/time/adjustment machinery.
\item \textbf{Negative controls.} A set of theoretically unrelated exposure--outcome pairs (or substituted contrasts/outcomes) used to estimate false-positive rates and constrain ``validation success'' inflation. We report the fraction of spurious significant findings (or analogous error metrics) under the same pipeline.
\end{enumerate}

\subsection*{X.5\quad Ten Representative Mismatch Cases}

Below we summarize 10 representative mismatches. For each case we report: (i) literature-side canonical $(\varepsilon,c)$, (ii) the HPP emulation specification (time-zero, horizon, adjustment), (iii) five-category attribution, and (iv) diagnostic flags. The tag \texttt{direction\_flip\_point} denotes a point-estimate sign flip with potential CI overlap at 0.

\paragraph{Case 1 --- Abrocitinib 100mg QD $\rightarrow$ IGA response @16w (point-estimate sign flip).}
\begin{itemize}[leftmargin=*,itemsep=1pt]
\item \textbf{Literature A:} \emph{Abrocitinib vs Placebo or Dupilumab for Atopic Dermatitis} (NEJM 2021, 10.1056/NEJMoa2019380). Canonical $(\varepsilon,c)$: $\texttt{risk\_diff } +0.221\ [0.137,0.305]$ (Abrocitinib 100mg vs placebo).
\item \textbf{Literature B:} same study. Canonical $(\varepsilon,c)$: $\texttt{risk\_diff } -0.035\ [-0.122,0.052]$ (Abrocitinib 100mg vs dupilumab).
\item \textbf{RWD emulation:} time-zero = baseline window P28D; horizon = P16W; adjustment = baseline disease severity (moderate/severe).
\item \textbf{Attribution (5 categories):} population mismatch (moderate/severe AD trial population); contrast mismatch (placebo/dupilumab not directly emulable); outcome definition mismatch (IGA missing in HPP); unmeasured confounding/indication bias; time-window mismatch (short 16-week horizon).
\item \textbf{Flags:} \texttt{conflict\_ci\_non\_overlap}, \texttt{direction\_flip\_point}, \texttt{hpp\_missing\_X}, \texttt{hpp\_missing\_Y}, \texttt{tau\_short}, \texttt{exposure\_contrast\_mismatch}, \texttt{population\_specific}.
\end{itemize}

\paragraph{Case 2 --- Abrocitinib 100mg QD $\rightarrow$ EASI-75 @16w (point-estimate sign flip).}
\begin{itemize}[leftmargin=*,itemsep=1pt]
\item \textbf{Literature A:} NEJM 2021, 10.1056/NEJMoa2019380. Canonical $(\varepsilon,c)$: $\texttt{risk\_diff } +0.297\ [0.195,0.399]$ (vs placebo).
\item \textbf{Literature B:} same study. Canonical $(\varepsilon,c)$: $\texttt{risk\_diff } -0.051\ [-0.139,0.037]$ (vs dupilumab).
\item \textbf{RWD emulation:} time-zero = P28D; horizon = P16W; adjustment = baseline disease severity.
\item \textbf{Attribution:} population mismatch; contrast mismatch; outcome definition mismatch (EASI-75 missing); unmeasured confounding; time-window mismatch.
\item \textbf{Flags:} \texttt{conflict\_ci\_non\_overlap}, \texttt{direction\_flip\_point}, \texttt{hpp\_missing\_X}, \texttt{hpp\_missing\_Y}, \texttt{tau\_short}, \texttt{exposure\_contrast\_mismatch}, \texttt{population\_specific}.
\end{itemize}

\paragraph{Case 3 --- HF quality score ($\ge$50\% vs $<$50\%) $\rightarrow$ all-cause mortality (point-estimate sign flip).}
\begin{itemize}[leftmargin=*,itemsep=1pt]
\item \textbf{Literature A:} \emph{Association between heart failure quality of care and mortality} (EJHF 2022, 10.1002/ejhf.2725). Canonical $(\varepsilon,c)$: $\texttt{logHR } -0.211\ [-0.329,-0.094]$.
\item \textbf{Literature B:} same study. Canonical $(\varepsilon,c)$: $\texttt{logHR } +0.049\ [-0.030,0.131]$.
\item \textbf{RWD emulation:} time-zero = first SwedeHF registry entry; horizon = P23.6M; adjustment = HF severity/comorbidities/NT-proBNP, etc. (extended covariates).
\item \textbf{Attribution:} population mismatch (registry, older/high-risk); exposure definition mismatch (quality score system missing); outcome definition mismatch; unmeasured confounding (quality correlates with latent severity); follow-up mismatch.
\item \textbf{Flags:} \texttt{conflict\_ci\_non\_overlap}, \texttt{direction\_flip\_point}, \texttt{hpp\_missing\_X}, \texttt{hpp\_missing\_Y}, \texttt{tau\_specific}, \texttt{population\_specific}.
\end{itemize}

\paragraph{Case 4 --- Dapagliflozin 10mg $\rightarrow$ chronic eGFR slope (subgroup inconsistency).}
\begin{itemize}[leftmargin=*,itemsep=1pt]
\item \textbf{Literature A:} DAPA-CKD prespecified analysis (2021, 10.1016/S2213-8587(21)00242-4). Canonical $(\varepsilon,c)$: $\texttt{mean\_diff } +2.26\ [1.88,2.64]$ (T2D subgroup).
\item \textbf{Literature B:} same study. Canonical $(\varepsilon,c)$: $\texttt{mean\_diff } +1.29\ [0.73,1.85]$ (non-T2D subgroup).
\item \textbf{RWD emulation:} time-zero = baseline; horizon = P2W\_to\_P2.3Y; adjustment = baseline UACR.
\item \textbf{Attribution:} subgroup/population mismatch; contrast mismatch (RCT assignment vs observational); outcome mismatch (slope missing); unmeasured confounding; time-window mismatch (acute+chronic phases).
\item \textbf{Flags:} \texttt{conflict\_ci\_non\_overlap}, \texttt{subgroup\_inconsistency}, \texttt{hpp\_missing\_X}, \texttt{hpp\_missing\_Y}, \texttt{tau\_specific}, \texttt{population\_specific}.
\end{itemize}

\paragraph{Case 5 --- Life's Essential 8 score $\rightarrow$ all-cause mortality (magnitude gap).}
\begin{itemize}[leftmargin=*,itemsep=1pt]
\item \textbf{Literature A:} BMJ Medicine 2023, 10.1186/s12916-023-02824-8. Canonical $(\varepsilon,c)$: $\texttt{logHR } -0.511\ [-0.673,-0.342]$ (Intermediate 50--74 vs Low 0--49).
\item \textbf{Literature B:} same study. Canonical $(\varepsilon,c)$: $\texttt{logHR } -1.273\ [-1.561,-1.022]$ (High 75--100 vs Low 0--49).
\item \textbf{RWD emulation:} time-zero = NHANES baseline; horizon = P7Y7M; adjustment = age/sex/race + SES/history (model-dependent).
\item \textbf{Attribution:} population mismatch; derived score/threshold differences; outcome mismatch; residual confounding (behavioral composite); long follow-up differences.
\item \textbf{Flags:} \texttt{conflict\_ci\_non\_overlap}, \texttt{hpp\_derived\_X}, \texttt{hpp\_missing\_Y}, \texttt{contrast\_mismatch}, \texttt{tau\_long}.
\end{itemize}

\paragraph{Case 6 --- SES $\rightarrow$ all-cause mortality (total vs mediated effect).}
\begin{itemize}[leftmargin=*,itemsep=1pt]
\item \textbf{Literature A:} BMJ 2021, 10.1136/bmj.n604 (total effect). Canonical $(\varepsilon,c)$: $\texttt{logHR } +0.859\ [0.747,0.975]$ (low vs high SES).
\item \textbf{Literature B:} same study (mediation/lifestyle pathway). Canonical $(\varepsilon,c)$: $\texttt{logHR } +0.103\ [0.091,0.114]$.
\item \textbf{RWD emulation:} time-zero = baseline; horizon = P11Y2M; adjustment = age/sex/race/comorbidities.
\item \textbf{Attribution:} population mismatch; SES as latent proxy; outcome proxying; unmeasured confounding (SES--lifestyle common causes); long follow-up.
\item \textbf{Flags:} \texttt{conflict\_ci\_non\_overlap}, \texttt{hpp\_derived\_X}, \texttt{hpp\_tentative\_Y}, \texttt{estimand\_mismatch}, \texttt{tau\_long}.
\end{itemize}

\paragraph{Case 7 --- SES $\rightarrow$ all-cause mortality (adjustment mismatch vs mediated effect).}
\begin{itemize}[leftmargin=*,itemsep=1pt]
\item \textbf{Literature A:} BMJ 2021, 10.1136/bmj.n604 (adjusted for Healthy Lifestyle Score). Canonical $(\varepsilon,c)$: $\texttt{logHR } +0.756\ [0.642,0.867]$.
\item \textbf{Literature B:} same study (mediation). Canonical $(\varepsilon,c)$: $\texttt{logHR } +0.103\ [0.091,0.114]$.
\item \textbf{RWD emulation:} time-zero = baseline; horizon = P11Y2M; adjustment = age/sex/race/comorbidities/lifestyle.
\item \textbf{Attribution:} population mismatch; SES proxy; outcome proxying; overlap between SES and lifestyle; adjustment/estimand mismatch.
\item \textbf{Flags:} \texttt{conflict\_ci\_non\_overlap}, \texttt{hpp\_derived\_X}, \texttt{hpp\_tentative\_Y}, \texttt{adjustment\_mismatch}, \texttt{estimand\_mismatch}.
\end{itemize}

\paragraph{Case 8 --- Basal metabolic rate (BMR) $\rightarrow$ kidney stones (magnitude gap + heterogeneity).}
\begin{itemize}[leftmargin=*,itemsep=1pt]
\item \textbf{Literature A:} 2025, 10.1097/JS9.0000000000002658. Canonical $(\varepsilon,c)$: $\texttt{logOR } -0.198\ [-0.329,-0.062]$ (Q2 vs Q1).
\item \textbf{Literature B:} same study. Canonical $(\varepsilon,c)$: $\texttt{logOR } -0.713\ [-1.050,-0.357]$ (Q4 vs Q1).
\item \textbf{RWD emulation:} time-zero = 2006--2010 baseline; horizon = P14Y (partial $\tau$); adjustment = BMI/waist/smoking/metabolic comorbidities, etc.
\item \textbf{Attribution:} cohort mismatch; derived BMR/quantile differences; code-system mapping (ICD-9/10 to ICD-11); unmeasured lifestyle/metabolic confounding; incomplete $\tau$ under long follow-up.
\item \textbf{Flags:} \texttt{conflict\_ci\_non\_overlap}, \texttt{conflict\_high\_heterogeneity}, \texttt{hpp\_derived\_X}, \texttt{hpp\_close\_Y}, \texttt{contrast\_mismatch}, \texttt{tau\_long}, \texttt{tau\_missing\_partial}.
\end{itemize}

\paragraph{Case 9 --- Planetary Health Diet Index (PHDI) $\rightarrow$ CVD (gradient heterogeneity).}
\begin{itemize}[leftmargin=*,itemsep=1pt]
\item \textbf{Literature A:} 2024, 10.1016/S2542-5196(24)00184-2. Canonical $(\varepsilon,c)$: $\texttt{logHR } -0.400\ [-0.462,-0.342]$ (Q5 vs Q1).
\item \textbf{Literature B:} same study. Canonical $(\varepsilon,c)$: $\texttt{logHR } -0.041\ [-0.094,0.030]$ (Q2 vs Q1).
\item \textbf{RWD emulation:} time-zero = 1986/1991 baseline; horizon = P30Y; adjustment = energy intake/smoking/physical activity/lipids, etc.
\item \textbf{Attribution:} cohort-specific population; derived index/gradient mismatch; composite endpoint coding; diet measurement error and time-varying exposure; extremely long follow-up.
\item \textbf{Flags:} \texttt{conflict\_high\_heterogeneity}, \texttt{hpp\_derived\_X}, \texttt{hpp\_derived\_Y}, \texttt{contrast\_gradient}, \texttt{cohort\_subgroup}, \texttt{tau\_long}.
\end{itemize}

\paragraph{Case 10 --- Genetic liability to T2D $\rightarrow$ CHD (method/subgroup heterogeneity).}
\begin{itemize}[leftmargin=*,itemsep=1pt]
\item \textbf{Literature A:} MR-Egger (2021, 10.2337/dc20-1137). Canonical $(\varepsilon,c)$: $\texttt{logOR } 0.000\ [-0.041,0.039]$.
\item \textbf{Literature B:} TSRI + radial MR. Canonical $(\varepsilon,c)$: $\texttt{logOR } +0.215\ [0.182,0.255]$.
\item \textbf{RWD emulation:} time-zero = 2006--2010 baseline; horizon = P5.3Y; adjustment = genetic PCs/array (method-specific).
\item \textbf{Attribution:} population/sex differences; exposure definition (genetic risk vs observed T2D); CHD coding/mapping; pleiotropy and MR method differences; time-window mismatch.
\item \textbf{Flags:} \texttt{conflict\_high\_heterogeneity}, \texttt{hpp\_close\_X}, \texttt{hpp\_close\_Y}, \texttt{method\_heterogeneity}, \texttt{subgroup\_inconsistency}.
\end{itemize}

\subsection{Summary of Mismatch Drivers}

Across the 10 cases, mismatches largely arise from: (i) \textbf{comparability gaps} (population, contrast, and outcome definitions not equivalently realizable in RWD), (ii) \textbf{time semantics and estimand/adjustment misalignment} (incomplete $\tau$, acute vs chronic phase mixing, total vs mediated effects), and (iii) \textbf{measurement/mapping error} (complex derived constructs, coding-system conversion, missing key variables). These failure profiles are made explicit via flags, enabling diagnosis rather than collapsing validation into a single binary verdict.

\section{Theoretical Foundations}
\label{app:theory}

\subsection{Complete Definitions}
\label{app:complete-defs}

\subsubsection{Evidence Object Specification}
\label{app:estimand-object}

This subsection specifies the complete evidence object used throughout the paper.
We build on the estimand IR in Definition~\ref{def:estimand-ir} and the canonicalization operator in
Definition~\ref{def:canonicalization}. A complete evidence object is a triple
\begin{equation}
\label{eq:evidence-object}
(\varepsilon, c, \pi),
\end{equation}
where $\varepsilon=(\Pi,\iota,o,\tau,\mu)$ is the semantic estimand IR, $c=(\theta,\mathrm{CI})$ is the numeric claim,
and $\pi$ is provenance. Canonicalization acts on $(\varepsilon,c)$ and preserves $\pi$.

\paragraph{Population $\Pi$.}
The target population is represented as
\begin{equation}
\label{eq:population}
\Pi=(p_{\mathrm{bucket}},p_{\mathrm{setting}}).
\end{equation}
Here $p_{\mathrm{bucket}}\in\mathcal{P}$ is a coarse population bucket used for comparison (e.g., \textsf{T2DM},
\textsf{ASCVD}, \textsf{general-adults}), and $p_{\mathrm{setting}}$ captures population semantics that affect interpretation,
including eligibility criteria, recruitment context, and care setting. We intentionally exclude study-specific attributes
(e.g., sample size) from $\Pi$; these are recorded in provenance $\pi$.

\paragraph{Intervention contrast $\iota$.}
The intervention contrast encodes what is being compared. We represent it as
\begin{equation}
\label{eq:contrast}
\iota=(c_{\mathrm{type}},\delta,u,x_0,x_1),
\end{equation}
where $c_{\mathrm{type}}$ is the contrast type, $\delta$ is the contrast magnitude, $u$ is the unit, and $(x_0,x_1)$ denote
reference/active levels when applicable. We use four contrast types:
\begin{enumerate}[leftmargin=*, itemsep=2pt]
\item \textbf{Per-unit contrast} (dose--response): effect per $\delta$ units of exposure (e.g., per 10 mg/day).
\item \textbf{Binary contrast}: treated vs control (or exposed vs unexposed), specified by $(x_0,x_1)$.
\item \textbf{Arm-versus-control}: multi-arm trials where an active arm is compared against a designated control arm.
\item \textbf{Categorical contrast}: multi-level exposures where $(x_0,x_1)$ identifies two compared categories.
\end{enumerate}
When multiple parameterizations of the same contrast are reported (e.g., per 1 unit vs per 10 units), we do not treat them as
numerically interchangeable by default; any unit scaling used for alignment is recorded in the conditions $\alpha$ produced by
canonicalization (Definition~\ref{def:canonicalization}).

\paragraph{Outcome $o$.}
The outcome $o$ is a standardized phenotype identifier linked to a data dictionary. Each outcome includes metadata specifying
(i) outcome type (continuous, binary, time-to-event), (ii) measurement unit when applicable, and (iii) definitional notes
(e.g., composition for composite endpoints). Outcome standardization is required to avoid conflating semantically distinct endpoints
that share similar names.

\paragraph{Time horizon $\tau$.}
The time horizon $\tau$ specifies the window over which the effect is defined. We distinguish three common semantics:
\begin{enumerate}[leftmargin=*, itemsep=2pt]
\item \textbf{Acute horizon}: short windows after intervention start (e.g., hours or days).
\item \textbf{Fixed follow-up horizon}: outcomes assessed at, or aggregated over, a fixed follow-up time (e.g., 12-week change in HbA1c,
12-month cumulative incidence).
\item \textbf{Time-to-event horizon}: event-time outcomes with censoring, where the estimand is tied to an event definition and a follow-up regime
(e.g., a hazard ratio over the observed follow-up).
\end{enumerate}
We represent $\tau$ as either a duration (a nonnegative scalar with an explicit unit), an interval relative to a reference time
(e.g., baseline), or a time-to-event regime descriptor (event definition plus follow-up/censoring regime, when needed for interpretation).
Since materially different horizons define different estimands, comparability uses a deterministic alignment
$\widehat{\tau}=\mathrm{align}(\tau)$; the alignment policy is specified in Appendix~\ref{app:time-alignment}.

\paragraph{Measure functional $\mu$.}
The measure functional $\mu$ specifies how the effect is quantified, including its measure family, measure type, and canonical scale.
We write
\[
m_{\mathrm{family}}(\mu)\in\{\textsf{ratio},\textsf{difference}\},\quad
m_{\mathrm{type}}(\mu)\in\{\textsf{HR},\textsf{RR},\textsf{OR},\textsf{MD},\textsf{RD},\dots\},\quad
s_{\mathrm{canon}}(\mu)\in\{\textsf{log},\textsf{identity}\}.
\]
Measure families, canonical scales, and the associated within-type reparameterizations are defined in
Appendix~\ref{app:measure-families}.

\paragraph{Claim object $c=(\theta,\mathrm{CI})$.}
A claim object records a point estimate and uncertainty:
\begin{equation}
\label{eq:claim-object}
c=(\theta,\mathrm{CI}),\quad \mathrm{CI}=[\theta_{\mathrm{lower}},\theta_{\mathrm{upper}}].
\end{equation}
The parameter $\theta$ is the reported effect value on the study's reporting scale. Canonicalization maps $(\theta,\mathrm{CI})$ to
$(\theta_{\mathrm{canon}},\mathrm{CI}_{\mathrm{canon}})$ on the canonical scale associated with $\mu$
(Definition~\ref{def:canonicalization}). When a study reports standard errors or $p$-values instead of confidence intervals,
the system may derive $\mathrm{CI}$ under standard assumptions; such derivations are recorded in $\alpha$.

\paragraph{Provenance $\pi$.}
Provenance records source and study attributes:
\begin{equation}
\label{eq:provenance}
\pi=(\mathrm{ref},\mathrm{grade},n,\mathrm{card\_id}),
\end{equation}
where $\mathrm{ref}$ is a bibliographic reference (e.g., DOI), $\mathrm{grade}\in\{A,B,C\}$ is an evidence grade, $n$ is sample size,
and $\mathrm{card\_id}$ links to an evidence card for audit. Provenance is preserved under canonicalization and is used by the quality score $Q$ for default selection
(see Appendix~\ref{app:quality-score}).

\paragraph{Remark.}
The separation between the semantic estimand $\varepsilon$, numeric claim $c$, and provenance $\pi$ supports (i) semantic comparability
through bucket keys (Definition~\ref{def:comparability}), (ii) deterministic normalization via canonicalization
(Definition~\ref{def:canonicalization}), and (iii) auditability through stable provenance identifiers.

\subsubsection{Comparability and Poolability}
\label{app:comparability-poolability}

This subsection formalizes bucket keys, the comparability relation $\sim$ (Definition~\ref{def:comparability}),
and poolability. Comparability forms the foundation for evidence synthesis: it determines which claims can be grouped for
consistency checks, conflict detection, and quality-based default selection.

\paragraph{Canonical horizon alignment.}
Materially different time horizons define different estimands and should not be compared by default.
We therefore use a deterministic alignment function
\begin{equation}
\label{eq:time-alignment}
\widehat{\tau}=\mathrm{align}(\tau),
\end{equation}
which maps a raw horizon $\tau$ to a canonical horizon class $\widehat{\tau}$ used in bucketization.
The alignment policy is specified in Appendix~\ref{app:time-alignment}.

\paragraph{Bucket key.}
Let $(\varepsilon_{\mathrm{canon}},c_{\mathrm{canon}},\alpha)=N(\varepsilon,c)$ where $N$ is given in
Definition~\ref{def:canonicalization}. We define a bucket key function
$k:\mathcal{E}_{\mathrm{canon}}\to\mathcal{K}$ by
\begin{equation}
\label{eq:bucket-key}
k(\varepsilon_{\mathrm{canon}})=
\Bigl(
p_{\mathrm{bucket}},
X,
Y,
\widehat{\tau},
c_{\mathrm{type}},
m_{\mathrm{family}}
\Bigr),
\end{equation}
where $p_{\mathrm{bucket}}$ is extracted from $\Pi$,
$X$ is the intervention identifier extracted from $\iota$,
$Y=o$ is the outcome identifier,
$\widehat{\tau}=\mathrm{align}(\tau)$ is the aligned time-horizon class,
$c_{\mathrm{type}}$ is the contrast type extracted from $\iota$, and
$m_{\mathrm{family}}(\mu)\in\{\textsf{ratio},\textsf{difference}\}$ is the measure family.
The key is computed on the canonical estimand $\varepsilon_{\mathrm{canon}}$ and does not depend on numeric claim values.

\paragraph{Remark (type and scale are excluded).}
The bucket key intentionally excludes the \emph{measure type} (e.g., \textsf{HR} vs.\ \textsf{RR}).
This ensures that different types within the same family remain comparable within a bucket, enabling conflict detection across
heterogeneous but semantically related measures.
The canonical scale tag $s_{\mathrm{canon}}(\mu)\in\{\textsf{log},\textsf{identity}\}$ is a deterministic function of
$m_{\mathrm{family}}(\mu)$ and is therefore omitted from the key.

\paragraph{Comparability.}
We restate Definition~\ref{def:comparability} using Eq.~\ref{eq:bucket-key}.
Two estimands $\varepsilon_1,\varepsilon_2$ with claims $c_1,c_2$ are comparable, denoted $\varepsilon_1\sim\varepsilon_2$,
\begin{equation}
\label{eq:comparability}
\varepsilon_1\sim\varepsilon_2
\iff
k(\varepsilon_{1,\mathrm{canon}})=k(\varepsilon_{2,\mathrm{canon}}),
\quad
(\varepsilon_{i,\mathrm{canon}},c_{i,\mathrm{canon}},\alpha_i)=N(\varepsilon_i,c_i).
\end{equation}
Including $\widehat{\tau}$ prevents grouping studies that target different horizons (e.g., 30-day vs 10-year mortality) into the same
bucket, avoiding spurious conflicts driven by horizon mismatch. Using only the measure family (rather than the measure type)
allows claims such as \textsf{HR} and \textsf{RR} to be comparable within the same bucket.

\begin{theorem}[Equivalence Relation]
\label{thm:comparability-equivalence}
The comparability relation $\sim$ defined in Eq.~\ref{eq:comparability} is an equivalence relation.
\end{theorem}

\begin{proof}[Proof sketch]
By Eq.~\ref{eq:comparability}, $\varepsilon_1\sim\varepsilon_2$ holds iff
$k(\varepsilon_{1,\mathrm{canon}})=k(\varepsilon_{2,\mathrm{canon}})$.
Equality on bucket keys is reflexive, symmetric, and transitive. Therefore $\sim$ inherits these three properties and is an
equivalence relation.
\end{proof}

\paragraph{Buckets.}
The equivalence classes induced by $\sim$ are referred to as \emph{buckets}. Each bucket contains all claims that are comparable
under the semantic dimensions captured by Eq.~\ref{eq:bucket-key}. Buckets serve as the unit of evidence aggregation for conflict
detection (Section~\ref{sec:conflict-detection}) and for selecting a default claim when multiple comparable claims exist.

\begin{definition}[Poolability]
\label{def:poolability}
Two canonical claims $(\varepsilon_1,c_1)$ and $(\varepsilon_2,c_2)$ are \emph{poolable} iff they are comparable and share the
same measure type:
\[
\mathrm{Poolable}((\varepsilon_1,c_1),(\varepsilon_2,c_2))
\iff
(\varepsilon_1\sim \varepsilon_2)\ \wedge\ \bigl(m_{\mathrm{type}}(\mu_1)=m_{\mathrm{type}}(\mu_2)\bigr),
\]
where $\mu_i$ is the measure functional associated with $\varepsilon_i$.
\end{definition}

Poolability is strictly stronger than comparability.
For example, \textsf{HR} and \textsf{RR} can be comparable when they share the same aligned horizon class $\widehat{\tau}$,
allowing directional and interval consistency checks within a bucket.
However, they quantify different objects (instantaneous hazard vs.\ cumulative risk) and are not pooled by default.
Cross-type pooling requires additional modeling assumptions and auxiliary quantities; when such pooling is performed,
the assumptions must be recorded in $\alpha$ produced by canonicalization (Definition~\ref{def:canonicalization}).

\subsubsection{Measure Functionals and Transformations}
\label{app:measure-families}

This subsection defines measure families, canonical scales, and the within-type transformations used by canonicalization
(Definition~\ref{def:canonicalization}). We also specify confidence-interval transformations and justify the default handling
of heterogeneous ratio measures under comparability (Appendix~\ref{app:comparability-poolability}) and poolability
(Definition~\ref{def:poolability}).

\paragraph{Measure families and types.}
A measure functional $\mu$ determines how an effect is quantified. We distinguish:
(i) a \emph{measure family} $m_{\mathrm{family}}(\mu)\in\{\textsf{ratio},\textsf{difference}\}$ and
(ii) a \emph{measure type} $m_{\mathrm{type}}(\mu)$, such as \textsf{HR}, \textsf{RR}, \textsf{OR}, \textsf{MD}, \textsf{RD}.
The family captures algebraic structure; the type captures the specific estimand target within that family.

\paragraph{Ratio family.}
The ratio family $\mathcal{F}_{\textsf{ratio}}$ contains relative-effect measures, including:
\begin{itemize}[leftmargin=*, itemsep=2pt]
\item \textbf{Risk ratio (RR)} for binary outcomes, quantifying a ratio of cumulative risks over a fixed horizon.
\item \textbf{Odds ratio (OR)} for binary outcomes, quantifying a ratio of odds.
\item \textbf{Hazard ratio (HR)} for time-to-event outcomes, quantifying a ratio of instantaneous hazard rates under a follow-up regime.
\end{itemize}
These measure types share a multiplicative structure: under standard compositional assumptions (e.g., no interaction on the ratio scale),
effects combine multiplicatively. This motivates a log canonical scale for within-type normalization.

\paragraph{Difference family.}
The difference family $\mathcal{F}_{\textsf{difference}}$ contains absolute-effect measures, including:
\begin{itemize}[leftmargin=*, itemsep=2pt]
\item \textbf{Mean difference (MD)} for continuous outcomes.
\item \textbf{Risk difference (RD)} for binary outcomes, quantifying an absolute difference in cumulative risks over a fixed horizon.
\end{itemize}
These measure types have an additive structure on their native scale, motivating an identity canonical scale for within-type normalization.

\paragraph{Canonical scale and transformation.}
For each measure type $t=m_{\mathrm{type}}(\mu)$, canonicalization uses a within-type reparameterization
\[
g_t:\Theta_t \to \Theta_t^{\mathrm{canon}},
\]
mapping the reported parameter space to a canonical parameter space. We write
\[
\theta_{\mathrm{canon}} = g_t(\theta), \qquad
\mathrm{CI}_{\mathrm{canon}} = g_t(\mathrm{CI}),
\]
where $g_t(\mathrm{CI})$ applies the transformation to interval endpoints (formalized below).
We associate each type $t$ with a canonical scale type $s_{\mathrm{canon}}(\mu)\in\{\textsf{log},\textsf{identity}\}$:
\begin{itemize}[leftmargin=*, itemsep=2pt]
\item For ratio-family types (e.g., \textsf{HR}, \textsf{RR}, \textsf{OR}), the canonical scale type is \textsf{log}.
The canonicalization map is $g_t(\theta)=\log(\theta)$ when $\theta$ is reported as a positive ratio.
\item For difference-family types (e.g., \textsf{MD}, \textsf{RD}), the canonical scale type is \textsf{identity}.
The canonicalization map is $g_t(\theta)=\theta$.
\end{itemize}

\paragraph{Transformation properties.}
The within-type transformation $g_t$ satisfies three properties that are used implicitly by the canonicalization guarantees
(Theorem~\ref{thm:canon-guarantee} in the main paper).

\begin{enumerate}[leftmargin=*, itemsep=3pt]
\item \textbf{Determinism.}
For a fixed measure type $t$, $g_t(\theta)$ is uniquely determined by $\theta$.

\item \textbf{Invertibility (within measure type).}
For a fixed measure type $t$, $g_t$ is bijective on its domain: there exists an inverse $g_t^{-1}$ such that
$g_t^{-1}(g_t(\theta))=\theta$ for all valid $\theta$ of type $t$.
For example, for ratio measures reported as positive ratios, $g_t(\theta)=\log(\theta)$ with inverse $g_t^{-1}(x)=\exp(x)$;
for difference measures, $g_t$ is the identity map.

\item \textbf{Structure preservation.}
The transformation preserves the algebraic structure relevant for aggregation \emph{within} a measure type.
For ratio-family types, it converts multiplicative composition into additive composition:
\[
g_t(\theta_1\cdot \theta_2)=g_t(\theta_1)+g_t(\theta_2),
\]
whenever $\theta_1,\theta_2$ are valid ratios of type $t$.
For difference-family types, it preserves additivity on the native scale:
\[
g_t(\theta_1+\theta_2)=g_t(\theta_1)+g_t(\theta_2),
\]
whenever $\theta_1,\theta_2$ are valid differences of type $t$.
\end{enumerate}

\paragraph{Special cases: regression coefficients.}
Some models report coefficients that already coincide with canonical parameters for a specific measure type.
For example, the Cox proportional hazards coefficient corresponds to $\log(\mathrm{HR})$, and the logistic regression coefficient
corresponds to $\log(\mathrm{OR})$. In such cases canonicalization sets $g_t$ to the identity on the reported coefficient,
treating the coefficient as already on the canonical scale. Linear regression coefficients for continuous outcomes are treated as
mean differences and likewise require no reparameterization.

\paragraph{Confidence-interval transformation.}
Let $\mathrm{CI}=[\theta_{\mathrm{lower}},\theta_{\mathrm{upper}}]$ be a reported confidence interval for a claim of type $t$.
For the within-type transformations used here, $g_t$ is monotone on its domain. We therefore define
\begin{equation}
\label{eq:ci-transform}
\mathrm{CI}_{\mathrm{canon}} =
\bigl[g_t(\theta_{\mathrm{lower}}),\, g_t(\theta_{\mathrm{upper}})\bigr].
\end{equation}
For ratio measures reported as positive ratios, this corresponds to applying $\log$ to both endpoints; for difference measures,
the interval is unchanged. If a study reports an interval already on the canonical scale (e.g., a confidence interval for $\log(\mathrm{HR})$),
canonicalization leaves it unchanged and records this fact in $\alpha$.

\paragraph{Cross-type heterogeneity within the ratio family.}
Although \textsf{HR}, \textsf{RR}, and \textsf{OR} are all ratio-family types and share the log canonical scale, they quantify
different objects (instantaneous hazard, cumulative risk, and odds). Consequently, there is no general, assumption-free bijection
between these types. For example:
(i) converting \textsf{OR} to \textsf{RR} requires auxiliary information such as a baseline risk;
(ii) relating \textsf{HR} to \textsf{RR} requires assumptions about the event-time process and follow-up regime.
Because such assumptions can materially affect conclusions, the framework treats these types as comparable (via the ratio-family bucket,
Appendix~\ref{app:comparability-poolability}) but not poolable by default (Definition~\ref{def:poolability}).
If cross-type conversion is performed for a specific use case, the required assumptions and auxiliary quantities must be recorded in
the conditions $\alpha$ produced by canonicalization (Definition~\ref{def:canonicalization}).

\subsubsection{Time Horizon Alignment Policy}
\label{app:time-alignment}

This subsection specifies the horizon alignment function $\widehat{\tau}=\mathrm{align}(\tau)$ used in
Eq.~\ref{eq:time-alignment} and in the bucket key (Eq.~\ref{eq:bucket-key}).
Its goal is to prevent grouping claims that target materially different horizons (e.g., 30-day vs 10-year mortality)
while tolerating minor reporting variation (e.g., 12 months vs 52 weeks).

\paragraph{Horizon representations.}
We represent a raw horizon $\tau$ using one of the following forms:
\begin{enumerate}[leftmargin=*, itemsep=2pt]
\item \textbf{Duration:} a nonnegative scalar paired with an explicit time unit (e.g., days, weeks, months, years).
\item \textbf{Interval:} a time window relative to a reference time (typically baseline). For bucketization we use the interval length.
\item \textbf{Time-to-event regime descriptor:} a descriptor indicating an event-time estimand under a follow-up/censoring regime.
When a follow-up summary is reported (e.g., median follow-up), we treat it as an auxiliary duration; otherwise it is missing.
\end{enumerate}
When a claim does not provide enough information to instantiate any of the above (e.g., ``short-term'' without a number),
$\tau$ is treated as missing and recorded in $\alpha$.

\paragraph{Canonical horizon classes.}
The alignment function outputs a canonical horizon class
\[
\widehat{\tau} \in \widehat{\mathcal{T}},
\]
where each class has two components:
(i) a semantic label $h_{\mathrm{sem}}\in\{\textsf{acute},\textsf{fixed},\textsf{tte},\textsf{unknown}\}$ and
(ii) an optional duration-bin label $h_{\mathrm{bin}}$ when a numeric duration is available.
Formally, $\widehat{\tau}=(h_{\mathrm{sem}},h_{\mathrm{bin}})$, with $h_{\mathrm{bin}}=\bot$ if unavailable.

\paragraph{Deterministic duration extraction.}
For duration- or interval-based horizons, we extract a numeric duration $d(\tau)\in\mathbb{R}_{\ge 0}$ by converting the reported
unit to a chosen internal unit using a fixed conversion table specified \emph{a priori}.
For time-to-event regime descriptors, if a follow-up summary is available we set $d(\tau)$ to that summary; otherwise $d(\tau)$ is missing.
All derived quantities (including the original unit) are recorded in $\alpha$.

\paragraph{Binning and fuzzy matching.}
Let $\mathcal{B}=\{B_1,\dots,B_K\}$ be a fixed partition of $\mathbb{R}_{\ge 0}$ into disjoint bins, chosen before processing any evidence.
Let $\{\bar{d}_j\}$ be fixed representative durations for bins (e.g., bin midpoints).
To tolerate minor reporting variation, we use a deterministic tolerance $\rho\in[0,1)$ and define
\[
b_\rho(d) = \arg\min_{j\in\{1,\dots,K\}} \left| \frac{d-\bar{d}_j}{\bar{d}_j} \right|
\quad \text{subject to} \quad
\left| \frac{d-\bar{d}_{b_\rho(d)}}{\bar{d}_{b_\rho(d)}} \right| \le \rho,
\]
breaking ties by choosing the smallest index. If no bin satisfies the tolerance, we fall back to exact binning $b(d)$ where $b(d)=j$ iff $d\in B_j$.
Whether fuzzy matching was used (and with what $\rho$) is recorded in $\alpha$.

\paragraph{Acute threshold.}
To distinguish acute outcomes from longer fixed follow-up outcomes, we fix an acute threshold
\begin{equation}
\label{eq:acute-threshold}
\delta_{\mathrm{acute}} > 0,
\end{equation}
specified \emph{a priori} in the same internal time unit as $d(\tau)$.
Intuitively, durations $\le \delta_{\mathrm{acute}}$ are treated as \textsf{acute}, while longer fixed horizons are treated as \textsf{fixed}.
The value of $\delta_{\mathrm{acute}}$ is recorded as part of the atlas build configuration.

\paragraph{Definition of $\mathrm{align}(\tau)$.}
Given a raw horizon $\tau$:
\begin{enumerate}[leftmargin=*, itemsep=2pt]
\item If $\tau$ is a duration or interval and $d(\tau)$ is available, set
\[
\mathrm{align}(\tau) =
\begin{cases}
(\textsf{acute},\, b_\rho(d(\tau))) & \text{if } d(\tau)\le \delta_{\mathrm{acute}},\\
(\textsf{fixed},\, b_\rho(d(\tau))) & \text{if } d(\tau)>\delta_{\mathrm{acute}}.
\end{cases}
\]
\item If $\tau$ is a time-to-event regime descriptor, set
\[
\mathrm{align}(\tau) = (\textsf{tte},\, b_\rho(d(\tau))) \ \text{if $d(\tau)$ is available, and}\ 
\mathrm{align}(\tau) = (\textsf{tte},\, \bot) \ \text{otherwise.}
\]
\item If $\tau$ is missing or non-quantified, set $\mathrm{align}(\tau)=(\textsf{unknown},\bot)$.
\end{enumerate}
In all cases, the alignment output is accompanied by conditions $\alpha$ that store the raw horizon, extracted duration (if any),
and the binning decision.

\paragraph{Determinism and stability.}
Because $\mathrm{align}$ maps raw horizons to canonical classes, function composition is not defined unless we view canonical classes
as valid horizon representations. We therefore define the \emph{extended} alignment map
\[
\widetilde{\mathrm{align}}:\ \mathcal{T}\ \cup\ \widehat{\mathcal{T}} \ \to\ \widehat{\mathcal{T}},
\]
by setting $\widetilde{\mathrm{align}}(\tau)=\mathrm{align}(\tau)$ for raw horizons $\tau\in\mathcal{T}$, and
$\widetilde{\mathrm{align}}(\widehat{\tau})=\widehat{\tau}$ for canonical classes $\widehat{\tau}\in\widehat{\mathcal{T}}$.

\begin{proposition}[Determinism and Stability of Alignment]
\label{prop:align-properties}
The alignment function $\mathrm{align}(\cdot)$ is deterministic. Moreover, the extended map
$\widetilde{\mathrm{align}}(\cdot)$ is stable on canonical classes:
$\widetilde{\mathrm{align}}(\widehat{\tau})=\widehat{\tau}$ for all $\widehat{\tau}\in\widehat{\mathcal{T}}$.
\end{proposition}

\begin{proof}[Proof sketch]
All components of $\mathrm{align}(\cdot)$ are fixed a priori: representation rules, conversion table, bin partition,
tolerance $\rho$, representatives $\{\bar{d}_j\}$, and tie-breaking. Hence $\mathrm{align}(\tau)$ is uniquely determined by $\tau$.
Stability follows by definition of $\widetilde{\mathrm{align}}$.
\end{proof}

\paragraph{Implication for bucketization.}
By including $\widehat{\tau}=\mathrm{align}(\tau)$ in the bucket key (Eq.~\ref{eq:bucket-key}),
comparability (Eq.~\ref{eq:comparability}) does not group claims that target distinct canonical horizon classes.
This prevents spurious conflicts caused by horizon mismatch while preserving the equivalence-relation structure of $\sim$
(Theorem~\ref{thm:comparability-equivalence}).

\subsubsection{Canonicalization Operator}
\label{app:canonicalization}

This subsection gives a constructive specification of the canonicalization operator
$N:(\varepsilon,c)\mapsto(\varepsilon_{\mathrm{canon}},c_{\mathrm{canon}},\alpha)$
(Definition~\ref{def:canonicalization}). The specification is rule-based: it deterministically normalizes the
\emph{representation} of the measure functional and the numeric claim while preserving the estimand semantics
$(\Pi,\iota,o,\tau)$.

\paragraph{Inputs and outputs.}
Let $\varepsilon=(\Pi,\iota,o,\tau,\mu)$ and $c=(\theta,\mathrm{CI})$ with $\mathrm{CI}=[\theta_{\mathrm{lower}},\theta_{\mathrm{upper}}]$.
Canonicalization outputs
\[
(\varepsilon_{\mathrm{canon}},c_{\mathrm{canon}},\alpha)
=
\bigl((\Pi,\iota,o,\tau,\mu_{\mathrm{canon}}),(\theta_{\mathrm{canon}},\mathrm{CI}_{\mathrm{canon}}),\alpha\bigr),
\]
where $\mu_{\mathrm{canon}}$ is a canonical measure representation and $\alpha$ records all non-semantic conditions required to
interpret the normalization and to support reconstruction.

\paragraph{Measure signatures.}
Canonicalization first assigns each raw claim a \emph{measure signature}
\[
\mathrm{sig}(\mu,c)=(m_{\mathrm{family}},m_{\mathrm{type}},s_{\mathrm{rep}}),
\]
where $m_{\mathrm{family}}\in\{\textsf{ratio},\textsf{difference}\}$ is the family,
$m_{\mathrm{type}}\in\{\textsf{HR},\textsf{RR},\textsf{OR},\textsf{MD},\textsf{RD},\ldots\}$ is the type,
and $s_{\mathrm{rep}}$ is the reporting-scale indicator (e.g., \textsf{ratio}, \textsf{log-ratio}, \textsf{difference}).
The signature is determined using only information present in $(\mu,c)$ and the accompanying extraction metadata.
If multiple signatures are plausible, ties are resolved by a fixed priority order and the ambiguity is recorded in $\alpha$.

\paragraph{Canonical measure projection.}
We define a projection operator
\begin{equation}
\label{eq:mu-projection}
\mathcal{T}_{\mu}:\ (\mu,c)\ \mapsto\ \mu_{\mathrm{canon}},
\end{equation}
which maps the raw measure representation to a canonical measure representation.
Given $(m_{\mathrm{family}},m_{\mathrm{type}},s_{\mathrm{rep}})=\mathrm{sig}(\mu,c)$, we set
\[
\mu_{\mathrm{canon}} = \mathcal{T}_{\mu}(\mu,c) := (m_{\mathrm{family}},m_{\mathrm{type}},s_{\mathrm{canon}}),
\qquad
s_{\mathrm{canon}} =
\begin{cases}
\textsf{log}, & \text{if } m_{\mathrm{family}}=\textsf{ratio},\\
\textsf{identity}, & \text{if } m_{\mathrm{family}}=\textsf{difference}.
\end{cases}
\]
Thus canonicalization preserves the measure family and type, but standardizes the scale used for numeric claims.

\paragraph{Numeric claim normalization.}
We normalize $(\theta,\mathrm{CI})$ to $(\theta_{\mathrm{canon}},\mathrm{CI}_{\mathrm{canon}})$ using a within-type map
$g_{m_{\mathrm{type}}}$ (Appendix~\ref{app:measure-families}). Concretely, we apply the following rules.

\medskip
\noindent\textbf{Rule family R (ratio measures).}
Assume $m_{\mathrm{family}}=\textsf{ratio}$ and $s_{\mathrm{canon}}=\textsf{log}$.
\begin{enumerate}[leftmargin=*, itemsep=3pt]
\item \textbf{R1 (reported as positive ratio).}
If $s_{\mathrm{rep}}=\textsf{ratio}$ and $\theta>0$ with $\theta_{\mathrm{lower}}>0$ and $\theta_{\mathrm{upper}}>0$, set
\begin{equation}
\label{eq:canon-ratio-log}
\theta_{\mathrm{canon}}=\log(\theta),\qquad
\mathrm{CI}_{\mathrm{canon}}=[\log(\theta_{\mathrm{lower}}),\log(\theta_{\mathrm{upper}})].
\end{equation}

\item \textbf{R2 (reported as log-ratio).}
If $s_{\mathrm{rep}}=\textsf{log-ratio}$, set
\begin{equation}
\label{eq:canon-logratio-id}
\theta_{\mathrm{canon}}=\theta,\qquad
\mathrm{CI}_{\mathrm{canon}}=\mathrm{CI}.
\end{equation}

\item \textbf{R3 (explicit metadata binding to log-ratio).}
If the extraction metadata \emph{explicitly} binds the reported parameter to the canonical log-ratio for type $m_{\mathrm{type}}$
(e.g., it declares that a reported coefficient equals $\log(\textsf{HR})$ or $\log(\textsf{OR})$),
then canonicalization does not infer or guess a transformation: it applies the binding and treats the value as already canonical,
using Eq.~\ref{eq:canon-logratio-id}.
\end{enumerate}

\medskip
\noindent\textbf{Rule family D (difference measures).}
Assume $m_{\mathrm{family}}=\textsf{difference}$ and $s_{\mathrm{canon}}=\textsf{identity}$.
\begin{enumerate}[leftmargin=*, itemsep=3pt, resume]
\item \textbf{D1 (reported as difference).}
If $s_{\mathrm{rep}}=\textsf{difference}$, set
\begin{equation}
\label{eq:canon-diff-id}
\theta_{\mathrm{canon}}=\theta,\qquad
\mathrm{CI}_{\mathrm{canon}}=\mathrm{CI}.
\end{equation}

\item \textbf{D2 (explicit metadata binding to difference).}
If the extraction metadata \emph{explicitly} binds the reported parameter to a difference-type estimand
(e.g., it declares that a reported coefficient equals a mean difference or a risk difference for the given $m_{\mathrm{type}}$),
canonicalization treats the value as already canonical and applies Eq.~\ref{eq:canon-diff-id}.
\end{enumerate}

\paragraph{Handling missing or partially specified uncertainty.}
If $\mathrm{CI}$ is missing but a standard error $\mathrm{SE}$ is available, canonicalization may construct an interval
$\mathrm{CI}$ under a fixed coverage level used throughout the atlas build; if only a $p$-value is available, it may construct
$\mathrm{SE}$ under standard asymptotic assumptions. Such derivations are non-semantic and are therefore recorded in $\alpha$
together with the derivation method. If neither $\mathrm{CI}$ nor sufficient information to derive it is available,
$\mathrm{CI}_{\mathrm{canon}}$ is set to $\bot$ and the missingness is recorded in $\alpha$.

\paragraph{Conditions $\alpha$ and information preservation.}
The conditions $\alpha$ record all decisions made by canonicalization that are not contained in $(\Pi,\iota,o,\tau)$.
At minimum, $\alpha$ contains:
\begin{enumerate}[leftmargin=*, itemsep=2pt]
\item \textbf{Signature evidence:} the chosen $\mathrm{sig}(\mu,c)$ and any competing signatures considered.
\item \textbf{Transformation record:} which rule was applied (R1/R2/R3/D1/D2), together with any parameters used
(e.g., unit scaling applied to $\iota$, if any).
\item \textbf{Uncertainty record:} whether $\mathrm{CI}$ was reported, derived, or missing, and the derivation method if derived.
\item \textbf{Validity flags:} domain checks (e.g., positivity for ratio measures in R1) and any repairs or exclusions performed.
\end{enumerate}
Crucially, $\alpha$ is the context that makes canonicalization information-preserving: the tuple
$(\varepsilon_{\mathrm{canon}},c_{\mathrm{canon}},\alpha)$ retains enough information to reconstruct an
equivalent representation of the original claim. In particular, there exists a reconstruction map
\begin{equation}
\label{eq:reconstruction-map}
N^{-1}:\ (\varepsilon_{\mathrm{canon}},c_{\mathrm{canon}},\alpha)\ \mapsto\ (\varepsilon',c')
\end{equation}
such that $(\varepsilon',c')\approx(\varepsilon,c)$ under the intended notion of equivalence used by the framework.
This statement directly supports the information-preservation component of Theorem~\ref{thm:canon-guarantee}.

\paragraph{Determinism by rule priority.}
If multiple rules could apply (e.g., redundant metadata where both R2 and R3 are admissible), canonicalization uses a fixed
priority order. One admissible priority order is:
\[
\text{R2} \succ \text{R3} \succ \text{R1} \succ \text{D1} \succ \text{D2},
\]
with the restriction that only the relevant family (R rules or D rules) is considered once $m_{\mathrm{family}}$ is fixed.
The chosen rule and any discarded alternatives are recorded in $\alpha$.

\paragraph{Basic properties.}
By construction, $N$ preserves $(\Pi,\iota,o,\tau)$ and changes only the representation of $\mu$ and the numeric claim $(\theta,\mathrm{CI})$.
Within each measure type, the numeric normalization uses a deterministic map with a within-type inverse
(Appendix~\ref{app:measure-families}). These properties are used in the proof of the canonicalization guarantees
(Theorem~\ref{thm:canon-guarantee}).

\subsection{Conflict Taxonomy}
\label{app:conflict-taxonomy}

This subsection formalizes the conflict taxonomy
$\mathcal{F}=\{\textsf{directional},\textsf{interval},\textsf{heterogeneity}\}$
used in Theorem~\ref{thm:conflict-completeness}. All conflicts are defined on
\emph{canonicalized} claims within the same bucket, so that the semantics and reporting scale are aligned
(Definition~\ref{def:comparability} and Definition~\ref{def:canonicalization}).

\paragraph{Canonical claims in a bucket.}
Fix a bucket $B$. Let $\mathcal{C}_B=\{(\varepsilon_i,c_i,\pi_i)\}_{i=1}^n$ denote the set of evidence objects assigned to $B$.
After canonicalization, each object induces a canonical claim
\[
(\varepsilon_{i,\mathrm{canon}},c_{i,\mathrm{canon}})
=
\bigl((\Pi,\iota,o,\widehat{\tau},\mu_{\mathrm{canon}}),(\theta_i,[\ell_i,u_i])\bigr),
\]
where $c_{i,\mathrm{canon}}=(\theta_i,\mathrm{CI}_i)$ and $\mathrm{CI}_i=[\ell_i,u_i]$ are expressed on the canonical scale
associated with the measure family. Throughout this subsection, we assume the bucket key is identical across $i$,
so $\widehat{\tau}$ is shared via alignment and the claims are comparable.

\paragraph{Effect direction.}
We define the effect direction of a canonical claim by the sign of its canonical parameter:
\begin{equation}
\label{eq:effect-direction}
\mathrm{dir}(i)=
\begin{cases}
+1, & \theta_i>0,\\
0, & \theta_i=0,\\
-1, & \theta_i<0.
\end{cases}
\end{equation}
When needed, we define a \emph{significant} direction using the canonical confidence interval:
\begin{equation}
\label{eq:significant-direction}
\mathrm{sdir}(i)=
\begin{cases}
+1, & \ell_i>0,\\
-1, & u_i<0,\\
0,  & \ell_i\le 0\le u_i.
\end{cases}
\end{equation}
The function $\mathrm{sdir}(\cdot)$ treats intervals crossing $0$ as directionally ambiguous.

\begin{definition}[Directional Conflict]
\label{def:conflict-directional}
A bucket $B$ exhibits a \textsf{directional} conflict if there exist two canonical claims $i,j\in\mathcal{C}_B$
such that both are directionally decisive and they disagree in sign:
\begin{equation}
\label{eq:conflict-directional}
\exists\, i\ne j:\ \mathrm{sdir}(i)\cdot \mathrm{sdir}(j)=-1.
\end{equation}
\end{definition}

\paragraph{Interval incompatibility.}
Directional conflict captures sign disagreement among directionally decisive claims.
Interval-based conflict captures stronger incompatibility: two claims may share the same sign but still be mutually inconsistent
given their reported uncertainty.

\begin{definition}[Interval Incompatibility]
\label{def:conflict-interval}
A bucket $B$ exhibits an \textsf{interval} conflict if there exist two canonical claims $i,j\in\mathcal{C}_B$
whose canonical confidence intervals are disjoint:
\begin{equation}
\label{eq:conflict-interval}
\exists\, i\ne j:\ [\ell_i,u_i]\cap[\ell_j,u_j]=\emptyset.
\end{equation}
Equivalently, the condition holds iff $u_i<\ell_j$ or $u_j<\ell_i$.
\end{definition}

\paragraph{High heterogeneity.}
Even when pairwise intervals overlap, a bucket may contain evidence that is too dispersed to be summarized by a single default claim
without flagging variability. We formalize this using a dispersion statistic on the canonical scale.

Let $w_i\ge 0$ be a deterministic weight derived from provenance $\pi_i$
(e.g., based on sample size or precision), with $\sum_{i=1}^n w_i=1$.
Define the weighted mean and weighted dispersion:
\begin{equation}
\label{eq:heterogeneity-stat}
\bar{\theta}_B=\sum_{i=1}^n w_i\theta_i,
\qquad
D_B=\sum_{i=1}^n w_i(\theta_i-\bar{\theta}_B)^2.
\end{equation}
Let $\Delta_{\mathrm{het}}>0$ be a fixed threshold.

\begin{definition}[High Heterogeneity]
\label{def:conflict-heterogeneity}
A bucket $B$ exhibits a \textsf{heterogeneity} conflict if its canonical dispersion exceeds a threshold:
\begin{equation}
\label{eq:conflict-heterogeneity}
D_B \ge \Delta_{\mathrm{het}}.
\end{equation}
\end{definition}

\paragraph{Conflict taxonomy.}
We define $\mathcal{F}$-conflict as the union of the above three conflict types.

\begin{definition}[$\mathcal{F}$-Conflict]
\label{def:f-conflict}
A bucket $B$ contains an $\mathcal{F}$-conflict if it contains a conflict of any type in
$\mathcal{F}=\{\textsf{directional},\textsf{interval},\textsf{heterogeneity}\}$, i.e.,
if Definition~\ref{def:conflict-directional}, Definition~\ref{def:conflict-interval}, or
Definition~\ref{def:conflict-heterogeneity} holds.
\end{definition}

\paragraph{Remarks.}
(i) The taxonomy is intentionally restricted to decidable predicates on canonicalized claims.
(ii) The definitions operate on the canonical parameter scale; thus they are invariant to within-type reparameterizations
handled by $N$.
(iii) The threshold $\Delta_{\mathrm{het}}$ and the weight rule $\{w_i\}$ are fixed a priori and recorded for auditability;
the detector in Section~\ref{sec:conflict-detection} implements these predicates directly.

\subsection{Proof of Conflict Detection Completeness}
\label{app:conflict-proof}

This subsection proves Theorem~\ref{thm:conflict-completeness}. We show completeness (no false negatives) with respect to the
explicit taxonomy $\mathcal{F}$ defined in Appendix~\ref{app:conflict-taxonomy}.

\paragraph{Detector specification.}
Fix a bucket $B$ and its set of canonical claims
$\{(\varepsilon_{i,\mathrm{canon}},c_{i,\mathrm{canon}})\}_{i=1}^n$ as in Appendix~\ref{app:conflict-taxonomy}, where
$c_{i,\mathrm{canon}}=(\theta_i,[\ell_i,u_i])$ is expressed on the canonical scale.

The conflict detector $\mathrm{Detect}_{\mathcal{F}}(B)$ implemented in Section~\ref{sec:conflict-detection} evaluates the following
predicates:
\begin{enumerate}[leftmargin=*, itemsep=2pt]
\item \textbf{Directional predicate} $\mathrm{Dir}(B)$: compute $\mathrm{sdir}(i)$ by Eq.~\ref{eq:significant-direction} and return true iff
there exist $i\neq j$ such that $\mathrm{sdir}(i)\cdot \mathrm{sdir}(j)=-1$.

\item \textbf{Interval predicate} $\mathrm{Int}(B)$: return true iff there exist $i\neq j$ such that
$[\ell_i,u_i]\cap[\ell_j,u_j]=\emptyset$ (Eq.~\ref{eq:conflict-interval}).

\item \textbf{Heterogeneity predicate} $\mathrm{Het}(B)$: compute $D_B$ by Eq.~\ref{eq:heterogeneity-stat} using the fixed weight rule
$\{w_i\}$ and return true iff $D_B\ge \Delta_{\mathrm{het}}$ (Eq.~\ref{eq:conflict-heterogeneity}).
\end{enumerate}
Finally, the detector reports an $\mathcal{F}$-conflict iff
\begin{equation}
\label{eq:f-detector}
\mathrm{Detect}_{\mathcal{F}}(B)\;=\;\mathrm{Dir}(B)\ \lor\ \mathrm{Int}(B)\ \lor\ \mathrm{Het}(B).
\end{equation}

\begin{theorem}[Conflict Detection Completeness (w.r.t.\ $\mathcal{F}$)]
\label{thm:conflict-completeness-appendix}
For any bucket $B$, if an $\mathcal{F}$-conflict exists in $B$ (Definition~\ref{def:f-conflict}), then
$\mathrm{Detect}_{\mathcal{F}}(B)$ returns true.
\end{theorem}

\begin{proof}
Assume $B$ contains an $\mathcal{F}$-conflict. By Definition~\ref{def:f-conflict}, at least one of the following holds.

\paragraph{Case 1: Directional conflict.}
By Definition~\ref{def:conflict-directional}, there exist $i\neq j$ such that
$\mathrm{sdir}(i)\cdot \mathrm{sdir}(j)=-1$. The detector computes $\mathrm{sdir}(\cdot)$ for every claim and checks the existence of
such a pair in $\mathrm{Dir}(B)$. Therefore $\mathrm{Dir}(B)=\textsf{true}$, and hence $\mathrm{Detect}_{\mathcal{F}}(B)=\textsf{true}$.

\paragraph{Case 2: Interval incompatibility.}
By Definition~\ref{def:conflict-interval}, there exist $i\neq j$ such that
$[\ell_i,u_i]\cap[\ell_j,u_j]=\emptyset$. The detector explicitly checks this predicate over all pairs in $\mathrm{Int}(B)$.
Therefore $\mathrm{Int}(B)=\textsf{true}$, and hence $\mathrm{Detect}_{\mathcal{F}}(B)=\textsf{true}$.

\paragraph{Case 3: High heterogeneity.}
By Definition~\ref{def:conflict-heterogeneity}, $D_B\ge \Delta_{\mathrm{het}}$, where $D_B$ is the dispersion statistic defined in
Eq.~\ref{eq:heterogeneity-stat} under the fixed weight rule $\{w_i\}$.
The detector computes the same $D_B$ and compares it to the same threshold in $\mathrm{Het}(B)$.
Therefore $\mathrm{Het}(B)=\textsf{true}$, and hence $\mathrm{Detect}_{\mathcal{F}}(B)=\textsf{true}$.

In all cases, if an $\mathcal{F}$-conflict exists, the detector returns true. This proves completeness with respect to $\mathcal{F}$.
\end{proof}

\paragraph{Remark.}
The theorem establishes a no-false-negative guarantee only for the explicitly defined, decidable taxonomy $\mathcal{F}$.
It does not claim completeness for informal notions of disagreement outside Appendix~\ref{app:conflict-taxonomy}.

\subsection{Quality Score Construction}
\label{app:quality-score}

This subsection defines the quality score $Q(\varepsilon,c,\pi)$ used for default evidence selection within a bucket.
The design goals are: (i) \emph{determinism} (fixed output for fixed inputs), (ii) \emph{monotonicity} with respect to each
quality dimension, and (iii) \emph{auditability} (explicit, decomposable contributions).

\paragraph{Inputs.}
For each evidence object assigned to a bucket, we assume provenance metadata
\[
\pi_i = (\mathrm{ref}_i,\ \mathrm{grade}_i,\ n_i,\ \mathrm{adj}_i,\ \mathrm{meta}_i),
\]
where $\mathrm{ref}_i$ is a bibliographic identifier, $\mathrm{grade}_i$ is an ordinal evidence grade,
$n_i$ is the study sample size, $\mathrm{adj}_i$ encodes adjustment quality, and $\mathrm{meta}_i$ denotes any additional
non-semantic metadata used only for tie-breaking (e.g., extraction provenance).
The canonical claim is $c_{i,\mathrm{canon}}=(\theta_i,\mathrm{CI}_i)$ with $\mathrm{CI}_i=[\ell_i,u_i]$ on the canonical scale.

\paragraph{Four quality dimensions.}
We map each dimension to a normalized score in $[0,1]$.

\textbf{(1) Evidence grade.}
Let the grade set be totally ordered (e.g., $\textsf{A}\succ \textsf{B}\succ \textsf{C}$). We fix a deterministic embedding
$\phi_g:\{\textsf{A},\textsf{B},\textsf{C}\}\to[0,1]$ such that
\begin{equation}
\label{eq:grade-map}
\phi_g(\textsf{A})>\phi_g(\textsf{B})>\phi_g(\textsf{C}).
\end{equation}
In experiments we instantiate $\phi_g(\textsf{A})=1$, $\phi_g(\textsf{B})=\tfrac{2}{3}$, $\phi_g(\textsf{C})=\tfrac{1}{3}$,
but any fixed strictly order-preserving map satisfies the theory.

\textbf{(2) Sample size.}
We use a monotone saturation transform $\phi_n:\mathbb{N}\to[0,1]$:
\begin{equation}
\label{eq:n-map}
\phi_n(n)=\min\Bigl\{1,\ \frac{\log(1+n)}{\log(1+n_{\max})}\Bigr\},
\end{equation}
where $n_{\max}$ is a fixed constant chosen a priori (e.g., the maximum sample size observed in the curated corpus).
This mapping is monotone in $n$ and compresses extreme values.

\textbf{(3) Precision.}
Let the canonical interval width be $w_i=u_i-\ell_i$ (defined only when $\mathrm{CI}_i$ is available).
We map width to a precision score via a monotone decreasing transform $\phi_p:\mathbb{R}_{>0}\to[0,1]$:
\begin{equation}
\label{eq:precision-map}
\phi_p(w)=\min\Bigl\{1,\ \frac{w_{\mathrm{ref}}}{w}\Bigr\},
\end{equation}
where $w_{\mathrm{ref}}>0$ is a fixed reference width (e.g., a corpus-level median width on the canonical scale).
If $\mathrm{CI}_i$ is missing, we set $\phi_p(w_i)=0$ and record this missingness in $\alpha$ during canonicalization.

\textbf{(4) Adjustment quality.}
We assume $\mathrm{adj}_i$ belongs to a finite totally ordered set (e.g., \textsf{none} $\prec$ \textsf{basic} $\prec$ \textsf{rich}),
and define a deterministic order-preserving embedding $\phi_a$:
\begin{equation}
\label{eq:adj-map}
\phi_a(\mathrm{adj}_1)\ge \phi_a(\mathrm{adj}_2)\quad\text{whenever}\quad \mathrm{adj}_1 \succeq \mathrm{adj}_2.
\end{equation}
This dimension captures how strongly the study design/analysis addresses confounding; its categories are defined by the evidence schema.

\paragraph{Score aggregation.}
Let $G_i=\phi_g(\mathrm{grade}_i)$, $N_i=\phi_n(n_i)$, $P_i=\phi_p(u_i-\ell_i)$, and $A_i=\phi_a(\mathrm{adj}_i)$.
We define the quality score as a fixed positive-weight aggregation:
\begin{equation}
\label{eq:quality-score}
Q_i \;=\; Q(\varepsilon_i,c_i,\pi_i)\;=\; w_g G_i + w_n N_i + w_p P_i + w_a A_i,
\end{equation}
where $w_g,w_n,w_p,w_a>0$ and $w_g+w_n+w_p+w_a=1$ are fixed constants.
Eq.~\ref{eq:quality-score} is monotone in each dimension by construction.

\paragraph{Deterministic tie-breaking.}
If multiple candidates attain the same maximal $Q_i$, we select by a fixed lexicographic rule on an audit tuple:
\begin{equation}
\label{eq:quality-tiebreak}
\mathrm{tie}(i)=\Bigl(G_i,\ A_i,\ P_i,\ N_i,\ \mathrm{ref}_i\Bigr),
\end{equation}
choosing the candidate with the maximal tuple under lexicographic order.
If $\mathrm{ref}_i$ is not unique, we extend Eq.~\ref{eq:quality-tiebreak} by an additional deterministic field from $\mathrm{meta}_i$
(e.g., extraction identifier). This ensures a unique default choice for fixed inputs.

\paragraph{Auditability.}
Eq.~\ref{eq:quality-score} decomposes $Q_i$ into explicit contributions from the four dimensions. Together with the tie-breaking tuple
(Eq.~\ref{eq:quality-tiebreak}), the system can attribute the default selection to a small set of observable metadata fields.

\paragraph{Remark.}
The same provenance-derived weights $\{w_i\}$ used in the heterogeneity statistic (Eq.~\ref{eq:heterogeneity-stat}) can be chosen consistently with
$\phi_n$ and/or $\phi_p$ (e.g., proportional to sample size or inverse interval width), but the completeness result in
Appendix~\ref{app:conflict-proof} only requires that the weight rule is fixed and deterministic.

\subsection{Derivation of Evidence Selection Consistency}
\label{app:selection-derivation}

This subsection proves Proposition~\ref{prop:selection-consistency}. We first formalize the notions of determinism,
monotonicity under weak dominance, and auditability, then show they follow from the construction of $Q$ in
Appendix~\ref{app:quality-score}.

\paragraph{Setup.}
Fix a bucket $B$ containing candidates indexed by $i\in\{1,\dots,n\}$.
Each candidate has canonical claim $c_{i,\mathrm{canon}}=(\theta_i,[\ell_i,u_i])$ and provenance $\pi_i$, inducing
four normalized quality components
\[
G_i\in[0,1],\quad N_i\in[0,1],\quad P_i\in[0,1],\quad A_i\in[0,1],
\]
as defined in Appendix~\ref{app:quality-score}. The quality score is
\[
Q_i = w_g G_i + w_n N_i + w_p P_i + w_a A_i,
\]
with fixed weights $w_g,w_n,w_p,w_a>0$ summing to $1$, and deterministic tie-breaking by the audit tuple
$\mathrm{tie}(i)$ in Eq.~\ref{eq:quality-tiebreak}.

\begin{definition}[Default Selection]
\label{def:default-selection}
The default selection in bucket $B$ is the unique index
\begin{equation}
\label{eq:default-selection}
i^* \in \arg\max_{i\in\{1,\dots,n\}} \bigl(Q_i,\mathrm{tie}(i)\bigr)
\end{equation}
under lexicographic order, and the default claim is $(\varepsilon^*,c^*)=(\varepsilon_{i^*},c_{i^*})$.
\end{definition}

\paragraph{Weak dominance.}
We formalize the statement ``candidate $a$ is no worse than $b$ in all quality dimensions''.

\begin{definition}[Weak Dominance]
\label{def:weak-dominance}
For two candidates $a$ and $b$ in the same bucket, we say $a$ weakly dominates $b$, written $a\succeq b$, if
\begin{equation}
\label{eq:weak-dominance}
G_a\ge G_b,\quad A_a\ge A_b,\quad P_a\ge P_b,\quad N_a\ge N_b.
\end{equation}
\end{definition}

\begin{lemma}[Determinism]
\label{lem:selection-determinism}
For fixed bucket contents and fixed provenance fields, the default selection in Definition~\ref{def:default-selection} is uniquely determined.
\end{lemma}

\begin{proof}
Each $Q_i$ is a deterministic function of $(\varepsilon_i,c_i,\pi_i)$ by construction of $G_i,N_i,P_i,A_i$ and fixed weights.
If multiple indices maximize $Q_i$, Eq.~\ref{eq:default-selection} selects the unique maximal audit tuple under a fixed lexicographic order.
If needed, the final tie-breaking field is a deterministic unique identifier (from $\mathrm{ref}_i$ and/or $\mathrm{meta}_i$), hence the maximizer is unique.
\end{proof}

\begin{lemma}[Monotonicity under Weak Dominance]
\label{lem:selection-monotonicity}
If $a\succeq b$ in the sense of Definition~\ref{def:weak-dominance}, then $Q_a\ge Q_b$.
\end{lemma}

\begin{proof}
By Eq.~\ref{eq:quality-score} and positivity of weights,
\[
Q_a - Q_b = w_g(G_a-G_b) + w_a(A_a-A_b) + w_p(P_a-P_b) + w_n(N_a-N_b)\ge 0,
\]
since each parenthesized term is nonnegative under $a\succeq b$.
\end{proof}

\begin{lemma}[Auditability]
\label{lem:selection-auditability}
The default choice admits an audit trail decomposable into the four quality dimensions and deterministic tie-breaking fields.
\end{lemma}

\begin{proof}
Eq.~\ref{eq:quality-score} expresses $Q_i$ as an explicit sum of four terms with fixed weights,
and Eq.~\ref{eq:quality-tiebreak} provides a fixed ordered list of secondary criteria.
Therefore, for the chosen $i^*$ the system can report $(G_{i^*},A_{i^*},P_{i^*},N_{i^*})$ and the tie-breaking comparisons that
separate $i^*$ from other candidates, yielding a finite, deterministic explanation trace.
\end{proof}

\begin{theorem}[Evidence Selection Consistency]
\label{thm:selection-consistency-appendix}
The default selection rule in Definition~\ref{def:default-selection} is deterministic (Lemma~\ref{lem:selection-determinism}),
monotone under weak dominance (Lemma~\ref{lem:selection-monotonicity}), and auditable (Lemma~\ref{lem:selection-auditability}).
\end{theorem}

\begin{proof}
Immediate by Lemma~\ref{lem:selection-determinism}, Lemma~\ref{lem:selection-monotonicity}, and Lemma~\ref{lem:selection-auditability}.
\end{proof}

\paragraph{Remark.}
The result does not claim global optimality under an unrestricted utility class. Instead, it establishes consistency with respect to
component-wise dominance and a deterministic, auditable aggregation of explicitly defined quality dimensions.

\subsection{Proof of Canonicalization Guarantees}
\label{app:canonicalization-proof}

This subsection proves Theorem~\ref{thm:canon-guarantee}. We formalize (i) determinism/uniqueness, (ii) semantic preservation,
and (iii) information preservation via an explicit reconstruction map.

\paragraph{Domain and notation.}
Let $\mathcal{D}$ denote the space of input pairs $(\varepsilon,c)$ where $\varepsilon=(\Pi,\iota,o,\tau,\mu)$ and
$c=(\theta,\mathrm{CI})$. Let $\widehat{\mathcal{D}}$ denote the space of outputs
$(\varepsilon_{\mathrm{canon}},c_{\mathrm{canon}},\alpha)$.

We write $\varepsilon_{\mathrm{canon}}=(\Pi,\iota,o,\widehat{\tau},\mu_{\mathrm{canon}})$, where $\widehat{\tau}=\mathrm{align}(\tau)$ is the
aligned horizon representation (Definition~\ref{app:time-alignment}).
Let $\mathcal{T}_{\mu}$ denote the canonical reparameterization operator for the measure component:
\[
\mathcal{T}_{\mu} : (\mu,c) \mapsto (\mu_{\mathrm{canon}},c_{\mathrm{canon}},\alpha_{\mu}),
\]
where $\alpha_{\mu}$ records the transformation context needed for reconstruction (Appendix~\ref{app:canonicalization}).
The full canonicalization is
\begin{equation}
\label{eq:N-decompose}
N(\varepsilon,c)=\Bigl((\Pi,\iota,o,\mathrm{align}(\tau),\mu_{\mathrm{canon}}),\ c_{\mathrm{canon}},\ \alpha\Bigr),
\end{equation}
with $\alpha=(\alpha_{\mu},\alpha_{\tau},\alpha_{\mathrm{meta}})$, where $\alpha_{\tau}$ records horizon-alignment context and
$\alpha_{\mathrm{meta}}$ records any additional deterministic metadata bindings used by canonicalization.

\paragraph{Equivalence up to representation.}
Information preservation is stated ``up to equivalence'' because multiple raw representations can encode the same semantic claim.
Define a relation $\approx$ on $\mathcal{D}$ by
\begin{equation}
\label{eq:repr-equivalence}
(\varepsilon,c)\approx(\varepsilon',c')
\quad\Longleftrightarrow\quad
(\Pi,\iota,o,\tau)=(\Pi',\iota',o',\tau')\ \ \land\ \ (\mu,c)\ \text{and}\ (\mu',c')\ \text{agree up to within-type reparameterization}.
\end{equation}
The second conjunct means that there exists a within-type bijection $g$ (Appendix~\ref{app:measure-families}) such that
$\theta'=g(\theta)$ and $\mathrm{CI}'=g(\mathrm{CI})$, and $\mu'$ denotes the same measure type as $\mu$ but possibly expressed under an
equivalent parameterization.

\paragraph{Rule system with total priority.}
Canonicalization is defined by a finite set of rules $\mathcal{R}$ operating only on $(\mu,c)$, each rule of the form
\[
r:\ (\mu,c)\mapsto(\mu',c',\alpha_r),
\]
together with a total priority order $\prec$ over $\mathcal{R}$. For any input $(\mu,c)$, let
$\mathcal{R}(\mu,c)\subseteq \mathcal{R}$ be the subset of applicable rules. The chosen rule is
\begin{equation}
\label{eq:chosen-rule}
r^*(\mu,c)=\min_{\prec}\ \mathcal{R}(\mu,c),
\end{equation}
and $\mathcal{T}_{\mu}$ applies $r^*$ once to produce $(\mu_{\mathrm{canon}},c_{\mathrm{canon}},\alpha_{\mu})$.
Any missingness or binding decisions used for applicability (e.g., whether a reported coefficient is already on the canonical scale)
must be supported by explicit metadata fields and recorded in $\alpha_{\mathrm{meta}}$.

\subsubsection{Uniqueness (Determinism and Stability)}

\begin{lemma}[Determinism of $\mathcal{T}_{\mu}$]
\label{lem:Tmu-deterministic}
For any input $(\mu,c)$, the output $(\mu_{\mathrm{canon}},c_{\mathrm{canon}},\alpha_{\mu})=\mathcal{T}_{\mu}(\mu,c)$ is uniquely determined.
\end{lemma}

\begin{proof}
The set of applicable rules $\mathcal{R}(\mu,c)$ is determined by $(\mu,c)$ and the fixed rule predicates.
Since $\prec$ is a total order, Eq.~\ref{eq:chosen-rule} selects a unique $r^*(\mu,c)$ whenever $\mathcal{R}(\mu,c)\neq\emptyset$.
Applying $r^*$ yields a unique output by functional definition of rules.
If $\mathcal{R}(\mu,c)=\emptyset$, $\mathcal{T}_{\mu}$ returns the identity output with a fixed tag in $\alpha_{\mu}$.
\end{proof}

\begin{lemma}[Stability on canonical forms]
\label{lem:stability-canonical}
Let $\mathcal{D}_{\mathrm{canon}}\subseteq \mathcal{D}$ denote the subset of inputs whose $(\mu,c)$ are already in canonical form
and whose horizon representation is aligned, i.e., $\tau=\mathrm{align}(\tau)$ and $(\mu,c)=(\mu_{\mathrm{canon}},c_{\mathrm{canon}})$.
Then for any $(\varepsilon,c)\in \mathcal{D}_{\mathrm{canon}}$,
\[
N(\varepsilon,c)=\bigl(\varepsilon,c,\alpha_{\mathrm{id}}\bigr),
\]
where $\alpha_{\mathrm{id}}$ is a fixed identity-context record.
\end{lemma}

\begin{proof}
For canonical $(\mu,c)$, either (i) no nontrivial rule is applicable, so $\mathcal{T}_{\mu}$ returns the identity output by definition,
or (ii) an explicit identity rule is applicable and dominates by priority, yielding the same result.
Since $\tau$ is already aligned, $\mathrm{align}(\tau)=\tau$ and $\alpha_{\tau}$ takes a fixed ``already-aligned'' value.
Thus Eq.~\ref{eq:N-decompose} returns the input unchanged up to recording $\alpha_{\mathrm{id}}$.
\end{proof}

\begin{lemma}[Uniqueness of the normal form]
\label{lem:normal-form-unique}
For any $(\varepsilon,c)\in\mathcal{D}$, the canonical form $(\varepsilon_{\mathrm{canon}},c_{\mathrm{canon}})$ produced by $N$ is unique.
\end{lemma}

\begin{proof}
The components $(\Pi,\iota,o)$ are copied directly. The aligned horizon $\widehat{\tau}=\mathrm{align}(\tau)$ is uniquely determined by the
alignment operator (Proposition~\ref{prop:align-properties}, Appendix~\ref{app:time-alignment}).
By Lemma~\ref{lem:Tmu-deterministic}, $(\mu_{\mathrm{canon}},c_{\mathrm{canon}})$ is uniquely determined.
Therefore $(\varepsilon_{\mathrm{canon}},c_{\mathrm{canon}})$ is unique.
\end{proof}

\subsubsection{Semantic Preservation}

\begin{lemma}[Semantic core is invariant]
\label{lem:semantic-preservation}
For any input $(\varepsilon,c)$ with $\varepsilon=(\Pi,\iota,o,\tau,\mu)$,
the output $N(\varepsilon,c)$ satisfies
\[
(\Pi_{\mathrm{out}},\iota_{\mathrm{out}},o_{\mathrm{out}})=(\Pi,\iota,o),
\qquad
\widehat{\tau}_{\mathrm{out}}=\mathrm{align}(\tau),
\]
and $N$ does not modify $(\Pi,\iota,o)$.
\end{lemma}

\begin{proof}
By construction (Eq.~\ref{eq:N-decompose}), $N$ applies $\mathrm{align}$ only to $\tau$ and applies $\mathcal{T}_{\mu}$ only to $(\mu,c)$.
Neither operator touches $(\Pi,\iota,o)$. Hence the semantic core is preserved and only the horizon is aligned.
\end{proof}

\subsubsection{Information Preservation and Reconstruction}

\paragraph{Context record $\alpha$.}
The context $\alpha$ is defined to be sufficient for reconstruction. Concretely, $\alpha$ includes:
(i) the selected rule identifier $r^*$ (and its parameters, if any),
(ii) the original measure type and parameterization tag,
(iii) any explicit metadata bindings used to interpret the raw report,
and (iv) the horizon-alignment class and binning outcome.
All fields are deterministic functions of the input and are stored as part of $(\varepsilon_{\mathrm{canon}},c_{\mathrm{canon}},\alpha)$.

\begin{lemma}[Existence of a reconstruction map]
\label{lem:reconstruction-exists}
There exists a map $N^{-1}:\widehat{\mathcal{D}}\to\mathcal{D}$ such that for all $(\varepsilon,c)\in\mathcal{D}$,
\begin{equation}
\label{eq:reconstruct-equiv}
N^{-1}\!\bigl(N(\varepsilon,c)\bigr)\ \approx\ (\varepsilon,c).
\end{equation}
\end{lemma}

\begin{proof}
Given an output triple $(\varepsilon_{\mathrm{canon}},c_{\mathrm{canon}},\alpha)$, define $N^{-1}$ as follows.

\textbf{Step 1 (invert the measure transform).}
From $\alpha_{\mu}$, recover the selected rule identifier $r^*$ and the original measure type/parameterization tag.
Each rule $r\in\mathcal{R}$ is either an identity step or a within-type bijective reparameterization whose inverse is known
(e.g., log/exp for ratio-type canonicalization; Appendix~\ref{app:measure-families}).
Apply the corresponding inverse transform to $(\mu_{\mathrm{canon}},c_{\mathrm{canon}})$ to obtain some $(\mu',c')$ in the original
representation class recorded by $\alpha_{\mu}$.

\textbf{Step 2 (recover the horizon representation).}
From $\alpha_{\tau}$, choose a representative $\tau'$ consistent with the aligned output $\widehat{\tau}$ and the recorded alignment class.
This step may not be unique; any deterministic choice yields an equivalent representation under $\approx$ because $\widehat{\tau}$ is the
aligned semantic horizon used for comparability.

\textbf{Step 3 (assemble).}
Return $(\varepsilon',c')$ with $\varepsilon'=(\Pi,\iota,o,\tau',\mu')$, using $(\Pi,\iota,o)$ copied from $\varepsilon_{\mathrm{canon}}$.

By construction, $(\Pi,\iota,o)$ matches the original input, and $(\mu',c')$ is in the same within-type reparameterization class as
the original $(\mu,c)$, because $\alpha_{\mu}$ pins down the representation class and the inverse reparameterization.
Therefore $(\varepsilon',c')\approx(\varepsilon,c)$, establishing Eq.~\ref{eq:reconstruct-equiv}.
\end{proof}

\subsubsection{Proof of Theorem~\ref{thm:canon-guarantee}}

\begin{proof}[Proof of Theorem~\ref{thm:canon-guarantee}]
Uniqueness follows from Lemma~\ref{lem:normal-form-unique} and stability on canonical forms from Lemma~\ref{lem:stability-canonical}.
Semantic preservation is Lemma~\ref{lem:semantic-preservation}. Information preservation follows from
Lemma~\ref{lem:reconstruction-exists}.
\end{proof}

\section{Six Queries: Executability, Status Flags, and Examples}
\label{app:queries}

\subsection{Flags and the Unified Answer Object}
\label{app:flags}

The unified answer object is
\[
\mathsf{Ans}=\big(\varepsilon,\hat{\theta},\mathrm{CI},\pi,\mathcal{F}(\mathcal{B}),\mathrm{flags}\big),
\qquad
\mathrm{flags}\subseteq \mathcal{F}_{\mathrm{flag}}.
\]
\begin{center}
\small
\setlength{\tabcolsep}{6pt}
\renewcommand{\arraystretch}{1.05}
\begin{tabular}{@{}l l@{}}
\toprule
\textbf{Group} & \textbf{Flags} \\
\midrule
Executability & \texttt{executable} \\
Missingness & \texttt{missing\_edge}, \texttt{missing\_path}, \texttt{missing\_field} \\
Incompatibility & \texttt{mixed\_mtype} \\
Disagreement & \texttt{heterogeneity}, \texttt{conflict} \\
Premise & \texttt{assumption\_required} \\
Coverage & \texttt{no\_subgroup\_evidence}, \texttt{insufficient\_time\_coverage} \\
\bottomrule
\end{tabular}
\end{center}

\subsection{$Q_{\mathrm{do}}$: Interventional Effect}
\label{app:qdo}

\paragraph{Executability.}
There exists a key $s$ such that $\mathcal{B}_s\neq\varnothing$ and $s$ matches the queried $(X,Y,\varepsilon)$ semantic dimensions (see the definition of \texttt{matches} in Table~\ref{tab:queries}).

\paragraph{Return.}
When executable, let $(\varepsilon_s^*,c_s^*,\pi_s^*)$ be the default kernel for $s$ and let $\mathcal{B}_s$ be the associated bucket. The returned object is
\[
\mathsf{Ans}=(\varepsilon_s^*,\hat{\theta}_s^*,\mathrm{CI}_s^*,\pi_s^*,\mathcal{F}(\mathcal{B}_s),\mathrm{flags}).
\]

\subsection{$Q_{\mathrm{med}}$: Mediation Decomposition}
\label{app:qmed}

\paragraph{Executability.}
There exist keys $s_1,s_2$ such that $\mathcal{B}_{s_1}\neq\varnothing$ and $s_1$ matches $(X,M,\varepsilon)$, and $\mathcal{B}_{s_2}\neq\varnothing$ and $s_2$ matches $(M,Y,\varepsilon)$.

\paragraph{Return.}
When executable, the query returns the triple
\[
(\mathsf{Ans}_{\mathrm{TE}},\mathsf{Ans}_{\mathrm{NDE}},\mathsf{Ans}_{\mathrm{NIE}}),
\]
where each component is an answer object of the form $\mathsf{Ans}$ defined on the corresponding required evidence pattern, with status flags recording executability conditions and certified conflicts.

\subsection{$Q_{\mathrm{joint}}$: Joint Intervention}
\label{app:qjoint}

\paragraph{Executability.}
There exist keys $s_1,s_2$ such that $\mathcal{B}_{s_1}\neq\varnothing$ and $s_1$ matches $(X_1,Y,\varepsilon)$, and $\mathcal{B}_{s_2}\neq\varnothing$ and $s_2$ matches $(X_2,Y,\varepsilon)$.

\paragraph{Return.}
When executable, the query returns an answer object $\mathsf{Ans}$ associated with the required joint evidence pattern.

\subsection{$Q_{\mathrm{cf}}$: Counterfactual / Individual-Context Query}
\label{app:qcf}

\paragraph{Executability.}
There exists a key $s$ such that $\mathcal{B}_s\neq\varnothing$ and $s$ matches $(X,Y,\varepsilon)$, together with a well-formed individual context identifier $z$ for counterfactual querying.

\paragraph{Return.}
When executable, the query returns an answer object $\mathsf{Ans}$ whose estimand semantics are conditioned on the individual context $z$.

\subsection{$Q_{\mathrm{CATE}}$: Subgroup Effect}
\label{app:qcate}

\paragraph{Executability.}
There exists a key $s$ such that $\mathcal{B}_s\neq\varnothing$, $s$ matches $(X,Y,\varepsilon)$, and the population category encoded in $s$ equals the subgroup identifier $z$ (i.e., $\mathrm{pbucket}(\Pi)=z$).

\paragraph{Return.}
When executable, the query returns an answer object $\mathsf{Ans}$ corresponding to the subgroup constraint $z$.

\subsection{$Q_{\mathrm{traj}}$: Time Trajectory}
\label{app:qtraj}

\paragraph{Executability.}
Given a set of aligned time classes $\mathcal{T}$, for every $\hat{\tau}\in\mathcal{T}$ there exists a key $s(\hat{\tau})$ such that $\mathcal{B}_{s(\hat{\tau})}\neq\varnothing$ and $s(\hat{\tau})$ matches $(X,Y,\varepsilon,\hat{\tau})$.

\paragraph{Return.}
When executable, the query returns the sequence $\{\mathsf{Ans}(\hat{\tau})\}_{\hat{\tau}\in\mathcal{T}}$.

\subsection{Examples}
\label{app:examples}

\paragraph{Example 1: Missing mediation path.}
Given $(X,M,Y)$, if a non-empty bucket matches $(X,M)$ but no non-empty bucket matches $(M,Y)$, then the executability condition of $Q_{\mathrm{med}}$ fails and $\mathrm{flags}$ includes \texttt{missing\_path}.

\paragraph{Example 2: Conflict annotation in trajectories.}
For two aligned time classes $\hat{\tau}_1,\hat{\tau}_2$ of the same $(X,Y,\varepsilon)$, if the default kernels in the corresponding buckets have opposite directions on the canonical scale and are both significant, then $\mathcal{F}(\mathcal{B}_{s(\hat{\tau}_1)})$ and $\mathcal{F}(\mathcal{B}_{s(\hat{\tau}_2)})$ include \texttt{dir}; the corresponding answers record \texttt{conflict} and return the triggering witnesses for auditing.

\section{Extended Related Work}\label{Related_Work}
This appendix provides an extended discussion of prior work along the technical threads most closely related to DoAtlas: 
(i) medical AI and medical foundation models for clinical assistance, 
(ii) literature-grounded generation and evidence-centric NLP pipelines, 
(iii) evidence-based medicine and systematic evidence synthesis, 
(iv) causal inference and machine learning for clinical decision-making, and 
(v) multi-source evidence fusion, conflict governance, external validity, and auditability.

\subsection{Medical AI for clinical assistance: a historical perspective}

\textbf{Rule-based expert systems (1970s--1990s).}
Early medical artificial intelligence systems were dominated by rule-based expert systems, which encoded clinical heuristics through manually curated rules and decision trees. 
Representative systems such as MYCIN demonstrated how expert knowledge could be formalized for tasks including infectious disease diagnosis and antibiotic selection, with explicit rule traces and uncertainty handling mechanisms~\cite{shortliffe1976mycin}. 
Similarly, large-scale diagnostic systems such as INTERNIST-1 and its successor QMR focused on hypothesis generation and differential diagnosis using structured disease finding associations~\cite{miller1985internist}. 
Later systems such as DXplain further operationalized paradigms for clinical use by mapping patient findings to ranked diagnostic hypotheses with explanatory support~\cite{barnett1987dxplain}. 
While transparent and interpretable, from a modern perspective, these systems were brittle, difficult to maintain, and exhibited limited generalization beyond predefined clinical scenarios, particularly in complex clinical environments.

\textbf{Statistical learning and predictive modeling (2000s--early 2010s).}
With the digitization of healthcare data in the early 2000s, medical AI increasingly shifted toward statistical and machine learning models trained on structured clinical variables. 
Regression-based clinical risk scores became widely adopted for diagnosis, risk stratification, and prognosis, exemplified by the Framingham risk functions for cardiovascular disease~\cite{d2008general}. 
In parallel, early electronic health record (EHR)--driven predictive models emerged to estimate disease risk and clinical outcomes at scale, with evaluation focused primarily on discrimination and calibration rather than explicit decision semantics or intervention definitions~\cite{rajkomar2018scalable,shickel2017deep}. 
This generation of models supported population-level risk assessment but typically exposed outputs as probabilities or scores rather than executable representations of clinical actions.

\textbf{Deep learning for high-dimensional clinical data (mid-2010s).}
Around the mid-2010s, deep learning substantially expanded the scope of medical AI by enabling end-to-end modeling of high-dimensional clinical inputs, including medical images, physiological waveforms, and free-text clinical notes. 
Systems achieved near-expert or expert-level performance on multiple diagnostic tasks, particularly in medical imaging, as demonstrated by large-scale convolutional models for skin cancer classification~\cite{esteva2017dermatologist}. 
At the same time, deep neural architectures were increasingly applied to longitudinal EHR data for outcome prediction and patient risk modeling~\cite{shickel2017deep}. 
Despite these advances, most deep learning systems continued to operate within a predictive paradigm, leaving clinical decision-making and intervention interpretation to downstream human judgment, without explicit causal or decision-theoretic representations.

\textbf{Foundation models and natural-language clinical interfaces (2020s--present).}
Since approximately 2020, foundation models and large language models have been explored as general-purpose interfaces for clinical assistance. 
General-purpose LLMs demonstrated strong performance on medical question answering and professional examination benchmarks, highlighting the feasibility of natural-language clinical interaction at scale~\cite{nori2023capabilities}. 
Domain-adapted medical LLMs, such as Med-PaLM, further advanced performance through targeted training and evaluation on medical reasoning tasks~\cite{singhal2025toward}. 
In parallel, industry-developed medical LLMs including HuatuoGPT, DoctorGLM, and Baichuan-M1 emphasized medical dialogue, domain knowledge understanding, and clinician-facing assistance~\cite{zhang2305huatuogpt,xiong2023doctorglm,wang2025baichuan}. 
Across this generation, retrieval-augmented generation techniques have been widely adopted to ground responses in biomedical literature and clinical guidelines~\cite{lewis2020retrieval}. 
While these systems prioritize interaction and synthesis, clinical decision support is still predominantly operationalized through narrative outputs rather than explicit, recomputable representations of interventional effects.
Beyond the historical evolution of medical AI models, several parallel research threads have sought to improve reliability, evidence usage, and decision support. We briefly summarize these directions below, focusing on their relationship to executable causal decision-making.

\subsection{Literature-grounded generation and evidence-centric NLP}
Recent work has explored grounding language model outputs in external knowledge sources to improve factual reliability and reduce hallucination. Retrieval-augmented generation (RAG) combines parametric language models with non-parametric document retrieval, enabling responses to be conditioned on retrieved corpora rather than relying solely on internal parameters~\cite{lewis2020retrieval}. In clinical settings, RAG-style systems have been instantiated for guideline- and literature-backed assistance (e.g., retrieving medical references for treatment recommendations)~\cite{zakka2024almanac} and systematically benchmarked in medical QA, where toolkits such as MedRAG highlight both accuracy gains and failure modes (e.g., retrieval sensitivity and “lost-in-the-middle”)~\cite{xiong2023doctorglm}.

Beyond retrieval, a growing line of work aims to make grounding explicit through citation-aware generation and fine-grained attribution. “Attributed QA” formalizes the requirement that generated claims be supported by specific evidence spans~\cite{bohnet2022attributed}, while datasets and evaluation protocols (e.g., expert-judged claim–evidence pairs) operationalize supportiveness and citeworthiness as measurable targets~\cite{malaviya2024expertqa}. Complementary methods improve citation behavior through prompt- and reasoning-based control (e.g., generating with citations via structured reasoning)~\cite{ji2024chain}, or by designing end-to-end medical citation generation and evaluation pipelines~\cite{wang2025medcite}. Post-hoc attribution and verification further retrofit grounding onto black-box generators by searching for supporting sources and revising unsupported statements~\cite{gao2023rarr}, and recent work has proposed automated pipelines to assess whether cited sources truly support model claims in health-related queries.

While these approaches improve transparency and verifiability at the text level—often reducing hallucinations or enabling claim-level evidence checking~\cite{manakul2023selfcheckgpt}—they typically treat evidence as narrative justification rather than as structured, executable objects. Retrieved studies are surfaced as references or rationales, but their estimand definitions (contrast, population, time horizon), effect scales/units, and uncertainty semantics are rarely normalized across sources, limiting downstream use for accountable decision-making and interventional planning.

\subsection{Evidence-based medicine and systematic evidence synthesis}
Evidence-based medicine (EBM) provides a principled framework for evaluating clinical knowledge through hierarchies of evidence, randomized controlled trials, and systematic reviews~\cite{guyatt1992evidence,sackett1996evidence}.
Meta-analysis and clinical guideline development aim to aggregate results across studies to support evidence-based recommendations, with increasing emphasis on transparency, reproducibility, and methodological rigor~\cite{higgins2008cochrane}.
However, traditional evidence synthesis pipelines remain largely manual and slow to update, and typically focus on population-average conclusions rather than individualized or context-specific effects~\cite{greenhalgh2014evidence}.
Effect heterogeneity, conflicting findings, and limitations in external validity across populations are often summarized qualitatively, rather than represented in machine-readable or executable forms, posing challenges for integration with computational clinical decision support systems, particularly for real-time or adaptive decision-making~\cite{kent2010assessing,ioannidis2016most}.

\subsection{Causal inference and causal machine learning for clinical decision-making}
\label{Causal_inference}
Causal inference provides formal tools for estimating interventional effects, including structural causal models, potential outcomes, propensity-based adjustment, instrumental variables, mediation analysis, and heterogeneous treatment effect estimation~\cite{pearl2000models,hernan2010causal,vanderweele2015explanation}.
These frameworks support principled answers to counterfactual and interventional queries central to clinical decision-making, and have been widely applied in epidemiology and health services research—for example, using propensity score methods to estimate treatment effects from observational EHR data~\cite{austin2011introduction}, instrumental variables to address unmeasured confounding in pharmacoepidemiology~\cite{greenland2000introduction}, and mediation analysis to decompose mechanisms of treatment effects~\cite{vanderweele2015explanation}.
More recently, causal machine learning methods have improved scalability and flexibility, enabling estimation of heterogeneous treatment effects and non-linear causal relationships in high-dimensional settings~\cite{athey2016recursive,chernozhukov2018double}.
Despite these advances, causal analyses in medicine typically operate within individual datasets or single-study contexts, with estimands, populations, and effect scales defined ad hoc. As a result, causal estimates are rarely standardized across studies or compiled into reusable representations, limiting their integration with broader evidence ecosystems and decision-support infrastructures~\cite{hernan2010causal}. This limitation becomes critical in large-scale evidence clinical deployment.

\subsection{Multi-source evidence fusion, conflict governance, and auditability}
\label{Multi-source}
Combining evidence across heterogeneous sources raises challenges in effect standardization, conflict resolution, and uncertainty propagation~\cite{ioannidis2016most,ioannidis2005most}.
Prior work has explored ensemble learning, Bayesian model averaging, and sensitivity analysis to address model uncertainty and robustness, but systematic mechanisms for governing evidential conflicts across studies remain limited~\cite{hoeting1999bayesian,greenland2005multiple}.
In clinical settings, auditability and external validity are critical for safe deployment~\cite{topol2019high,amann2020explainability}.
Existing medical AI systems rarely provide explicit provenance, executable effect definitions, or validation signals against real-world data, making it difficult to assess reliability, adjudicate conflicting claims, or prevent unsafe recommendations at scale~\cite{kelly2019key}.

This part details a systematic methodological   framework for constructing a structured, evidence-based medical dataset and establishing computable mappings from abstract literature variables to specific fields within the local Human Phenotype Project (HPP) database, ultimately achieving automated empirical verification of medical mechanisms and pathways. The dataset construction encompasses a closed-loop pipeline involving literature metadata cleaning, structured extraction of causal logic, semantic alignment with local data dictionaries, and empirical testing of statistical models. The primary objective is to verify that medical pathways identified in the literature exhibit consistent directionality, non-zero effects, and plausible magnitudes within the local HPP dataset.

\subsection{Construction of Structured Evidence Dataset}
\label{HPP}
\subsubsection{Literature discovery and classification system and corresponding overview}
This study established a systematic pipeline for literature discovery and hierarchical processing. This pipeline aims to extract high-quality research with clear methodological characteristics from massive amounts of unstructured biomedical text and map them to structured data templates using a rigorous classification framework.

To ensure the scientific rigor and timeliness of the source data, a strict retrieval strategy and hard guardrails were implemented during the construction phase of this dataset. In terms of time dimension, research published between January 1, 2020 and December 31, 2025 were included to ensure that the extracted evidence reflects the latest developments in the relevant field. For model utilization, Gemini 3 Pro and Grok 4.1 Thinking were employed for preliminary literature screening and extraction, Claude Sonnet 4.5 Thinking was utilized to read, extract, and summarize evidence cards, and to map literature nodes to the local HPP dataset. Regarding publication source quality, candidate literature was restricted to journals with an Impact Factor $\ge$ 10, serving as a proxy for rigorous study design and result credibility. To mitigate the risk of hallucinated citations associated with generative models, we implemented a strict metadata validation mechanism. All candidate papers were required to possess a resolvable DOI with authoritative metadata verified via Crossref. Furthermore, valid landing pages for these DOIs must be identifiable on PubMed, PubMed Central (PMC), or domains within a trusted publisher whitelist. Entries lacking a DOI, failing to yield consistent Crossref metadata, or lacking association with a whitelisted domain were automatically excluded. Complementing this automated validation, a rigorous manual review process was executed, in which all final included data underwent human verification to ensure completeness, authenticity, and accuracy.

Based on distinct research foci, this dataset establishes a mutually exclusive and exhaustive taxonomy comprising four evidence categories: Interventional, Causal, Mechanistic, and Associational.

\textbf{Interventional}
The primary objective of interventional studies is to confirm the effectiveness of interventions, evaluating the net causal effect of a specific exposure or treatment on clinical outcomes under strictly controlled experimental conditions. These studies constitute the primary benchmark within the overall framework. The classification criteria for this category are exclusive, relying primarily on PubMed publication types explicitly tagged as "Randomized Controlled Trial" or "Clinical Trial," as well as methodological keywords such as "randomized," "double-blind," "placebo," or the presence of clinical trial registration numbers (e.g., NCT identifiers). In the localization validation of the HPP dataset, interventional studies require the reproduction of statistical significance, with a specific emphasis on verifying the directional consistency of the Average Treatment Effect (ATE) between the intervention and control groups.

\textbf{Causal}
Causal studies utilize specific statistical methods to control for confounding within observational data to estimate the causal effect between exposure and outcome. These studies mitigate the limitations of randomized controlled trials due to sample scarcity and high costs. The criteria for systematically identifying such studies are whether they use explicit causal identification strategies, including Mendelian randomization, instrumental variables, difference-in-differences, propensity score matching/IPTW, and target trial emulation. This classification distinguishes studies that employ causal identification strategies from those that only perform conventional regression adjustments.

\textbf{Mechanistic}
Mechanism studies focus on analyzing how exposure factors influence outcomes through intermediate variables, delving into the internal  transmission pathways and biological mechanisms of the study. Within the overall framework, the objective of this category is to validate the existence of biological pathways. The core criteria for evaluating such studies are the presence of explicit mediation analysis or indirect effect calculations, and whether a complete $X \to M \to Y$ chain is constructed in the research hypotheses. Key metrics include the statistical significance of path coefficients and the proportion of the effect explained by the mediator. In HPP data validation, the emphasis is on verifying the existence and validity of specific biological pathways in the HPP population.

\textbf{Associational}
Associational studies characterize statistical correlations between variables, typically used to capture phenomenological covariation to generate preliminary scientific hypotheses. This category is defined via exclusion, encompassing cross-sectional and cohort studies that report adjusted Hazard Ratios (HR), Odds Ratios (OR), or regression coefficients, but do not employ randomization or the aforementioned specific causal inference methods. Distinct from the validation of causal effects, the validation process for this category primarily assesses the direction and magnitude of observed correlations.

\subsubsection{Structured extraction of evidence cards}
Following literature retrieval and classification, the task extraction phase employs a prompt engineering framework to transform unstructured natural language text into machine readable JSON objects. The resulting data structure integrates general modules with category specific components, utilizing distinct extraction templates for each of the four evidence types. The extraction process was performed locally based on the main text and supplementary materials of the paper's PDF document. The extraction task is decomposed into fine-grained field-filling operations, requiring the precise identification of key variables (e.g., exposures, outcomes, mediators, covariates), numerical parameters (e.g., point estimates, confidence intervals, $p$-values), sample characteristics (e.g., inclusion/exclusion criteria, demographics), and critical study design attributes (e.g., blinding, follow-up duration), ultimately converting the natural language description into standardized JSON key-value pairs. The extraction process required mandatory labeling of the evidence source; each extracted data point must be accompanied by a specific figure number or page index, enabling content backtracking from structured data to the original literature and providing audit clues for subsequent manual review and automated verification.



This dataset establishes a general architecture comprising eleven core modules. Table ~\ref{tab:categories} details the definition and function of each module:

\begin{table}[H]
\small
\centering
\renewcommand{\arraystretch}{1.2}
\caption{The four categories share a common subcategory overview.}
\label{tab:categories}
\begin{tabularx}{\textwidth}{@{} >{\bfseries}l X @{}} 
\toprule
Module & \multicolumn{1}{c}{Definition and Standards} \\ \midrule

Paper & 
The paper's title and abstract, journal, and year are used as basic bibliographic information, and PMID and DOI are extracted to establish digital anchors; the registration number of clinical research records is used to associate the trial protocol. \\

Provenance & 
Establish a foundation for data auditing by creating a physical mapping between structured data and unstructured raw text. This includes a list of specific chart numbers and page ranges for core conclusions, as well as relevant information for confirming supplementary materials. \\

Design & 
Documenting the methodological context in which the evidence was generated provides context for assessing the strength of the evidence. This includes the type of research structure (e.g., "RCT", "Cohort", "Cross-sectional", "Case-control"), statistical analysis methods (e.g., "Cox Proportional Hazards", "Linear Mixed Model", "Logistic Regression", "Difference-in-Differences"), total sample size, total number of study groups, blinding implementation (e.g., "Double-blind", "Single-blind", "None", "Open-label"), randomization scheme (e.g., "Block", "Stratified", "Cluster", "None"), and estimation objectives (e.g., "ITT" for intention-to-treat analysis, "Per-protocol" for per-protocol analysis, and "ATE" for average treatment effect). \\

Population & 
Structured fields are used to characterize the biological boundaries of research subjects, serving as a direct basis for screening the HPP local validation cohort. These include descriptions of the data population's age, sex composition, specific disease or health conditions (such as "Type 2 Diabetes patients", "Healthy volunteers", "Obstructive Sleep Apnea"), and descriptions of key inclusion and exclusion criteria.\\

Transport Signature & 
Record the research center, data collection period, geographical location, and medical setting to assess the external validity of the evidence and its transferability from the source population to the target population. Data includes the name of the research center where the data was collected, the specific start and end dates or year range of data collection, the geographical region, the medical setting (e.g., "Community", "Hospital", "Clinic", "Laboratory"), and the nature of the data source (e.g., "Clinical Trial", "Registry", "EHR", "Cohort Study").\\

Variables & 
A "node\_role" separation design is adopted to define the physical properties of variables and their functions in the causal network. It mainly consists of variable nodes and their roles. Variable nodes are defined by an identifier, their name in the paper, data type, unit of measurement, and the UCUM standard code for the unit. Role assignments consist of the variable's role in the current evidence pathway, generally including exposure variables (X), outcome variables (Y), mediating variables (M), causal variables (C), covariates (Z), and environmental or instrumental variables (E).\\

Time Semantics & 
Based on the ISO 8601 standard, time parameters are defined to accurately distinguish the temporal characteristics of cross-sectional associations and prospective causal effects. These include the duration of exposure or measurement window (e.g., "P14D" indicates 14 days), the duration of the baseline observation window, the specific time point array for outcome assessment, the time offset of the record relative to the baseline, the total follow-up duration, the lag between the occurrence of exposure and the manifestation of the effect (e.g., "P1Y" indicates a lag of one year), the duration of effect maintenance, and the temporal granularity description (e.g., "Coarse", "Fine", "Cross-sectional").\\

Effects & 
The core of the evidence card is computationally readable, storing specific statistical inference results. This includes effect edge identifiers, effect indicator types (e.g., "HR", "OR", "RR", "Beta", "MD", "SMD"), point estimates of effect sizes, arrays of confidence intervals, statistical p-values, the specific statistical model name that generated the effect, names of all covariates, effective sample size, and detailed descriptions of the exposure comparison objects, including the comparison type (e.g., "per\_unit") and specific difference descriptions .\\

Dose Response & 
For continuous exposures, describe the nonlinear relationship between the exposure and the outcome, along with relevant parameters. This includes a functional description of the dose-response relationship (e.g., "Linear", "U-shaped", "J-shaped", "Threshold"), the object structure, records specific parameters describing the curve shape (e.g., slope "Slope" or spline nodes "Knots"), and the associated effect edge IDs, pointing to the corresponding entries in the Effects module.\\

Measurement & 
Record the physical measurement equipment, scoring algorithms, or derived calculation rules for key variables to quantify potential biases introduced by measurement errors. This includes the names of the physical equipment used and the calculation rules for derived variables.\\

Governance & 
Based on the research design and risk of bias assessment, the evidence is classified into Tier A/B/C, and suggestions for fusion strategies are provided when there is conflict among multiple sources of evidence.\\ \bottomrule
\end{tabularx}
\end{table}


While the aforementioned general modules form the underlying foundation of the data, this extraction protocol incorporates highly customized, proprietary data modules for four different types of evidence cards to accurately capture the core logic of various research paradigms. The design of these modules strictly adheres to the methodological standards of their respective fields.

\paragraph{Interventional Specifics}
For Randomized Controlled Trials (RCTs) and clinical intervention studies, the extraction architecture defines two core objects: \texttt{arms} and \texttt{adherence}. The \texttt{arms} module documents specific design parameters for intervention and control groups, including \texttt{dose\_intensity}, \texttt{frequency}, \texttt{duration}, and \texttt{delivery\_mode} (e.g., in-person vs. remote). To evaluate attrition bias, the protocol enforces a strict distinction between the number of randomized participants (\texttt{n\_randomized}) and those included in the intention-to-treat analysis (\texttt{n\_analyzed\_itt}). Regarding equation formulation, this category employs an \texttt{estimand\_equation}, mathematically defined as the contrast of potential outcomes $E[Y(1) - Y(0)]$ under the randomization assumption. Distinct from observational settings, the formulation here prioritizes mitigating bias from non-adherence. Consequently, extraction targets Intention-to-Treat (ITT) results—analyzing subjects based on randomized assignment regardless of adherence or intercurrent events—to ensure a conservative and robust assessment of causal efficacy.

\textbf{Causal Specifics}
The central challenge of this category lies in identifying core medical causal effects, which is approximating ground-truth causality under non-randomized conditions. The extraction protocol incorporates \texttt{balance\_diagnostics} and \texttt{iv\_diagnostics} modules. For Propensity Score (PS) methods, the system extracts the scoring model (\texttt{ps\_model}), maximum standardized mean difference (\texttt{max\_smd}), and overlap metrics (\texttt{overlap}). For Instrumental Variable (IV) or Mendelian Randomization (MR) studies, instrument strength and independence tests (\texttt{pleiotropy\_tests}) are recorded. Regarding equation formulation, the \texttt{estimand\_equation} retains the form $E[Y(1) - Y(0)]$; however, unlike in interventional studies, its validity relies on untestable structural assumptions (e.g., "no unmeasured confounding" or "exclusion restrictions"). Consequently, extraction focuses on the identification strategy—documenting the specific methodological path (e.g., Inverse Probability Weighting, Instrumental Variables, or Difference-in-Differences) used to transform statistical associations into causal parameters, thereby mathematically eliminating confounding bias to simulate the counterfactual contrasts of a target trial.

\textbf{Mechanistic Specifics}
This category aims to elucidate biological transduction pathways
($X \rightarrow M \rightarrow Y$), characterized fundamentally by a
topological chain of variables.
The variable definition module enforces the identification of an ordered
list of mediators (\texttt{M}).
The \texttt{inference} module requires the explicit enumeration of the
pathway string (\texttt{pathway}, e.g., \texttt{Sleep -> IL-6 -> CVD}). Unique to this category, the \texttt{mechanism\_equation} focuses on pathway decomposition rather than solely on endpoint effects. Mathematically grounded in Structural Equation Modeling (SEM), typically
comprising a mediator model ($M = aX + \ldots$) and an outcome model
($Y = c'X + bM + \ldots$), the protocol mandates a clear distinction between the Natural Direct Effect (NDE, $c'$) and the Natural Indirect Effect (NIE, typically $a \times b$), alongside the calculation of the proportion mediated. This shifts the validation focus from intervention efficacy to whether a specific biological process is effective, constructing an explanatory evidence chain from phenomenon to mechanism.

\textbf{Associational Specifics}
Serving as the foundational layer of the evidence hierarchy, the extraction protocol for associational studies prioritizes the accurate recording of statistical correlations while rigorously preventing causal over-interpretation. This category excludes all diagnostic fields related to causal identification, instead introducing modules for the \texttt{association\_equation} and \texttt{descriptives}.Distinct from the preceding categories, the mathematical logic of the \texttt{association\_equation} is defined by the conditional expectation
\(E[Y \mid X, W]\), representing the distribution of outcome \(Y\) within a population defined by observed features \(W\). Here, control variables are strictly defined as covariates; the equation functions fundamentally as curve fitting rather than counterfactual inference. The \texttt{descriptives} module mandates the extraction of outcome distribution data stratified by exposure category, facilitating baseline characteristic comparisons and trend visualization analysis. Finally, the
\texttt{association\_nature} module explicitly records \texttt{residual\_confounding\_risk}, strictly confining the evidentiary scope to prediction and description to preclude misleading inferences regarding interventional guidance.

\begin{figure*}[h]
    \centering
    \includegraphics[width=1\linewidth]{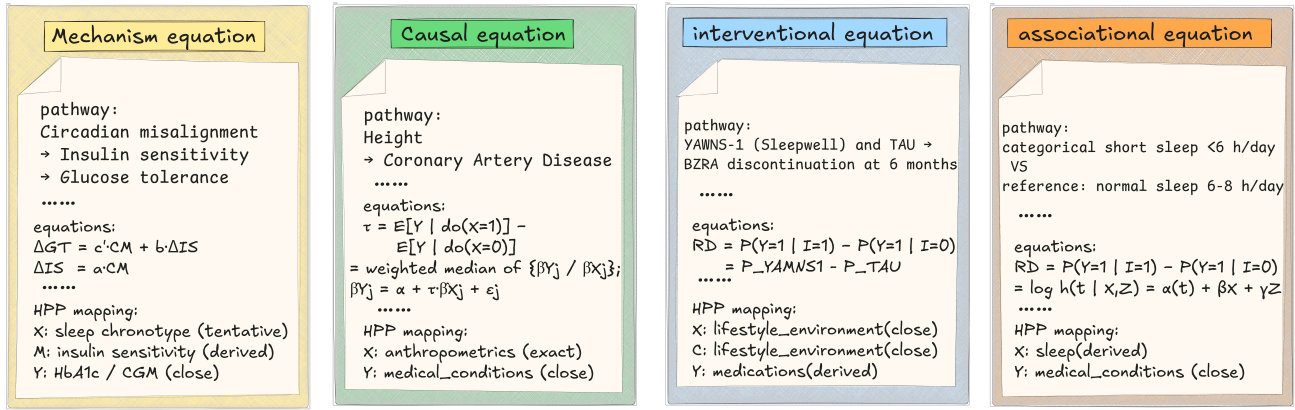}
    \vspace{-1em}
    \caption{Examples of evidence card files for each category.}
    \label{fig:evidence_card_temple}
    \vspace{-0.5em}
\end{figure*}

\begin{table*}[h]
\centering
\small
\resizebox{\linewidth}{!}{
\begin{tabular}{llll}
\toprule
\textbf{Card type} & \textbf{Component} & \textbf{Parameter} & \textbf{Value / Definition} \\
\midrule

\multirow{3}{*}{Mechanism}
& Glucose tolerance equation
& $c'$
& $6.0\%$ increase \\

& Glucose tolerance equation
& $b$
& $-1.0$ (directional) \\

& Insulin sensitivity equation
& $a$
& $-1.0$ (directional) \\

\midrule

\multirow{2}{*}{Causal}
& MR-weighted median
& $\tau$
& $0.113$ \\

& MR-Egger
& $\tau$
& $0.095$ \\

\midrule

\multirow{3}{*}{Interventional}
& Risk difference
& RD
& $P(Y{=}1\mid I{=}1) - P(Y{=}1\mid I{=}0)$ \\

& Intervention arm
& $P_{\text{YAWNS1}}$
& $0.262$ / $0.294$ \\

& Control arm
& $P_{\text{TAU}}$
& $0.075$ / $0.076$ \\

\midrule

\multirow{6}{*}{Associational}
& Cox proportional hazards model
& $\beta$
& $\log(1.21) \approx 0.1906$ \\

& Cox proportional hazards model
& $\gamma$
& 3.2436 \\

& Exposure definition
& $X$
& Indicator: short sleep $<6$ h/day (ref: 6--8 h/day) \\

& Covariate set
& $Z$
& 18 baseline covariates (demographic, lifestyle, metabolic, comorbidity) \\

& Cox proportional hazards model
& $\alpha(t)$
& Estimated non-parametrically \\

& Spline specification
& Knots
& 5th, 35th, 65th, 95th percentiles \\

\bottomrule
\end{tabular}
}
\caption{Registry of key parameters and hyperparameters across evidence card categories. Structural equations are summarized in Figure~\ref{fig:evidence_card_temple}; this table reports numerical parameters, operational thresholds, and implementation-specific settings used in each evidence card.}
\label{tab:evidence_card_temple}
\end{table*}

\subsection{HPP Field Mapping}
Building upon the evidence cards detailed in the preceding subsection, we establish a bridge between unstructured natural language descriptions from the literature and the specific local Human Phenotype Project (HPP) cohort data. This integration is achieved through deep alignment across semantic consistency, measurement methodology, and unit standardization.

\subsubsection{Mapping rule}

This study establishes a precise mapping rule grounded in a metadata dictionary. The core objective of this rule is to anchor abstract variable nodes extracted from evidence cards, including exposures ($X$), outcomes ($Y$), mediators ($M$), and instrumental variables ($IV$) to specific physical fields within the HPP database. The entire mapping system comprises a two-layer architecture featuring unique identifier locking and confidence grading. This framework first establishes variable node directionality via a strict data dictionary benchmark, and subsequently quantifies the match quality of each mapping relationship through a standardized status classification system.

The sole Ground Truth for this mapping process is restricted to the official Data Dictionary of the HPP platform \href{[https://knowledgebase.pheno.ai/participant_journey.html](https://knowledgebase.pheno.ai/participant_journey.html)}{Pheno AI Knowledgebase}, which serves as the definitive metadata source for all subdatasets.  Specifically, mappings must point strictly to the \texttt{tabular\_field\_name} within the database, which functions as the unique, persistent identifier bridging literature concepts and local data. The extraction protocol mandates that the mapping results must be spelled exactly as in the data dictionary, including the preservation of case sensitivity and underscore formatting. The use of non-standardized generic names or fabricated created field names is strictly prohibited to ensure that subsequent automated validation scripts can accurately call the corresponding column data. Acknowledging the inevitable heterogeneity between variable definitions in medical literature and the actual measurement methods of the local HPP cohort, this study introduces state classification to quantify the matching accuracy and confidence of each mapping relationship. Based on the degree of consistency in the definition, the system divides the mapping results into the following four mutually exclusive categories: Exact, Close, Derived, and Tentative. Examples of abridged evidence cards are shown in Figure~\ref{fig:evidence_card_temple} and Table~\ref{tab:evidence_card_temple}. The criteria and operating procedures for determining each mapping state are as follows:






\textbf{Exact}
Indicates that the variable from the literature is identical to the HPP field in definition, unit, and measurement methodology. For example, "Systolic BP (mmHg)" in the literature maps directly to \texttt{sbp\_mmHg} in HPP. This category offers the highest validation validity.

\textbf{Close}
Indicates high conceptual consistency between the literature variable and the HPP field, despite non-essential differences in specific device models, questionnaire versions, or unit magnitudes. These differences are judged not to fundamentally alter the biological significance of the variable. For instance, "AHI measured by PSG" in the literature versus "pAHI measured by portable devices" in HPP. In such cases, discrepancies must be explicitly documented in the \texttt{notes} field to facilitate bias assessment during validation.

\textbf{Derived}
Indicates that the literature variable does not correspond to a single raw field in HPP but can be obtained via mathematical calculation, logical aggregation, or thresholding of existing fields. Common scenarios include unit conversion (e.g., mg/dL to mmol/L), multi-field synthesis (e.g., calculating BMI from height and weight), or discretization of continuous variables (e.g., generating a categorical variable by applying a "<6 hours" threshold to a continuous sleep duration field). Such mappings require the explicit recording of the calculation formula or derivation logic in the \texttt{notes} field.

\textbf{Tentative}
Indicates that no clear corresponding field exists in the HPP data dictionary, or that ambiguous candidates exist without sufficient information to confirm consistency. These mappings are retained solely as potential leads for validation and are typically excluded from the automated validation of core mechanisms. The source of uncertainty and potential candidate fields must be noted in the \texttt{notes} field.

\subsubsection{Variable conversion and alignment}
Following basic field anchoring, we implemented a standardized variable transformation and alignment protocol for variables classified as Close or Derived to eliminate systematic bias arising from discrepancies in units, granularity, or definition thresholds. For cases where the units reported in the literature are inconsistent with the native units of the local HPP database, the system employs a conversion mechanism based on the UCUM (Unified Code for Units of Measure) standard, mandating the recording of precise conversion formulas in the remarks field (e.g., the literature uses the conventional unit mg/dL to report blood glucose, while HPP uses the international unit mmol/L), ensuring that the values are in the same dimensional space before entering the statistical validation model. For composite variables not representable by a single physical field, the protocol defines multi-field aggregation rules. Multiple source fields are linked via pipe separators, with synthesis algorithms—such as summation, averaging, or logical derivation—explicitly recorded to construct physiological metrics consistent with the literature (e.g., calculating Total Energy Intake by summing specific nutrients, or BMI from height and weight). For complex derived variables, such as "energy balance" inferred from daily weight changes and Dual-energy X-ray Absorptiometry (DXA) data, the system requires derivation formulas and intermediate variables are required to ensure logical alignment with the source text. Given that medical literature often uses specific cutoff values to convert continuous indicators into categorical variables, the alignment protocol mandates the explicit definition of classification thresholds and reference groups. These are marked as Derived states, with discretization rules specified to facilitate subsequent data binning. During validation, the system dynamically applies these rules to bin or binarize raw HPP data, reconstructing exposure and control groups homogeneous to those in the literature to ensure the comparability of categorical effect sizes. Finally, the system enforces strict temporal anchoring, aligning measurement time points from the literature with specific local data collection waves (e.g., mapping "Visit 1" to the corresponding HPP wave). Where literature employs means over specific time windows (e.g., "average sleep duration over the past 7 nights"), the protocol requires aggregating repeated measures from the corresponding local timeframe. This mechanism effectively prevents interference from time-varying confounding factors, ensuring that the temporal logic of exposure and outcome conforms to the basic assumptions of causal inference.


\subsection{Structural equation construction}
\subsubsection{Verification environment definition}
To enable the automated validation and causal inference of medical mechanism pathways on the local HPP dataset, we established a standardized validation environment and variable domain. Given the complexity of causal inference and the sparsity of longitudinal data, the current validation framework adopts a "Single Time-point Slice" model. This means that all variables exposure $X$, outcome $Y$, mediator $M$, and covariates $Z$, are anchored within a specific temporal window surrounding the participant's enrollment baseline (default: baseline $\pm$ 90 days). This setting treats each pathway to be validated, including the involved exposure $X$, outcome $Y$, potential mediator $M$, and adjustment set $Z$, as static observations within the same time slice, focusing on validating the causal structure of the mechanism pathway in the cross-sectional or short-term context. The construction of the validation cohort strictly follows the inclusion and exclusion criteria defined in the \texttt{population.eligibility\_signature} field of the evidence card. The automated system dynamically queries the HPP database to identify sub-cohorts matching specified age ranges, gender distributions, and disease histories, subsequently applying standardized preprocessing pipelines to key variables. For missing data, the system employs imputation strategies based on multiple imputation or full case analysis; for outliers in continuous variables, it uses interquartile range-based truncation. These definitions transform heterogeneous literature evidence into homogeneous, data-complete, static statistical tasks on the local dataset.

\subsubsection{Unified Structural Equations (EL-GSE) Specification}
To achieve a standardized transformation from unstructured documentary evidence to computable causal objects, this study defines and establishes the Edge-Local Generalized Structural Equation (EL-GSE) specification. EL-GSE is a general mathematical abstraction template designed to encapsulate each causal path extracted from the evidence card into an atomic unit with independent computational capabilities. For any candidate causal edge $X \to Y$, its EL-GSE formal definition is as follows:
$$g(\mathbb{E}[Y \mid X, Z]) = \alpha + \beta \cdot T_X(X) + \gamma^\top T_Z(Z)$$
In this equation, $g(\cdot)$ represents the link function, the form of which is strictly determined by the data distribution type of the outcome variable $Y$; $T_X(\cdot)$ is a contrastive coding function for the exposure variable $X$ (including standardization, quantile contrast, or treatment/control coding) used to mathematically align the effect definitions reported in the literature; $T_Z(\cdot)$ is the vectorized coded representation of the adjustment set $Z$; $\beta$ is the core causal parameter to be validated, corresponding to the effect size in the evidence card; $\gamma$ is the regression coefficient vector of the adjustment set; and $\alpha$ is the intercept term.

Based on the distribution characteristics of the outcome variables, the EL-GSE specification further derives two standard sub-forms to adapt to actual data: For continuous outcomes, the equation adopts the identity link function and the Gaussian noise assumption, degenerating into a linear structure equation $Y := \alpha + \beta T_X(X) + \gamma^\top T_Z(Z) + \epsilon$, where $\epsilon \sim \mathcal{N}(0, \sigma^2)$. For binary outcomes, the equation adopts the Logit link function and the Bernoulli distribution assumption, transforming into a logistic structure equation $Y := \mathrm{Bernoulli}(\pi)$, where $\text{logit}(\pi) = \alpha + \beta T_X(X) + \gamma^\top T_Z(Z)$.

The core advantage of the EL-GSE specification lies not only in its support for statistical reproducibility at the parameter level (such as verifying the direction and significance of $\beta$), but more importantly, in its inherent support for causal computation (Do-calculus). Under the premise of a single time point and identifiable backdoor path, this equation can be directly used as the execution carrier of the \texttt{do()} operator, utilizing the G-computation formula $\mathbb{E}[Y \mid do(X=x)] = \mathbb{E}_Z [g^{-1}(\alpha + \beta T_X(x) + \gamma^\top T_Z(Z))]$ to calculate the average intervention effect at the population level, or to calculate personalized counterfactual prognoses based on the covariate $Z=z$ for a specific individual. Finally, these equations and their parameters are serialized into a standard JSON object containing \texttt{form}, \texttt{parameters}, and \texttt{do\_support}, completing the transition from static literature knowledge to a dynamic, executable causal model.

\end{document}